%% file: FifthParallel.tex
\documentclass{article}
\usepackage{graphicx} 
\usepackage{amssymb, latexsym, amsmath, amscd}

\usepackage{enumitem}
\usepackage{pdfpages}

\newcommand{\Prob}{\rm Prob}

\newcommand{\Expect}{\bf\rm \mathbb{E}}

\newcommand{\upperbound}{1+\frac{1+\ln(1+p)}{p}+\frac{1}{p^2}\left(2\left({\ln(1+p)}\right)^2+{\ln(1+p})\right)}

\newcommand{\lowerbound}{\left( 1+\frac{1+\ln (1+p)}{p}-\frac{d\ln\ln p}{p}\right)}

\begin{document}

\include{mac90}

\title{
Adaptivity via a Parallel Architecture for Stochastic Gradient Methods
}

\author{Bin Fu\\\\
 University of Texas Rio Grande Valley
 \\
 bin.fu@utrgv.edu
 }
\date{} 
\maketitle


\begin{abstract}
Let $\mathrm{A}(x_0,y)$ be an algorithm with two inputs: an initial point $x_0$ and an integer parameter $y$, which specifies that $\mathrm{A}(.,.)$ executes at most $y$ iterations or steps. Given an integer $p\ge 1$, the $p$ processors in the proposed parallel framework search over a geometric sequence of iteration budgets to identify an appropriate value of $T$ for which has the desired convergence guarantees for $A(.)$. Each processor executes an infinite sequence of stages indexed by $i=0,1,2,\ldots$. At stage $i$, processor $j$ is assigned
$T_{j,i}=h(j,i),$
where $h:\mathbb{N}\times\mathbb{N}\rightarrow\mathbb{R}^{+}$ is a prescribed function. Processor $j$ $(j=0,1,\ldots,p-1)$ then executes $\mathrm{A}(x_0,T_{j,i})$.

The efficiency of the parallel framework is characterized by its $(p,\alpha_p)$-approximation guarantee. Specifically, for every integer $T\ge T_0$, there exist a processor $j$ and a stage $i$ such that
$T\le T_{j,i}\le T_{j,i}^*<\alpha_p T,$
where
$T_{j,i}^*=\sum_{t=0}^{i}T_{j,t}$
denotes the cumulative number of iterations executed by processor $j$ from the beginning to  stage $i$. Thus, $T_{j,i}^*$ represents the total computational effort expended by processor $j$ before completing stage $i$. We prove that this framework achieves a $(p,\alpha_p)$-approximation, where
$
\alpha_p
\le \upperbound$ for all $p\ge 1$, 
and a lower bound showing that for any  fixed $d>1$, 
$\alpha_p\ge \lowerbound$
for all sufficiently large $p$.

We develop arithmetically simple stochastic gradient methods in which every division is of the form $x/2^t$ for some integer $t$, and integrate them into the proposed parallel framework. This yields a simpler convergence analysis for stochastic gradient descent on a nonconvex objective function $F(x)$ while preserving adaptivity to unknown problem parameters, including the Lipschitz smoothness constant and stochastic gradient characteristics such as the variance and noise level.

\end{abstract}

\input{FifthContent.tex}

\end{document}

\section{Introduction}

Stochastic Gradient Descent (SGD)~\cite{RobbinsAndMonro51} is one of the most widely used optimization methods in deep learning because of its efficiency and scalability in training large-scale neural networks. Unlike batch gradient descent, which computes the gradient using the entire training dataset at every iteration, SGD updates the model parameters using a single training example or a small mini-batch. Consequently, SGD requires significantly less memory and has a much lower computational cost per iteration. By processing only a small subset of the data at each step, SGD often converges more quickly in practice, particularly for large-scale datasets.



Gradient descent with diminishing step sizes has a long history. Classical stochastic approximation theory shows that the step sizes ${\eta_i}$ should satisfy
$
\sum_{i=1}^{\infty}\eta_i=+\infty
\quad\text{and}\quad
\sum_{i=1}^{\infty}\eta_i^2<+\infty
$
to guarantee convergence to a stationary point~\cite{RobbinsAndMonro51}. 
For stochastic optimization of smooth nonconvex functions, gradient descent with either a constant step size or a diminishing step size $\eta_i=O(1/\sqrt{i})$ achieves an $O(1/\sqrt{T})$ convergence rate to a stationary point~\cite{GhadimiAndLan13}. In particular, the analysis in~\cite{GhadimiAndLan13} selects the step size as
$\eta_i=\min \left(\frac{1}{L},
\sqrt{\frac{2(F(x_1)-F(x^*))}{L\sigma_0^2N}}
\right),
$
which depends on the Lipschitz smoothness constant $L$, the stochastic gradient variance parameter $\sigma_0$, and the optimality gap $F(x_1)-F(x^*)$, where $F(x^*)=\inf_x(F(x))$. Since these problem-dependent parameters are typically unknown in advance, the resulting static gradient method is non-adaptive. Moreover, the convergence rate of $O(1/\sqrt{T})$ is known to be optimal, matching the corresponding lower bound~\cite{BottouCurtisNocedal18,ArjevaniCDFSW23}.


Adaptive gradient descent methods have become widely used in deep learning in recent years. Unlike static gradient methods, adaptive methods dynamically adjust the learning rate during training according to the historical gradients of individual parameters. This adaptive mechanism reduces the need for manual tuning of learning rates and often improves optimization efficiency and robustness across a wide range of machine learning tasks. A large body of research has established convergence guarantees and convergence rates for adaptive gradient methods under various assumptions and optimization settings~\cite{DuchiHazanSinger2011,DBLP:journals/corr/abs-1002-4908,NemirovskiJuditskyLanShapiro2009,BottouCurtisNocedal2018,WardWuBottou19,XieWuWard2020,DBLP:journals/corr/abs-1212-5701}.


Since the introduction of AdaGrad~\cite{DuchiHazanSinger2011}, numerous adaptive gradient methods have been proposed, including AdaDelta~\cite{DBLP:journals/corr/abs-1212-5701}, Adam~\cite{KingmaBa2015}, AdamW~\cite{Loshchilov2019Decoupled}, AdaFTRL~\cite{OrabonaPal2015}, SGD-BB~\cite{TanMaDaiQian2016}, AdaBatch~\cite{DBLP:journals/corr/abs-1711-01761}, SC-AdaGrad~\cite{DBLP:conf/icml/MukkamalaH17}, AMSGrad~\cite{DBLP:conf/iclr/ReddiKK18}, and Padam~~\cite{DBLP:conf/ijcai/ChenZTYCG20}. These developments reflect the continuing effort to improve adaptive gradient methods by enhancing their efficiency, robustness, theoretical guarantees, and ease of use for large-scale machine learning applications.


Adaptive stochastic gradient descent methods dynamically adjust the step size according to predefined update rules. For example, AdaGrad-Norm~\cite{WardWuBottou19} updates the accumulated scaling factor and the model parameters as
$
s_{i+1}=s_i+|G(\xi,x_i)|^2,$
and
$x_{i+1}=x_i-\frac{\eta}{\sqrt{s_{i+1}}}\cdot G(\xi,x_i),
$
where $G(\xi,x_i)$ denotes the stochastic gradient evaluated at $x_i$. The convergence properties of adaptive stochastic gradient methods have been extensively studied in~\cite{WardWuBottou19,XieWuWard2020,FawRoutCaramanisShakkottai23,WangZhangMaChen023}. Under suitable assumptions, these methods are proven to converge to a stationary point with the optimal convergence rate of $O(1/\sqrt{N})$.


Parallel gradient descent has become an important optimization framework for large-scale machine learning and scientific computing because it enables gradient computations to be distributed across multiple processors, thereby significantly reducing training time and improving scalability. Early theoretical foundations for parallel and asynchronous iterative optimization were established by Dimitri P. Bertsekas and John N. Tsitsiklis~\cite{BertsekasTsitsiklis89}, who analyzed convergence properties under delayed and distributed updates.
Large-scale machine learning later motivated parallel SGD algorithms such as Hogwild \cite{recht2011hogwild}, parameter-server architectures \cite{li2014scaling}, and distributed deep learning systems \cite{dean2012large}. Recent adaptive parallel methods further combine distributed computation with adaptive learning-rate mechanisms for improved convergence behavior \cite{reddi2018convergence}.  More recent research has focused on adaptive and communication-efficient distributed optimization, including adaptive SGD methods \cite{cutkosky2018distributed} and multi-timescale distributed adaptive optimization frameworks \cite{iacob2025mtdao}.

\subsection{Our Contributions}

Our goal is to endow static gradient methods with adaptivity through a parallel framework. A static gradient method, denoted by $\mathrm{GD}(x_0,T)$, takes as input an initial point $x_0$ and an integer $T$ specifying the number of iterations. The step size is determined by
$s=S(T),$
where $S(\cdot)$ is a prescribed function of $T$. The method then performs the iterations
$
x_{i+1}=x_i-\frac{\eta}{s}\cdot g_i,
$
where $g_i$ is a stochastic gradient evaluated at $x_i$, and $\eta$ is a scaling factor. 
A fundamental challenge in applying a static gradient method is selecting an appropriate value of $T$. The parameter $T$ must be sufficiently large to guarantee the desired convergence, yet the required number of iterations typically depends on unknown problem characteristics, such as the Lipschitz smoothness constant and the stochastic gradient parameters. Our parallel framework addresses this challenge by searching for a suitable value of $T$ through parallel execution.


We develop a parallel framework for gradient descent that searches for an appropriate iteration budget $T$ in parallel according to a carefully designed geometric sequence. The framework assembles multiple static gradient methods into an adaptive gradient descent method. It consists of $p$ processors running in parallel. The number of iterations assigned to processor $j$ at stage $i$ is determined by $
T_{j,i}=h(j,i).
$
Processor $j$ $(j=0,1,\ldots,p-1)$ executes $\mathrm{GD}(x_0,T_{j,i})$ at stage $i$ ($i=0,1,2,\ldots$). We choose
$
h(j,i)=b_p^{jp+i}T_0,
$
where
$
b_p=(p+1)^{1/p},
$
and $T_0$ is the minimum number of iterations assigned to any stage.

We prove that, for every $T\ge T_0$, there exist a processor $j$ and a stage $i$ such that
\[T\le T_{j,i}\le T_{j,i}^*<\alpha_pT,\]
where
$T_{j,i}^*=\sum_{t=0}^{i}T_{j,t}$ 
is the total number of iterations executed by processor $j$ through stage $i$, and
\[
\alpha_p=\left(1+\frac{1}{p}\right)(p+1)^{1/p}\le\upperbound.
\]
The approximation factor $\alpha_p$ measures the computational overhead incurred before a processor reaches an iteration budget that is sufficient to satisfy the desired convergence guarantee. A smaller value of $\alpha_p$ indicates that fewer iterations are wasted in the preceding stages.

We further establish a nearly matching lower bound by proving that, for any scheduling function $h(j,i)$ and any constant $d>1$, every $(p,\alpha_p)$-approximation must satisfy
\[
\alpha_p\ge\lowerbound
\]
for all sufficiently large $p$.

The convergence behavior of the resulting parallel algorithm is therefore essentially the same as that of the underlying static gradient method $\mathrm{GD}(\cdot)$, whose convergence is determined by the iteration budget $T$. Consequently, our framework provides the adaptivity of parallel search while preserving the relatively simple convergence analysis of static gradient methods. Theoretical analysis establishes nearly matching upper and lower bounds on $\alpha_p$, revealing an intrinsic tradeoff between parallelism, adaptivity, and computational overhead.

We develop a static gradient descent method under the following $(\lambda,\sigma_0,\sigma_1)$-stochastic model. Let $\xi$ be a random variable, and let $G(\xi,x)$ denote a stochastic gradient of $F(x)$. We assume that

\[
\left\langle \mathbb{E}_{\xi}[G(\xi,x)],,\nabla F(x)\right\rangle
\ge
\lambda|\nabla F(x)|^2
\]
for some $\lambda\in(0,\infty)$; and

\[
\mathbb{E}_{\xi}\left[|\nabla F(x)-G(\xi,x)|^2\right]
\le
\sigma_0^2+\sigma_1^2|\nabla F(x)|^2
\]
for some $\sigma_0,\sigma_1\in[0,\infty)$.

This model generalizes the standard stochastic gradient model, which assumes
$
\mathbb{E}_{\xi}[G(\xi,x)]=\nabla F(x),
$
and
$
\mathbb{E}_{\xi}\left[|\nabla F(x)-G(\xi,x)|^2\right]
\le
\sigma_0^2
$
for some $\sigma_0\in[0,\infty)$.

Rigorous convergence analysis of stochastic and adaptive gradient methods is essential for understanding their theoretical behavior, improving their performance, and ensuring their reliability across a broad range of machine learning tasks. Such analyses also reveal how convergence depends on problem characteristics and algorithmic hyperparameters, thereby guiding the design of more robust optimization algorithms.

To the best of our knowledge, convergence guarantees under the above $(\lambda,\sigma_0,\sigma_1)$-stochastic model, in which all three parameters $\lambda$, $\sigma_0$, and $\sigma_1$ are allowed to be positive, have not been established in the existing literature. We establish the following convergence results under this model.


We develop a new gradient descent method under the proposed $(\lambda,\sigma_0,\sigma_1)$-stochastic model. Given an iteration budget of $T$ steps, the method performs the updates
\[
x_{j+1}=x_j-\frac{\eta}{s_T}G(\xi,x_j),
\]
where
$s_T=2^{\left\lceil \frac{\lceil\log T\rceil}{2}\right\rceil},$
and $\eta>0$ is an arbitrary input parameter.

Assuming that the objective function satisfies the standard $L$-Lipschitz smoothness condition,
\[
|\nabla F(x)-\nabla F(y)|\le L|x-y|,
\]
we prove that the proposed method converges to a stationary point for nonconvex optimization under the $(\lambda,\sigma_0,\sigma_1)$-stochastic model. Moreover, it achieves the optimal convergence rate of $O(1/\sqrt{T})$.

In the gradient descent methods proposed in this paper, every denominator is of the form $2^t$ for some integer $t$. As a result, division and square-root operations, which are commonly used in gradient descent algorithms, are eliminated and replaced by binary shift operations. This simplification makes the proposed methods more suitable for efficient hardware implementation and chip design.

The convergence analysis of the proposed stochastic model provides a theoretical explanation for why a parallel search over the iteration budget $T$ is necessary. Our algorithm is parameter-adaptive: it automatically adapts to the unknown Lipschitz smoothness constant $L$ and the stochastic gradient parameters $\lambda$, $\sigma_0$, and $\sigma_1$. Moreover, the parallel architecture is constructed independently of these unknown parameters, making the framework broadly applicable without prior knowledge of the optimization problem.


\subsection{Organization of This Paper}

The remainder of this paper is organized as follows. In Section~\ref{overview-sec}, we present an overview of the proposed parallel framework for static gradient descent. 
Section~\ref{parallle-sec} formally introduces the parallel model and the notion of a $(p,\alpha_p)$-approximation. In Section~\ref{Upper-sec}, we derive upper bounds on $\alpha_p$, while Section~\ref{lower-sec} establishes corresponding lower bounds. 
Section~\ref{alg-sec} 
presents the convergence analysis of the proposed static stochastic gradient method and demonstrates how it fits into the parallel framework. In Section~\ref{parallle2-sec}, we introduce a refined parallel model that avoids repeatedly restarting from the same initial point $x_0$. 
Instead, it progressively replaces $x_0$ 
with an improved starting point 
$x_0^*$ satisfying 
$F(x_0^*)\le F(x_0)$. 
Finally, we conclude that adaptivity can be achieved through parallelization while preserving the simplicity of convergence analysis for static gradient methods.


\section{Overview of Our Method}\label{overview-sec}

In the parallel framework developed in this paper, we assume that $p$ processors execute concurrently. The processors cooperatively search for a suitable iteration budget $T$ for the given static gradient method $\mathrm{GD}(x_0,T)$ by exploring a geometric sequence of candidate values. Each processor $j$ $(j=0,1,\ldots,p-1)$ proceeds through an infinite sequence of stages. At stage $i$, processor $j$ is assigned an iteration budget $T_{j,i}$ and executes $\mathrm{GD}(x_0,T_{j,i})$.

The candidate iteration budgets are selected from the geometric sequence
\[
T_0,\ (1+\epsilon)T_0,\ (1+\epsilon)^2T_0,\ \ldots,\ (1+\epsilon)^kT_0,\ \ldots.
\]
Specifically, we define
\[
T_{j,i}=h(j,i)=(1+\epsilon)^{ip+j}T_0,
\]
for $j=0,1,\ldots,p-1$. Consequently, for any desired iteration budget $T$ that is sufficiently large to satisfy the convergence guarantee, there always exists a value $T_{j,i}$ in the sequence such that $T_{j,i}$ is only slightly larger than $T$.

The scheduling function $h(j,i)$ and the parameter $\epsilon$ are determined by the number of processors $p$. As $p$ increases, the value of $\epsilon$ decreases, yielding a denser geometric sequence and thereby reducing the gap between the selected iteration budget and the desired value $T$.



\begin{figure}[htbp]
    \centering
\includegraphics[page=1,width=360pt]{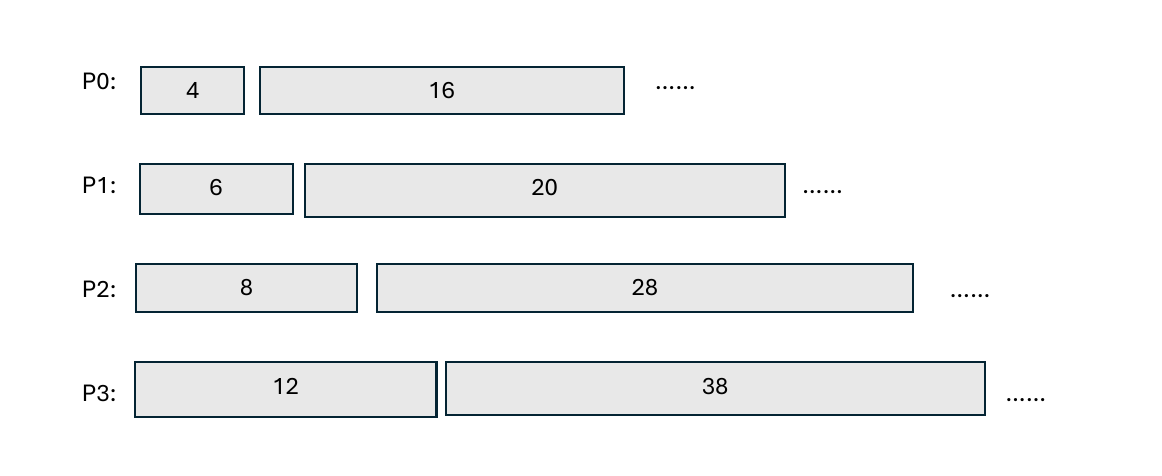}\caption{Four Parallel Processors with Infinitely Many Stages.}
    \label{Fig}
\end{figure}

We show that, for every target iteration budget $T$, there exist a processor $j$ and a stage $i$ such that
\[
T\le T_{j,i}\le \sum_{t=0}^{i}T_{j,t}<\alpha_pT.
\]
We derive both upper and lower bounds for $\alpha_p$, and show that these bounds are nearly tight in the proposed parallel model. The scheduling function $h(j,i)$ and the parameter $\epsilon$ are designed according to the number of processors $p$. As $p$ increases, $\alpha_p$ approaches $1$, implying that only a small amount of computation is wasted before reaching an iteration budget that satisfies the desired convergence guarantee.

Figure~\ref{Fig} illustrates the parallel framework with four processors. Each rectangle represents one stage of a processor, and the integer inside the rectangle denotes the number of iterations assigned to that stage. For example, when $T=26$, processor $P_2$ reaches Stage~1 with
\[
T\le T_{2,1}=28,
\]
and the cumulative number of iterations executed by that processor is
\[
T_{2,1}^*=T_{2,0}+T_{2,1}=36.
\]

We design the parameter $\epsilon$ to balance two competing objectives: efficiently locating a suitable iteration budget $T_{j,i}$ and keeping the approximation factor $\alpha_p$ close to $1$. Once a sufficiently large iteration budget $T_{j,i}$ is identified, the corresponding step size, determined by the function $S(T_{j,i})$, satisfies the conditions required for the convergence guarantee of the underlying static gradient method $\mathrm{GD}(\cdot)$.





\section{A Parallel Framework for Gradient Methods}\label{parallle-sec}

In this section, we introduce a parallel architecture that enables a gradient method to adapt automatically to unknown problem parameters. By running multiple instances of a static gradient method in parallel, each with a different fixed step size, our framework transforms a static gradient method into an adaptive one.



Let $\mathbb{R}=(-\infty,+\infty)$ denote the set of real numbers, and let $\mathbb{R}^{+}=(0,+\infty)$ denote the set of positive real numbers. Let $\mathbb{N}=\{0,1,2,\ldots\}$ denote the set of nonnegative integers. For a real number $x$, let $\lceil x\rceil$ denote the smallest integer greater than or equal to $x$, and $\floor{x}$ denote the largest integer less than or equal to $x$.


\subsection{Parallel Frameowrk}

\begin{definition}
A function $h:\mathbb{N}\times\mathbb{N}\rightarrow\mathbb{R}$ is called \emph{geometric} if there exist constants $h_0>0$, $a_0>1$, and an integer $p\ge1$ such that
\[
h(j,i)=h_0a_0^{jp+i}
\]
for all $j,i\in\mathbb{N}$.
Equivalently, the values of $h(j,i)$ are given by the geometric sequence
\[
h_0,\ h_0a_0,\ h_0a_0^2,\ \ldots.
\]
\end{definition}



We have the following parallel framework that calls a static gradient descent method $GD(x_0, T)$. The parallel executions of GD$(x_0,T_{j,i})$ finds a $T_{j,i}$ that will satisfy the condition of convergence. 

\vskip 20pt
{\bf Algorithm} Parallel-Framework$(A(.,.),  x_0, h(.,.), T_0, p)$

Related Parameters: 

\begin{itemize}
\item
$A(x_0, T)$ is an algorithm  with input point $x_0$, and an integer $T$ for the number of  iterations or steps to execute $A(.)$.

    \item $T_0$ is the least number of steps to execute

    \item $h(j, i): \mathbb{N}\times \mathbb{N}\rightarrow \mathbb{N}$ is a function to assign the number of iterations when calling a existing gradient descent method.


    \item $x_0\in \mathbb{R}^m$ is the start point,

    \item $p$ is the number of processors.
\end{itemize}

Processor $j$ ($j=0,1,\ldots, p-1$):

\begin{enumerate}[label=\arabic*.]

   \item $i=0$

\item Repeat

   \item $\{$ 

     \item \qquad Let $T_{j,i}=h(j,i)$

\item \qquad $A(x_0, T_{j,i})$


    \item \qquad  $i=i+1$

\item $\}$


\end{enumerate}

{\bf End of Algorithm}

\begin{definition}\label{App-def}
Let $\mathrm{Parallel\mbox{-}Framework}(\cdot)$ denote the parallel framework defined by the algorithm.

\begin{enumerate}
\item
We say that $\mathrm{Parallel\mbox{-}Framework}(\cdot)$ has a {\it $(p,\alpha_p)$-approximation} if it consists of $p$ processors indexed by $0,1,\ldots,p-1$, and for every integer $T\ge T_0$, there exist a processor $j<p$ and a stage $i$ such that
\begin{eqnarray}
T\le T_{j,i}\le T_{j,i}^*<\alpha_pT,\label{basic-ineqn}    
\end{eqnarray}

where $
T_{j,i}^*=\sum_{t=0}^{i}T_{j,t}.$

\item
A {\it geometric parallel framework of $p$ processor} is a parallel framework in which the scheduling function $h(j,i)$ for $p$ processors generates the iteration budgets according to a geometric progression, and can be expressed as $h(j,i)=b_p^{ip+j}T_0$ for some $b_p>1$ and $T_0\ge 1$.
\end{enumerate}
\end{definition}

In Definition~\ref{App-def}, the condition (\ref{basic-ineqn})
measures the computational overhead incurred before reaching an iteration budget $T_{j,i}$ that is at least the target value $T$. We derive both upper and lower bounds for the approximation factor $\alpha_p$. Furthermore, the proposed parallel framework guarantees that $\alpha_p$ can be made arbitrarily close to $1$ as the number of processors $p$ increases.

Lemma~\ref{monotonic-lemma} establishes a monotonicity property of $T_{j,i}$ and $T_{j,i}^*$ in the geometric parallel framework with $p$ processors. This property will be used to derive a lower bound that matches the corresponding upper bound for $\alpha_p$ in a geometric parallel framework.

\begin{lemma}\label{monotonic-lemma}
For the geometric parallel framework with $p$ processors, if $0 \le j < k \le p-1$, then
\[
T_{j,i}<T_{k,i}
\quad\text{and}\quad
T_{j,i}^*<T_{k,i}^*
\]
for every stage $i$.
\end{lemma}

\begin{proof}
The result follows directly from Definition~\ref{App-def}, which defines $T_{j,i}$, $T_{j,i}^*$, and the geometric parallel framework, together with the assumptions $b_p>1$ and $T_0\ge 1$.
\end{proof}





\subsection{An Example for A(.)}

We first describe a static gradient method whose step size is determined by a function $S(T)$ of the prescribed number of iterations $T$. For example, let $S(T)=\sqrt{T}$. Instead of computing $\sqrt{T}$ exactly, we seek an integer $m$ such that
\[
2^m\in[\sqrt{T},\,4\sqrt{T}].
\]
This approximation eliminates square-root and division operations while preserving the desired asymptotic behavior.

We give a description of a static gradient descent. Its step size is determined by a function $S(T)$. For example, $S(T)=\sqrt{T}$. We tend to find an integer $m$ such that $S(T)=2^m\in [\sqrt{T}, 4\sqrt{T}]$.  This can remove division and square root operations.
\vskip 20pt
{\bf Algorithm} Static-SGD$(x_0, T)$

Related Parameters: 
\begin{itemize}
    \item 
$x_0$ is the start point
\item 
$T\in [1,+\infty)$ controls the number of iterations

\item 
$\eta$ is a scaling factor

\item 
$S(T):\mathbb{R^+}\rightarrow\mathbb{R^+}$ is a function to determine the stepsize based on $T$

\item 
$G(\xi, x)$ is a stochastic (approximate) gradient for $x$.
\end{itemize}

Steps:

\begin{enumerate}[label=\arabic*.]
    \item $i=1$

\item $x_1=x_0$

\item $s=S(T)$

    \item while $i\le T$ 

    \item $\{$

    \item \hskip 20pt $g_i=G(\xi, x_i)$

\item \hskip 20pt $x_{i+1}=x_i-\frac{\eta}{s} \cdot g_i$

    \item \hskip 20pt $i=i+1$

    \item $\}$

\end{enumerate}

{\bf End of Algorithm}

\vskip 20pt

\subsection{A Matrix Construction}

We transform the parallel framework into a $p\times \infty$ matrix $\Gamma_{p\times \infty}=(T_{j,i})_{p\times  \infty}$ for $j=0,1,\ldots$ and $i=0,1,\ldots$ (Both row and column indices start from $0$). The entry at row $j$ and column $i$ is denoted by $T_{j,i}$. $T_{j,i}^*$ and $(p,\alpha_p)$-approximation are defined as in Definition~\ref{App-def}.  A row $j$ in $\Gamma$ repents processor $j$ for its stages, and column $j$ in $\Gamma$ is for stage $j$. Matrix $\Gamma_{p\times \infty}=(T_{j,i})_{p\times  \infty}$ is generated by $h(j,i)$ if $T_{j,i}=h(j,i)$ for $j=0,1,\ldots,p-1$, and $j=0,1,\ldots$.
 
\section{Upper Bound for $\alpha_p$ in Parallel Model}\label{Upper-sec}
In this section, we show a $(p,\alpha_p)$-approximation for the parallel model. An upper bound for the parameter $\alpha_p$ will be derived.



\begin{lemma}\label{Taylor-lemma}
   For $x\in [0,1]$, $e^x\le 1+x+x^2$. 
\end{lemma}
\begin{proof}
  It follows from the Taylor expansion of $e^x$: $e^x=1+x+\frac{x^2}{2!}+\frac{x^3}{3!}+\ldots\le 1+x+\frac{x^2}{2!}+x^3(\frac{1}{3!}+\frac{1}{4!}+\ldots)\le 1+x+\frac{x^2}{2!}+x^3(\frac{1}{2^2}+\frac{1}{2^3}+\ldots)\le 1+x+x^2$.  
\end{proof}

\begin{lemma}\label{help-lemma}
If $p$ is an integer with $p\ge 1$,     then
\begin{eqnarray*}
\left(1+\frac{1}{p}\right)(1+p)^{1/p}\le 1+\frac{1+\ln(1+p)}{p}+\frac{1}{p^2}\left(2\left({\ln(1+p)}\right)^2+{\ln(1+p})\right).    
\end{eqnarray*}
\end{lemma}

\begin{proof}
It is easy to verify that $\frac{\ln(1+p)}{p}<1$ for all integers $p\ge 1$. A simple induction shows $\ln (1+p)< p$. It is true at   $p=1$  as $e\approx 2.71828$. Assume $\ln (1+p)\le p$. We have $\ln(1+(p+1))< \ln e(1+p)= 1+\ln(1+p)< 1+p$.
By Lemma~\ref{Taylor-lemma}, we have
  \begin{eqnarray*}
(1+p)^{1/p}=e^{\frac{\ln(1+p)}{p}}\le 1+\left(\frac{\ln(1+p)}{p}\right)+\left(\frac{\ln(1+p)}{p}\right)^2.    
  \end{eqnarray*}

Therefore, 

\begin{eqnarray*}
   &&\left(1+\frac{1}{p}\right)(1+p)^{1/p}\le \left(1+\frac{1}{p}\right) \left(1+\left(\frac{\ln(1+p)}{p}\right)+\left(\frac{\ln(1+p)}{p}\right)^2\right)\\
   &=&1+\frac{1}{p}+\left(\frac{\ln(1+p)}{p}\right)+\left(\frac{\ln(1+p)}{p}\right)^2+\frac{1}{p}\left(\left(\frac{\ln(1+p)}{p}\right)+\left(\frac{\ln(1+p)}{p}\right)^2\right)\\
   &\le&1+\frac{1+\ln(1+p)}{p}+\frac{1}{p^2}\left(2\left({\ln(1+p)}\right)^2+{\ln(1+p})\right).
\end{eqnarray*}
\end{proof}

Theorem~\ref{main2b-thm} shows an upper bound for $\alpha_p$ for $(p,\alpha_p)$-approximation. It covers all the cases for $p\ge 1$. Its proof shows how to select function $h(.,.)$.

\begin{theorem}\label{main2b-thm}
Let function $h(j, i)=b_p^{ip+j}T_0$ and  $b_p=(p+1)^{\frac{1}{p}}$.
For any integer $T\ge T_0$, the Parallel-Framework(.) has $(p,\alpha_p)$-approximation   with $\alpha_p= (1+\frac{1}{p})(1+p)^{\frac{1}{p}}
\le \upperbound.$
\end{theorem}

\begin{proof} The processor $j$ will use the steps $b_p^{j}\cdot T_0, b_p^{p+j}\cdot T_0, b_p^{2p+j}\cdot T_0,\ldots, b_p^{ip+j}\cdot T_0,\cdots$. At phase $i$, processor $j$ uses $T_{j, i}=h(j,i)=b_p^{ip+j}T_0$ to control the number of  of iterations in GD($x_0, T_{j,i}$). The proof also shows how $b_p$ is computed to get a minimal $\alpha_p$.

Define
\begin{eqnarray*}
T_{j, i}^*=\sum_{t=0}^{i} b_p^{tp+j}T_0=b_p^jT_0\sum_{t=0}^{i} b_p^{tp}=b_p^jT_0\cdot \frac{b_p^{(i+1)p}-1}{b_p^p-1}< \frac{b_p^{(i+1)p+j}T_0}{b_p^p-1}=\frac{b_p^pT_{j,i}}{b_p^p-1}.  
\end{eqnarray*}

We note that $T_{j,t}=b_p^{tp+j}T_0$ is the number of steps in the $t$-th iteration. 
Let $T_{j,t}=b_p^{tp+j}\cdot T_0$ be the least with $T\le T_{j,t}$. We have $T\le T_{j,t}<b_pT$. 






\begin{eqnarray*}
T\le T_{j,t}\le T_{j,t}^*&<& 
 \frac{b_p^p}{b_p^p-1}\cdot T_{j,t}\\
&\le& \frac{b_p^{p+1}}{b_p^p-1}\cdot T
\end{eqnarray*}

Define $f(x)$ by 
    \begin{eqnarray}
   f(x)=\frac{x^{p+1}}{x^p-1}. 
\end{eqnarray}

Take derivative for $f(x)$.

    \begin{eqnarray}  f(x)'=\frac{(p+1)x^p(x^p-1)-px^{2p}}{(x^p-1)^2}. 
\end{eqnarray}

Let 
    \begin{eqnarray}
   (p+1)x^p(x^p-1)-px^{2p}=0. 
\end{eqnarray}
It transformed into

    \begin{eqnarray}
   x^p-(p+1)=0. 
\end{eqnarray}
So, we can let $b_p=(p+1)^{1/p}$ to have least $f(b_p)$.

So,
    \begin{eqnarray*}  f(b_p)&=&\frac{b_p^{p+1}}{b_p^p-1}=\frac{(1+p)b_p}{p}\\
   &=&\frac{(1+p)(1+p)^{1/p}}{p}   =\left(1+\frac{1}{p}\right){(1+p)^{1/p}}\\  &\le&\upperbound (by~Lemma~\ref{help-lemma}). 
\end{eqnarray*}

Therefore, if $b_p=(p+1)^{1/p}$,  we have $T\le T_{j,i}\le T_{j,i}^*\le \alpha_pT$ with  $\alpha_p= (1+\frac{1}{p})(1+p)^{\frac{1}{p}}$.

\end{proof}




    
We have Corollary~\ref{main2b-cor} for the cases $p=1,2$. They correspond to the cases for one processor, and two processors, respectively.

\begin{corollary}\label{main2b-cor}Let $p$ be the number of processors in Parallel-Framework(.). 
We have 
\begin{enumerate}
    \item 
For $p=1$,  
  Parallel-Framework(.) has $(1,4)$-approximation with  $b_1=2$.
\item 
For $p=2$, Parallel-Framework(.) has $(2,2.5981)$-approximation
 with  $b_2=\sqrt{3}$.
\end{enumerate}
\begin{proof}
   It follows from Theorem~\ref{main2b-thm} with $\alpha_p= (1+\frac{1}{p})(1+p)^{\frac{1}{p}}$. 
\end{proof}

\end{corollary}

Using the numerical solutions for the expression of $\alpha_p$ in Theorem~\ref{main2b-thm}, we have upper bounds below:
\begin{eqnarray*}
    \alpha_3&\le& 2.11654,
    \alpha_4\le 1.86919,
    \alpha_5\le 1.71717,
    \alpha_6\le 1.61229,
    \alpha_7\le 1.53459\\
    \alpha_8&\le& 1.474397,
    \alpha_{9}\le 1.42615,
    \alpha_{10}\le 1.38644,
    \alpha_{11}\le 1.35309,
    \alpha_{12}\le 1.32459,\\
    \alpha_{13}&\le& 1.29994,
    \alpha_{14}\le 1.27844,
    \alpha_{15}\le 1.25966,
    \alpha_{16}\le 1.24329.
\end{eqnarray*}

\section{Lower Bounds for $\alpha_p$ with Arbitrary $h(.,.)$}\label{lower-sec}

In this section, we show a lower bound in the parallel model. The lower bound of this section has a small gap with the upper bound of Section~\ref{parallle-sec}. Our lower bound almost matches the upper bound.

\begin{lemma}\label{lowerbd-lemma} For any function  $h(j, i)$, 
 if  Parallel-Framework(.) has   $(p,\alpha_p)$-approximation,  then  we have 
 \begin{enumerate}
     \item\label{first-case-alphap} for any positive integer $z$,  $\alpha_p\ge 1+\frac{1}{\alpha_p^p}+\frac{1}{\alpha_p^{2p}}+\ldots+\frac{1}{\alpha_p^{zp}}$, and   
     \item $\alpha_p\ge r_0$, where $r_0>1$ is a root of $x^p-x^{p-1}-1=0$.
 \end{enumerate}

\end{lemma}

\begin{proof} We fix $p$ and $\alpha_p>1$  (by its definition).  Define $T_{j,i}^*=\sum_{1\le i\le t}T_{j,t}$
Let consider the sequence $V_0=T_0, V_1=\beta V_0,\ldots, V_k=\beta^k T_0,\ldots$. By the condition of $(a,\alpha_p)$-approximation (Definition~\ref{App-def}), for each $V_k$, we have a $T_{j,i}$ such that $V_k\le T_{j, i}\le T_{j, i}^*\le \alpha_p V_k$. Let 
\begin{eqnarray}
H=\sum_{k=0}^mV_k=T_0(1+\beta+\beta^2+\ldots+\beta^m)=\frac{\beta^{m+1}-1}{\beta-1}T_0.  \label{H-eqn}    
\end{eqnarray}

We have that for each $k\le m$, $V_k\le T_{j, i_j}^*< \alpha_p V_k$. We will select $\beta=\alpha_p$. This makes the case for each $T_{j,i}$, there is at most one $V_k$ to have $V_k\le T_{j,i}\le T_{j,i}^*< \alpha_p V_k$. This is because $V_{k+1}=V_k\beta=V_k\alpha_p> T_{j,i}^*$. Thus, $V_{k+1}$ does not satisfy the inequality $V_{k+1}\le T_{j,i}\le T_{j,i}^*< \alpha_p V_{k+1}$. 

 Let $Q\subseteq \{0,1,\ldots, p-1\}$ such that for each $q\in Q$, there is a $V_k$ with $0\le k\le m$ and $V_k\le T_{q, i}\le T_{q, i}^*\le \alpha_p V_k$ for some $i$. For a $q\in Q$, let $T_{q, i_q}$ be the largest with $V_k\le T_{q, i_q}\le T_{q, i_q}^*< \alpha_p V_k$ for some $V_k$ ($0\le k\le m$).

Among the series $V_1,V_2,\ldots, V_m$, the largest $p$ items are $V_{m-p+1}, V_{m-p+2},\ldots, V_m$. For each $q\in Q$, there is only one $T_{q,i_q}$ according to its definition.
As each $V_k$ has at most one $T_{j,i}$ with $V_k\le T_{j, i}\le T_{j, i}^*< \alpha_p V_k$, we have inequality
\begin{eqnarray*}
    \sum_{j\in Q} T_{j, i_j}^*< \alpha_p\cdot \sum_{(m-p+1)\le k\le m} V_k.
\end{eqnarray*}

By equation (\ref{H-eqn}), we have 

\begin{eqnarray*}
\frac{\beta^{m+1}-1}{\beta-1}T_0=H&\le& \sum_{j\in Q} T_{j, i_j}^*< \alpha_p\cdot \sum_{(m-p+1)\le k\le m} V_k\\   &=& \alpha_p\cdot \sum_{(m-p+1)\le k\le m} \beta^kT_0\\
&=& \alpha_p\cdot\beta^{m-p+1} T_0(1+\beta+\ldots+\beta^{p-1})\\
&=& \alpha_p\cdot\beta^{m-p+1} \cdot \frac{\beta^p-1}{\beta-1}\cdot T_0.
\end{eqnarray*}

Therefore, 
\begin{eqnarray*}
   \alpha_p&\ge&  \frac{(\beta^{m+1}-1)}{\beta^{m-p+1} \cdot (\beta^p-1)}=\frac{\beta^{m+1}-1}{\beta^{m+1}-\beta^{m-p+1}}\\
   &=&\frac{1-\frac{1}{\beta^{m+1}}}{1-\frac{1}{\beta^p}}
\end{eqnarray*}

Let $m=(z+1)p-1$. We have
\begin{eqnarray*}
   \alpha_p\ge  \frac{1-\frac{1}{\beta^{(z+1)p}}}{1-\frac{1}{\beta^p}}=1+\frac{1}{\beta^p}+\frac{1}{\beta^{2p}}+\ldots+\frac{1}{\beta^{zp}}.
\end{eqnarray*}

As $\beta=\alpha_p$, this proves (\ref{first-case-alphap}) of the lemma. 
We have
\begin{eqnarray*}
   1\ge \frac{1}{\alpha_p}++\frac{1}{\alpha_p^{p+1}}+\frac{1}{\alpha_p^{2p+1}}+\ldots+\frac{1}{\alpha_p^{zp+1}}.
\end{eqnarray*}

The number $\alpha_p$ is fixed in the beginning of this proof. Taking limit for $z\rightarrow+\infty$, we have 

\begin{eqnarray*}
   1\ge \frac{1}{\alpha_p}++\frac{1}{\alpha_p^{p+1}}+\frac{1}{\alpha_p^{2p+1}}+\ldots+\frac{1}{\alpha_p^{zp+1}}+\ldots.
\end{eqnarray*}

We consider the equation, 
\begin{eqnarray}
1&=&\frac{1}{x}+\frac{1}{x^{p+1}}+\frac{1}{x^{2p+1}}+\ldots+\frac{1}{x^{zp+1}}+\ldots\label{x-eqn}  \\
&=&\frac{1}{x}\cdot \frac{1}{1-\frac{1}{x^p}}= \frac{x^{p-1}}{x^p-1}
\end{eqnarray}

Thus, we have equation $x^p-x^{p-1}-1=0$. If $r_0>1$ is a root, then $r_0$ is also the root of equation~(\ref{x-eqn}).
The right side of equation~(\ref{x-eqn}) is strictly decreasing. 
We have $\alpha_p\ge r_0$.
\end{proof}

\subsection{The Case for Large Number of Processors $p$}

We derive a lower bound for the case $p$ is large. A special analysis for be given for the case $p=1$ in the next section.

\begin{theorem} For any function  $h(j, i)$, 
 if   Parallel-Framework(.)  has   $(p,\alpha_p)$-approximation,  then for any fixed $d>1$, $\alpha_p\ge 1+\frac{1+\ln (1+p)}{p}-\frac{d\ln\ln p}{p}$  for all large $p$.   
\end{theorem}

\begin{proof} 
By (\ref{first-case-alphap}) of Lemma~\ref{lowerbd-lemma}, 
we have
\begin{eqnarray*}
   \alpha_p\ge  1+\frac{1}{\alpha_p^p}.
\end{eqnarray*}



We will use the classical fact that $(1+\frac{1}{x})^{x}$ is increasing for all $x>1$, and $\lim_{x\rightarrow+\infty}(1+\frac{1}{x})^{x}=e\approx 2.71828$ (Euler's number). It can be found in most calculus textbooks.

Assume that $\alpha_p< 1+\frac{1+\ln (1+p)-d\ln\ln p}{p}$ with a fixed $d\in (1,+\infty)$. We have 
\begin{eqnarray*}
1+\frac{1}{\alpha_p^p}&\ge& 1+\frac{1}{(1+\frac{1+\ln (1+p)-d\ln\ln p}{p})^p}\\
&=& 1+\frac{1}{\left(1+\frac{1+\ln (1+p)-d\ln\ln p}{p}\right)^{\frac{p}{(1+\ln (1+p)-d\ln\ln p)}\cdot (1+\ln (1+p)-d\ln\ln p)}}\\
&>& 1+\frac{1}{e^{1+\ln (1+p)-d\ln\ln p}}\\
&=& 1+\frac{(\ln p)^d}{e(1+p)}>1+\frac{1+\ln (1+p)-d\ln\ln p}{p}>\alpha_p\ \ (for ~a~large~p).
\end{eqnarray*}
This brings a contradiction when $p$ is large.

\end{proof}

\subsection{The Case for Small Number of Processors $p$}

In this section, we give a lower for the case $p=1,2$. The case $p=1$ is important as it is related to single processor computation.  The case $p=2$ is the simplest parallel computation with two processors.

\begin{theorem}\label{lower-1-2-theorem} In the Parallel-Framework(.) model, for any function  $h(j, i)$, we have
\begin{enumerate}
    \item  
 if the parallel model has   $(1,\alpha_1)$-approximation,  then $\alpha_1\ge 2$.   
 \item 
 if the parallel model has   $(2,\alpha_3)$-approximation,  then  $\alpha_2\ge \frac{\sqrt{5}+1}{2}$.   
\end{enumerate}

\end{theorem}

\begin{proof} 
By Lemma~\ref{lowerbd-lemma}, we have the equation $x^p-x^{p-1}-1=0$ for the cases $p=1,2$. For $p=1$, $x=2$ is the only root. For $p=2$, $x=\frac{\sqrt{5}+1}{2}$ is the root greater than $1$. Therefore, we have $\alpha_1\ge 2$, and $\alpha_2\ge \frac{\sqrt{5}+1}{2}$.






\end{proof}









\begin{theorem} 
For any function  $h(j, i)$, 
 if   Parallel-Framework(.)  has   $(3,\alpha_3)$-approximation,  then we have $\alpha_3\ge r_3$, where $r_3=\frac{1}{3}+\sqrt[3]{\frac{29}{54}+\sqrt{\frac{31}{108}}} +\sqrt[3]{\frac{29}{54}-\sqrt{\frac{31}{108}}}\ge  1.46557$.   
\end{theorem}
\begin{proof}
By Lemma~\ref{lowerbd-lemma}, we have the equation $x^3-x^{2}-1=0$ for the cases $p=3$. With the transformation $x=y+\frac{1}{3}$,   it removes the quadratic term, and  becomes the Cardano's form:
\begin{eqnarray*}
    y^3-\frac{1}{3}y-\frac{29}{27}=0.
\end{eqnarray*}
We have root $y=\sqrt[3]{\frac{29}{54}+\sqrt{\frac{31}{108}}} +\sqrt[3]{\frac{29}{54}-\sqrt{\frac{31}{108}}}$ to satisfy that  $x$ is real number  greater than $1$.
It has root for $x$:
\begin{eqnarray*}
r_3=\frac{1}{3}+\sqrt[3]{\frac{29}{54}+\sqrt{\frac{31}{108}}} +\sqrt[3]{\frac{29}{54}-\sqrt{\frac{31}{108}}}\ge 1.46557.   
\end{eqnarray*}

\end{proof}

Using the numerical solutions, we have lower bounds when  $p$ goes from $4$ to $16$ below:
\begin{eqnarray*}
    \alpha_4&\ge& 1.38027,
    \alpha_5\ge 1.32471,
    \alpha_6\ge 1.28519,
    \alpha_7\ge 1.25542,
    \alpha_8\ge 1.23205,\\
    \alpha_9&\ge& 1.21314,
    \alpha_{10}\ge 1.19749,
    \alpha_{11}\ge 1.18427,
    \alpha_{12}\ge 1.17295,
    \alpha_{13}\ge 1.16311,\\
    \alpha_{14}&\ge& 1.15449,
    \alpha_{15}\ge 1.14685,
    \alpha_{16}\ge 1.14003.
\end{eqnarray*}

\section{Tight Lower Bounds  for $\alpha_p$ with Geometric $h(.,.)$}

In this section, we derive lower bound for $\alpha_p$ when $h(j,i)=T_0b_p^{ip+j}$ for some $b_p>1$. It matches the upper bound for each integer $p\ge 1$.

\begin{theorem}\label{lower-Geometric-thm}
Let function $h(j, i)=b_p^{ip+j}T_0$ for some $b_p>1$. If Parallel-Framework(.) has $(p,\alpha_p)$-approximation, then $\alpha_p\ge (1+\frac{1}{p})(p+1)^{1/p}$.
\end{theorem}

\begin{proof} We fix $p$ and $\alpha_p>1$ (by its definition). At phase $i$, processor $j$ uses $T_{j, i}=h(j,i)=b_p^{ip+j}T_0$ to control the number of  of iterations in GD($x_0, T_{j,i}$).

Define
\begin{eqnarray*}
T_{j, i}^*=\sum_{t=0}^{i} b_p^{tp+j}T_0=b_p^jT_0\sum_{t=0}^{i} b_p^{tp}=b_p^jT_0\cdot \frac{b_p^{(i+1)p}-1}{b_p^p-1}. 
\end{eqnarray*}

We note that $T_{j,t}=b_p^{tp+j}T_0$ is the number of steps in the $t$-th iteration. 
Let $T=b_p^{tp+j-1}+1$ with a large $t$. So, $T_{j,t}=b_p^{tp+j}$ is the least with $T\le T_{j,t}$. We have $T\le T_{j,t}<b_pT$. 

By Definition~\ref{App-def} and 
Lemma~\ref{monotonic-lemma}, 
we have 

\begin{eqnarray*}
\alpha_p T\ge &&T_{j,t}^*=b_p^jT_0\cdot \frac{b_p^{(t+1)p}-1}{b_p^p-1}\\
&=& \frac{b_p^{(t+1)p+j}T_0-b_p^jT_0}{b_p^p-1}\\
&=& \frac{b_p^{p+1}(b_p^{tp+j-1}T_0)-b_p^jT_0}{b_p^p-1}\\
&=& \frac{b_p^{p+1}(T-1)-b_p^jT_0}{b_p^p-1}\\
&=& \frac{b_p^{p+1}T-b_p^{p+1}-b_p^jT_0}{b_p^p-1}\\
&\ge&\frac{b_p^{p+1}T}{b_p^p-1}-\frac{b_p^{p+1}+b_p^jT_0}{b_p^p-1}
\end{eqnarray*}

We have

\begin{eqnarray*}
\alpha_p \ge \frac{b_p^{p+1}}{b_p^p-1}-\frac{b_p^jT_0}{T(b_p^p-1)}
\end{eqnarray*}

Let $f(x)=\frac{x^{p+1}}{x^p-1}$. 
Taking derivative, we have 

\begin{eqnarray*}
f(x)'=\frac{(p+1)x^p(x^p-1)-px^{p-1}\cdot x^{p+1}}{(x^p-1)^2} =\frac{x^{2p}-(p+1)x^p}{(x^p-1)^2}.   
\end{eqnarray*}

So, we let $x^p=p+1$ to have minimal $f(x)=\frac{(p+1)(p+1)^{1/p}}{p}=(1+\frac{1}{p})(p+1)^{1/p}$.

Therefore, 
\begin{eqnarray*}
\alpha_p &\ge& (1+\frac{1}{p})(p+1)^{1/p}-\frac{b_p^jT_0}{T(b_p^p-1)}\\
&\ge& (1+\frac{1}{p})(p+1)^{1/p}-\epsilon \ \ \ (for\ large\ T).
\end{eqnarray*}
Since both $p$ and $\alpha_p$ are fixed in the beginning of this proof and $\epsilon$ is arbitrarily close to zero, we have $\alpha\ge (1+\frac{1}{p})(p+1)^{1/p}$.
\end{proof}

\section{Arithmetically Simple Gradient Descent for Nonconvex Optimization}\label{alg-sec}

In this section, we present an arithmetically simple static gradient descent method for nonconvex optimization. The algorithm takes the iteration budget $T$ as input, which determines the total number of gradient descent iterations. The step size is computed using a denominator of the form $2^t$, where the integer $t$ is determined from $T$.

This design makes the algorithm particularly suitable for hardware implementation. Since every denominator is a power of two, division operations can be replaced by binary shift operations, eliminating expensive floating-point division. Moreover, the algorithm avoids square-root computations altogether. These arithmetic simplifications make the proposed method attractive for hardware accelerators and chip implementations.

The convergence analysis in this section also explains the motivation for the proposed parallel framework. Because the convergence guarantee of the static gradient method depends on selecting an appropriate iteration budget $T$, the parallel framework searches for a suitable value of $T$ adaptively while preserving the simplicity of the underlying static algorithm.


\subsection{Notations for Gradient Descent}

 A vector in $\mathbb{R}^m$ is $(a_1, a_2,\ldots, a_m)$ with $a_i\in \mathbb{R}$ for $i=1,2,\cdots, m$. The inner product between two vectors $V=(v_1,v_2,\ldots, v_m)$ and $U=(u_1,u_2,\ldots, u_m)$ is denoted by $\langle U, V\rangle=\sum_{i=1}^m u_iv_i$. The length of a vector $V=(v_1,v_2,\ldots, v_m)$  is denoted by $\Vert V\Vert=\sqrt{v_1^2+v_2^2+\ldots+v_m^2}$. For a differentiable function $F(x_1,x_2,\cdots, x_m):\mathbb{R}^m\rightarrow \mathbb{R}$, its gradient at a point $(x_1,x_2,\cdots, x_m)$ is $\bigtriangledown F(x_1,x_2,\cdots, x_m)=\left(\frac{\partial F(x_1,x_2,\cdots, x_m)}{\partial x_1}, \frac{\partial F(x_1,x_2,\cdots, x_m)}{\partial x_2},\ldots,\frac{\partial F(x_1,x_2,\cdots, x_m)}{\partial x_m}\right)$. In the rest of this paper, let $x^*\in \mathbb{R}^m$ be a point with $F(x^*)=\inf_x\{F(x)\}$ if $\inf_x\{F(x)\}>-\infty$.  The expectation on a random variable $\xi$ is expressed $\Expect_{\xi}(.)$. For example, a stochastic gradient $G(\xi, x)$ for function $F(x)$ may satisfy the condition $\Expect_{\xi}(\xi, x)=\bigtriangledown F(x)$, which is often assumed in many SGD algorithms. 
In this section, we give some theoretical results about the rate of convergence. The following two conditions are often assumed for non-convex optimization.

\begin{itemize}
    \item 
    {\bf $L$-Lipschitz} smoothness: $\Vert \bigtriangledown(F(x))-\bigtriangledown(F(y))\Vert\le L\Vert x-y\Vert$.

\item $F^*=\inf_{x} F(x)>-\infty$

\end{itemize}
Let $C_L^1$ be the class of functions that are $L$-Lipschitz smooth.
The following Lemma~\ref{basic-lemma}, which is often mentioned in existing publications, can be easily proven by $L$-Lipschitz condition and Taylor expansion (See~\cite{Lan2020}). 

\begin{lemma}\label{basic-lemma}
    Let $F(x_1,\cdots, x_d)$ be a function $\mathbb{R}^d\rightarrow \mathbb{R}$ in $C_L^1$, we have $F(x)\le F(y)+(\bigtriangledown F(y), x-y)+\frac{L}{2}\Vert x-y\Vert^2$.
\end{lemma}

The stochastic gradient is controlled by three parameters $\lambda, \sigma_0$, and $\sigma_1$. It is given in Definition~\ref{condition}.

\begin{definition}\label{condition} A $(\lambda,\sigma_0,\sigma_1)$-stochastic gradient $G(\xi,x)$ for $F(x):\mathbb{R}^m\rightarrow\mathbb{R}$ is that $\xi$ is a random variable and $G(\xi, x)$ is an approximation for $\bigtriangledown F(x)$ satisfying the conditions:
 
\begin{enumerate}
    \item\label{condition1-def} $\langle\Expect_{\xi}(G(\xi, x)),\bigtriangledown F(x)\rangle \ge \lambda\Vert \bigtriangledown F(x)\Vert^2$ for some $\lambda\in (0,+\infty)$, and
    \item\label{condition2-def} 
    $\Expect_{\xi}(\Vert \bigtriangledown F(x)-G(\xi, x)\Vert^2)\le\sigma_0^2+\sigma_1^2\Vert \bigtriangledown F(x)\Vert^2$ for some $\sigma_0,\sigma_1\in [0,+\infty)$.
\end{enumerate}

\end{definition}

A standard stochastic model, which is broadly used in stochastic gradient descent,  is the special case with $\lambda=1$ and $\sigma_1=0$. Our stochastic model is more general, and fits the convergence analysis for our  algorithm.

\begin{lemma}\label{stochastic-bound-lemma}
    Assume $F(x)$ and $G(\xi, x)$ satisfy 
        $\Expect_{\xi}(\Vert \bigtriangledown F(x)-G(\xi, x)\Vert^2)\le\sigma_0^2+\sigma_1^2\Vert \bigtriangledown F(x)\Vert^2$ for some $\sigma_0,\sigma_1\in [0,+\infty)$.
    Then $\Expect_{\xi}(\Vert G(\xi, x)\Vert^2)\le 2\sigma_0^2+ (2+2\sigma_1^2)\Vert \bigtriangledown F(x)\Vert^2$.
\end{lemma}

\begin{proof} We have inequalities
\begin{eqnarray*}
    &&\Vert G(\xi, x)\Vert^2=\langle G(\xi, x),G(\xi, x)\rangle\\
    &=&\langle G(\xi, x)-\bigtriangledown F(x)+\bigtriangledown F(x),G(\xi, x)-\bigtriangledown F(x)+\bigtriangledown F(x)\rangle\\
    &\le& \Vert G(\xi, x)-\bigtriangledown F(x)\Vert^2+\Vert \bigtriangledown F(x)\Vert^2+2\langle G(\xi, x)-\bigtriangledown F(x), \bigtriangledown F(x)\rangle\\
    &\le& \Vert G(\xi, x)-\bigtriangledown F(x)\Vert^2+\Vert \bigtriangledown F(x)\Vert^2+\Vert G(\xi, x)-\bigtriangledown F(x)\Vert^2+ \Vert\bigtriangledown F(x)\Vert^2\\
    &&(by~Cauchy-Schwarz~inequality)\\    
    &=& 2\Vert G(\xi, x)-\bigtriangledown F(x)\Vert^2+2\Vert \bigtriangledown F(x)\Vert^2.
\end{eqnarray*}
Therefore,

\begin{eqnarray*}
    \Expect_{\xi}(\Vert G(\xi, x)\Vert^2)&\le&\Expect( 2\Vert G(\xi, x)-\bigtriangledown F(x)\Vert^2+2\Vert \bigtriangledown F(x)\Vert^2)\\
&=&2\Expect( \Vert G(\xi, x)-\bigtriangledown F(x)\Vert^2)+2\Vert \bigtriangledown F(x)\Vert^2\\
&\le&2(\sigma_0^2+\sigma_1^2\Vert\bigtriangledown F(x)\Vert^2)+2\Vert \bigtriangledown F(x)\Vert^2\\
&=&2\sigma_0^2+(2+2\sigma_1^2)\Vert\bigtriangledown F(x)\Vert^2.
\end{eqnarray*}

\end{proof}

\subsection{A Static Gradient Method}

We give a static gradient descent algorithm in this section. The learning rate is computed based on  one of the parameters.

\begin{definition}
  A gradient descent method is {\it arithmetically simple} if the operations are limited to  $+,-, \times$, and  division $x/y$ with $y=2^t$ for some integer $t$.
\end{definition}

We present a version of SGD that is arithmetically simple. When the stochastic gradient oracle $G(\cdot)$ is treated as a black box, the algorithm requires neither floating-point division nor square-root computations.

\vskip 20pt
{\bf Algorithm} SGD$(G(.,.), \eta,  x_0, t, T)$

Input: 

\begin{itemize}
\item
$G(\xi, x):\mathbb{R}^m\rightarrow \mathbb{R}$ is an stochastic approximation for $\bigtriangledown F(x)$,

    \item $\eta\in (0,+\infty)$, 



    \item $x_0\in \mathbb{R}^m$ is the start point, 

    \item $t$ is an integer to control rate, 

    \item $T$ is for the number of steps 
\end{itemize}

Steps:

\begin{enumerate}[label=\arabic*.]


 \item Let $x_{1}=x_0$



    \item Let $s_t=2^t$
   
    \item Let $j=1$

\item while $i\le T$

   \item $\{$

\item\qquad $g_i=G(\xi_{j},x_{j})$
       
    \item \qquad
     $x_{j+1}=x_{j}-\frac{\eta}{ s_t}\cdot g_i$ 

    \item \qquad  $j=j+1$

\item $\}$


\end{enumerate}

{\bf End of Algorithm}

\subsection{Convergence at $(\lambda,\sigma_0,\sigma_1)$-Stochastic Model}

The convergence of the algorithm at $(\lambda,\sigma_0,\sigma_1)$-Stochastic Model is proven in this section. With $t=\ceiling{\ceiling{\log_2T}/2}$, it converges to a stationary point with rate $\Omega\left(\frac{1}{ \sqrt{T}}\right)$.



Lemma~\ref{up-adjustment1-lemma} derives an upper bound by summing the inequalities in Lemma~\ref{basic-lemma} for the gradient method. It then follows that at least one iterate generated by SGD(.) has a gradient whose expected norm is close to zero. As the step size $s_t$ does not change with a given $T$, the convergence analysis turns out to be simple.

\vskip 20pt

\begin{lemma}\label{up-adjustment1-lemma}
Assume $F(x)$ is $L$-Lipschitz  smooth and $G(\xi, x)$ satifies the condition in Definition~\ref{condition}.
Then
  \begin{eqnarray*}
&&\sum_{j=1}^T\Expect\left(\frac{\eta\lambda}{ s_t}\Vert \bigtriangledown F(x_{j})\Vert^2\right)-\frac{\eta^2 L}{ 2s_t^{2}}\sum_{j=1}^T \Expect (\Vert G(\xi_{j}, x_{j})\Vert^2)\le F(x_1)-F(x^*)
   \end{eqnarray*}   
\end{lemma}

\begin{proof}
   As $F(x)$ is $L$-Lipschitz smooth, by Lemma~\ref{basic-lemma}, we have 
   \begin{eqnarray*}
   F(x_{j+1})&\le& F(x_{j})+\langle\bigtriangledown F(x_{j}), x_{j+1}-x_{j}\rangle+\frac{L}{ 2}\Vert x_{j+1}-x_{j}\Vert^2\\
   &=& F(x_{j})-\frac{\eta}{s_t }\langle \bigtriangledown F(x_{j}), G(\xi_j,x_{j})\rangle+\frac{L}{2}\Vert x_{j+1}-x_{j}\Vert^2\\
   &\le & F(x_{j})-\frac{\eta}{ s_t}\langle \bigtriangledown F(x_{j}), G(\xi_{j},x_{j})\rangle+\frac{\eta^2 L}{ 2s_t^2}\Vert G(\xi_{j},x_{j})\rangle\Vert^2\\
   &\le & F(x_{j})-\frac{\eta\lambda}{ s_t}\cdot \Vert\bigtriangledown F(x_{j})\Vert^2+\frac{\eta^2 L}{2s_t^2}\Vert G(\xi_j,x_{j})\Vert^2\ \ \ (by\ Condition~\ref{condition1-def}\ in\  Definition~\ref{condition})\\
   \end{eqnarray*}

Thus,
   \begin{eqnarray*}
&&\left(\frac{\eta\lambda}{s_t}(\Vert\bigtriangledown F(x_{j})\Vert^2\right)-\frac{\eta^2 L}{ 2s_t^{2}}\Vert G(\xi_{j},x_{j})\Vert^2\le F(x_{j})-F(x_{j+1}). \end{eqnarray*}

We have    \begin{eqnarray*}
&&\sum_{j=1}^T\Expect\left(\frac{\eta\lambda}{ s_t}\Vert \bigtriangledown F(x_{j})\Vert^2\right)-\sum_{j=1}^T \Expect (\frac{\eta^2 L}{ 2s_t^{2}}(\Vert G(\xi_{j}, x_{j})\Vert^2)\le F(x_1)-F(x^*).
   \end{eqnarray*}   
Thus,
   \begin{eqnarray*}
&&\sum_{j=1}^T\Expect\left(\frac{\eta\lambda}{ s_t}\Vert \bigtriangledown F(x_{j})\Vert^2\right)-\frac{\eta^2 L}{2s_t^{2}}\sum_{j=1}^T  \Expect (\Vert G(\xi_{j}, x_{j})\Vert^2)\le F(x_1)-F(x^*).
   \end{eqnarray*}   
\end{proof}

Lemma~\ref{up-adjustment-lemma} shows that one of the iterates generated by SGD(.) has a gradient whose expected norm is close to zero. Consequently, SGD(.) converges to a stationary point.

\begin{lemma}\label{up-adjustment-lemma}
Assume $F(x)$ is $L$-Lipschitz-smooth and $G(\xi, x)$ satisfies the condition in Definition~\ref{condition}. Assume that $s_t\in[\sqrt{T}, 4\sqrt{T}]$ and $T$ satisfy the conditions:

\begin{eqnarray}
\frac{\eta L(1+\sigma_1^2)}{s_t \lambda}&\le& \frac{1}{2}\label{first0-ineqn}
\end{eqnarray}
Then $\min_i\Expect(\Vert \bigtriangledown F(x_i)\Vert^2)\le \frac{U(\eta,\sigma_0, \lambda, L)}{\sqrt{T}}$, 
where 
\begin{eqnarray}
U(\eta,\sigma_0, \lambda, L)=\left(\frac{4}{\eta \lambda}(F(x_0)-F(x^*))+\frac{2\eta\sigma_0^2 L}{ \lambda}\right).\label{def-U-eqn}  
\end{eqnarray}
\end{lemma}

\begin{proof} By Lemma~\ref{up-adjustment1-lemma}, we have
   \begin{eqnarray*}
&&\sum_{j=1}^T\Expect\left(\frac{\eta(\lambda)}{s_t}\Vert \bigtriangledown F(x_{j})\Vert^2\right)-\frac{\eta^2 L}{ 2s_t^{2}}\sum_{j=1}^T \Expect (\Vert G(\xi_{j}, x_{j})\Vert^2)\le F(x_1)-F(x^*)
   \end{eqnarray*}   

By Lemma~\ref{stochastic-bound-lemma}, we have
   \begin{eqnarray*}
&&\sum_{j=1}^T\Expect\left(\frac{\eta \lambda}{ s_t}\Vert \bigtriangledown F(x_{j})\Vert^2\right)-\sum_{j=1}^T \frac{\eta^2 L}{ 2s_t^{2}}\Expect(2\sigma_0^2+ (2+2\sigma_1^2)\Vert \bigtriangledown F(x)\Vert^2)\le F(x_1)-F(x^*)
   \end{eqnarray*}

   \begin{eqnarray*}
&&\sum_{j=1}^T\frac{\eta\lambda}{s_t}\left(1-\frac{\eta L(1+\sigma_1^2)}{s_t \lambda}\right)\Expect(\Vert \bigtriangledown F(x_{j})\Vert^2)-\sum_{j=1}^T \frac{\eta^2 \sigma_0^2 L}{s_t^{2}}
\le F(x_1)-F(x^*)
   \end{eqnarray*}

   \begin{eqnarray*}
&&\sum_{j=1}^T\left(\frac{\eta\lambda}{ 2s_t}\Expect(\Vert \bigtriangledown F(x_{j})\Vert^2-\frac{\eta^2 \sigma_0^2 L}{s_t^{2}}\right)\le F(x_1)-F(x^*)\ \ \ (by~inequality~(\ref{first0-ineqn}))
   \end{eqnarray*}

 We have

   \begin{eqnarray*}
&&\sum_{j=1}^T \frac{\eta\lambda}{ 2s_t}\Expect(\Vert \bigtriangledown F(x_{j})\Vert^2)- \frac{\sigma_0^2\eta^2L T}{s_t^{2}}\le F(x_1)-F(x^*)
   \end{eqnarray*}

   \begin{eqnarray*}
\frac{\eta\lambda}{ 2s_t}\cdot T\min_{1\le j\le T} \Expect(\Vert \bigtriangledown F(x_{j})\Vert^2 )
&\le& \sum_{j=1}^T\frac{\eta\lambda}{2s_t}\Expect(\Vert \bigtriangledown F(x_{j})\Vert^2)\\
&\le& (F(x_1)-F(x^*))+\frac{\sigma_0^2\eta^2L T }{s_t^{2}}.
   \end{eqnarray*} 

We have

   \begin{eqnarray*}
&&\min_{1\le j\le T} \Expect(\Vert \bigtriangledown F(x_{j})\Vert^2 )
\le \frac{s_t}{T\eta \lambda}\left(F(x_1)-F(x^*)\right)+\frac{2\eta\sigma_0^2 L }{ s_t\lambda}\\
&\le& \frac{4}{\sqrt{T}\eta \lambda}(F(x_1)-F(x^*))+\frac{2\eta\sigma_0^2 L }{ \sqrt{T}\lambda}\\
&\le& \frac{1}{\sqrt{T}}\left(\frac{4}{ \eta \lambda}(F(x_1)-F(x^*))+\frac{2\eta\sigma_0^2 L }{\lambda}\right)\\
&=&\frac{1}{\sqrt{T}}\left(\frac{4}{\eta \lambda}(F(x_0)-F(x^*))+\frac{2\eta\sigma_0^2 L }{\lambda}\right).
\label{final-bound}
   \end{eqnarray*}    
\end{proof}

Assume $F(.)$ and it stochastic gradient  $G(\xi,x)$ satisfy the conditions in Definition~\ref{condition}. 
Theorem~\ref{main-thm} shows that the gradient descent algorithm SGD(.) converges to a stationary point at rate $\Omega\left(\frac{1}{\sqrt{T}}\right)$.

\begin{theorem}\label{main-thm}
Suppose $F(.)$ is in $\mathbb{C}_L^1$ and $\inf_x F(x)>-\infty$. Function $G(\xi,x)$ satisfies the conditions in Definition~\ref{condition}. Assume that integers $T$ and $t$ satisfy the conditions
\begin{eqnarray}
T&\ge& \left(\frac{2\eta L(1+\sigma_1^2)}{  \lambda}\right)^2\label{condition-T-ineqn}\\
t&=&\ceiling{\ceiling{\log_2T}/2}\label{condition-t-ineqn2}.    
\end{eqnarray}
 Then the algorithm SGD$(G(.,.), \eta,  x_0, t, T)$ is arithmetically simple and has   $\min_{1\le i\le T}\Expect(\Vert \bigtriangledown F(x_i)\Vert^2)\le \frac{U(\eta,\sigma_0, \lambda, L)}{ \sqrt{T}}$, where $U(.)$ is given at equation (\ref{def-U-eqn}).
\end{theorem}

\begin{proof}
Let $T$ be the number of steps to run. We select $a=\ceiling{\log_2 T}$, which is the number of bits if $T$ is in binary format (for example, number $7$ has binary format $111$, and $\ceiling{\log_2 7}=3$). We have $T\le 2^a\le 2T$. Let $t=\ceiling{\frac{a}{ 2}}\le \frac{a}{ 2}+1$. We have $\sqrt{T}\le 2^\frac{a}{ 2}\le 2^t\le 2\sqrt{2^a}\le 2\sqrt{2T}< 4\sqrt{T}$. Thus, $s_t=2^t\in [\sqrt{T}, 4\sqrt{T}]$.
With the condition $T\ge \left(\frac{2\eta L(1+\sigma_1^2)}{  \lambda}\right)^2$, we have
$s_t\ge \sqrt{T}\ge \left(\frac{2\eta L(1+\sigma_1^2)}{  \lambda}\right)$. So, inequality (\ref{first0-ineqn}) is satisfied.

Run 
GD$(G(.), \eta, x_0, t, T)$. We have $\min_{1\le i\le T}\Expect(\Vert \bigtriangledown F(x_i)\Vert^2)\le \frac{U(\eta,\sigma_0, \lambda, L)}{\sqrt{T}}$ by Lemma~\ref{up-adjustment-lemma}.


    
\end{proof}

\vskip 20pt
{\bf Algorithm} Static1-SGD$(   x_0,T)$

Input: 
\begin{itemize}
    \item $x_0$ is the start point

    \item $T$ is the number of iterations
\end{itemize}

Steps:

\begin{enumerate}[label=\arabic*.]

\item 
Assign to $t$ as equation (\ref{condition-t-ineqn2}) 
\item
SGD$(G(.,.), \eta,  x_0, t, T)$

\end{enumerate}

{\bf End of Algorithm}

\vskip 20pt

\begin{corollary}\label{SGD-cor1} Let $\delta\in (0,1)$.
Suppose $F(.)$ is in $\mathbb{C}_L^1$ and $\inf_x F(x)>-\infty$. Function $G(\xi,x)$ satisfies the conditions in Definition~\ref{condition}. Assume $T$ satisfies (\ref{condition-T-ineqn}). Then with probability at least $1-\delta$, the algorithm Static1-SGD$(G(.,.), \eta,  x_0, t, T)$ is arithmetically simple and  has  $\min_i (\Vert \bigtriangledown F(x_i)\Vert^2)\le \frac{U(\eta,\sigma_0, \lambda, L)}{ \delta\sqrt{T}}$.
\end{corollary}

\begin{proof}
It follows Theorem~\ref{main-thm} and Markov inequality $\Prob(X\ge \frac{\Expect(X)}{\delta})\le \delta$. The proof of Theorem~\ref{main-thm} also shows that $t$ is computed via a arithmetically simple way.
    
\end{proof}

\subsection{Faster Convergence in $(\lambda,0, \sigma_1)$ Stochastic Model}

In this section we show a faster adaptive gradient descent analysis in $(\lambda,0, \sigma_1)$ stochastic Model. It is convergence rate is almost linear, and  faster than the general $(\lambda,\sigma_0, \sigma_1)$ stochastic Model.

\begin{definition}\label{condition2} A $(\lambda, 0,\sigma_1)$-stochastic gradient $G(\xi,x)$ for $F(x):\mathbb{R}^m\rightarrow\mathbb{R}$ is that $\xi$ is a random variable and $G(\xi, x)$ is an approximation for $\bigtriangledown F(x)$ satisfying the conditions:
 
\begin{enumerate}
    \item $\langle\Expect_{\xi}(G(\xi, x)),\bigtriangledown F(x)\rangle \ge \lambda\Vert \bigtriangledown F(x)\Vert^2$ for some $\lambda\in (0,+\infty)$, and
    \item\label{condition2b-def} 
    $\Expect_{\xi}(\Vert \bigtriangledown F(x)-G(\xi, x)\Vert^2)\le\sigma_1^2\Vert \bigtriangledown F(x)\Vert^2$ for some $\sigma_1\in [0,+\infty)$.
\end{enumerate}
   
\end{definition}

\begin{lemma}\label{up-adjustment-lemma2}
Assume $F(x)$ is $L$-smooth and $G(\xi, x)$ satisfies the condition in Definition~\ref{condition2}. 
Assume $T$ and $s_t$ satisfy the following conditions:

\begin{eqnarray}
\frac{2\eta (2+2\sigma_1^2) L}{ s_t\lambda}&\le& \frac{1}{ 2}\label{second-lemma-ineqn}
\end{eqnarray}
Then
\begin{eqnarray*}
&&\min_{1\le j\le T} \Expect(\Vert \bigtriangledown F(x_{j})\Vert^2 )
\le \frac{2s_t}{T\eta\lambda}(F(x_0)-F(x^*)).
   \end{eqnarray*}  
\end{lemma}

\begin{proof}
By Lemma~\ref{up-adjustment1-lemma}, we have    \begin{eqnarray*}
&&\sum_{j=1}^T\Expect\left(\frac{\eta\lambda}{ s_t}\Vert \bigtriangledown F(x_{j})\Vert^2\right)-\frac{2\eta^2 L}{ s_t^{2}}\sum_{j=1}^T \Expect (\Vert G(\xi_{j}, x_{j})\Vert^2)\le F(x_1)-F(x^*).
   \end{eqnarray*} 
By Lemma~\ref{stochastic-bound-lemma}, we have 
   \begin{eqnarray*}
&&\sum_{j=1}^T\Expect\left(\frac{\eta\lambda}{ s_t}\Vert \bigtriangledown F(x_{j})\Vert^2\right)-\sum_{j=1}^T \frac{2\eta^2 L}{s_t^{2}}((2+2\sigma_1^2)\Vert \bigtriangledown F(x_{j})\Vert^2)\le F(x_1)-F(x^*).
   \end{eqnarray*}

   \begin{eqnarray*}
&&\sum_{j=1}^T\left(\frac{\eta\lambda}{s_t}-\frac{2\eta^2 (2+2\sigma_1^2) L}{ s_t^{2}}\right)\Expect(\Vert \bigtriangledown F(x_{j})\Vert^2)
\le F(x_1)-F(x^*).
   \end{eqnarray*} 

   \begin{eqnarray*}
&&\sum_{j=1}^T\frac{\eta\lambda}{s_t}\left(1-\frac{2\eta (2+2\sigma_1^2) L}{ s_t\lambda}\right)\Expect(\Vert \bigtriangledown F(x_{j})\Vert^2)
\le F(x_1)-F(x^*).
   \end{eqnarray*}

By inequality~(\ref{second-lemma-ineqn}), we have 
   \begin{eqnarray*}
&&\sum_{j=1}^T\left(\frac{\eta\lambda}{ 2s_t}\Expect(\Vert \bigtriangledown F(x_{j})\Vert^2\right)\le F(x_1)-F(x^*).
   \end{eqnarray*}

 We have

   \begin{eqnarray*}
T\cdot \left(\frac{\eta\lambda}{2s_t}\right)\min_{j} \Expect(\Vert \bigtriangledown F(x_{j})\Vert^2 )\le \sum_{j=1}^T(\frac{\eta}{ 2s_t})\Expect(\Vert \bigtriangledown F(x_{j})\Vert^2)\le (F(x_1)-F(x^*)).
   \end{eqnarray*} 

This brings inequality

   \begin{eqnarray*}
&&\min_{1\le j\le T} \Expect(\Vert \bigtriangledown F(x_{j})\Vert^2 )
\le \frac{2s_t}{ T\eta\lambda}(F(x_1)-F(x^*))=\frac{2s_t}{ T\eta\lambda}(F(x_0)-F(x^*)).\label{final2-bound}
   \end{eqnarray*}    
\end{proof}

Assume $F(.)$ and it stochastic gradient  $G(\xi,x)$ satisfy the conditions in Definition~\ref{condition2}. 
Theorem~\ref{main2-thm} shows that the gradient descent algorithm SGD(.) converges to a stationary point at rate $\Omega\left(\frac{1}{T^{1-\beta_0}}\right)$ for any $\beta_0\in (0,1)$. The parameter $t$ depends on parameter $\beta_0$.

\begin{theorem}\label{main2-thm}Let $\beta_0\in (0,1)$.
Suppose $F(.)$ is in $\mathbb{C}_L^1$ and $\inf_x F(x)>-\infty$. Function $G(\xi,x)$ satisfies the conditions in Definition~\ref{condition2}.  Assume 
\begin{eqnarray}
T&\ge& \left(\frac{2\eta L(1+\sigma_1^2)}{  \lambda}\right)^\frac{1}{ \beta_0}\label{condition-T2-ineqn}\\ t&=&\ceiling{\ceiling{\log_2T}\cdot \beta_0}\label{condition-t2-ineqn}.    
\end{eqnarray}
 Then the algorithm SGD$(G(.,.), \eta,  x_0, t, T)$ is arithmetically simple and  has 
 
 \begin{eqnarray}
\min_{1\le j\le T} \Expect(\Vert \bigtriangledown F(x_{j})\Vert^2 )
\le \frac{V(\eta,\lambda,L)}{T^{1-\beta_0}},\label{final3-bound}
   \end{eqnarray}   
   
 where $V(\eta,\lambda,L)=\frac{2}{\eta\lambda}(F(x_0)-F(x^*))$.
\end{theorem}

\begin{proof}Let $T$ be the number of steps to run. We select $a=\ceiling{\log_2 T}$. We have $T\le 2^a\le 2T$. Let $t=\ceiling{a\beta_0}\le {a\beta_0}+1$. We have $T^{\beta_0}\le 2^{a\beta_0}\le 2^t\le 2^{a\beta_0+1}=2\cdot (2^a)^{\beta_0}\le 2\cdot (2T)^{\beta_0}\le 4 T^{\beta_0}$. Thus, $s_t=2^t\in [T^{\beta_0}, 4T^{\beta_0}]$.
With the condition $T\ge \left(\frac{2\eta L(1+\sigma_1^2)}{  \lambda}\right)^\frac{1}{ \beta_0}$, we have
$s_t=2^t\ge T^{\beta_0}\ge \left(\frac{2\eta L(1+\sigma_1^2)}{  \lambda}\right)$. So, inequality (\ref{second-lemma-ineqn}) is satisfied.

Run the algorithm with SGD$(G(.), \eta,  x_0, t, T)$. By Lemma~\ref{up-adjustment-lemma2}, we get inequality (\ref{final3-bound}).

\end{proof}

\vskip 20pt
{\bf Algorithm} Static2-SGD$(   x_0,T)$

Input: 
\begin{itemize}
    \item $x_0$ is the start point

    \item $T$ is the number of iterations
\end{itemize}

Steps:

\begin{enumerate}[label=\arabic*.]

\item 
Assign to $t$ as equation (\ref{condition-t2-ineqn}) 
\item
SGD$(G(.,.), \eta,  x_0, t, T)$

\end{enumerate}

{\bf End of Algorithm}

\vskip 20pt

\begin{corollary}\label{SGD-cor}
Let $\delta\in (0,1)$. Let $\beta_0\in (0,1)$.
Suppose $F(.)$ is in $\mathbb{C}_L^1$ and $\inf_x F(x)>-\infty$. Function $G(\xi,x)$ satisfies the conditions in Definition~\ref{condition2}.  Assume $T$ satisfies condition (\ref{condition-T2-ineqn}).
Then with probability at least $1-\delta$, the algorithm Static2-SGD$(G(.,.), \eta,  x_0, t, T)$ is arithmetically simple and  has \begin{eqnarray*}
&&\min_{1\le j\le T} (\Vert \bigtriangledown F(x_{j})\Vert^2 )
\le \frac{V(\eta,\lambda,L)}{ \delta T^{1-\beta_0}}.\label{final2b-bound}
   \end{eqnarray*}  
\end{corollary}

\begin{proof}
 It follows Theorem~\ref{main2-thm} and Markov inequality $\Prob(X\ge \frac{\Expect(X)}{\delta})\le \delta$.  
 The proof of Theorem~\ref{main2-thm} also shows that $t$ is computed via a arithmetically simple way.
\end{proof}






\subsection{Why Do Static Gradient Methods Need a Parallel Framework?}

Theorems~\ref{main-thm} and~\ref{main2-thm} require Conditions~(\ref{condition-T-ineqn}) and~(\ref{condition-T2-ineqn}), respectively, to guarantee convergence. In practice, however, the parameters $L$, $\eta$, and $\sigma_1$ are typically unknown or difficult to estimate accurately. Consequently, selecting an appropriate iteration budget $T$ in advance is a challenging task. Our parallel framework addresses this difficulty by allowing multiple processors to search over a geometric sequence of candidate values $T_{j,i}$ simultaneously until one of them satisfies the required convergence conditions.

To illustrate the motivation, suppose a single processor tests the iteration budgets
\[
1,2,,2^2,,\ldots,,2^m,,\ldots,
\]
and assume that the smallest satisfactory choice is $T=2^m+1$. Before reaching the first candidate that is at least $T$, namely $2^{m+1}$, the processor must execute
\[
1+2+\cdots+2^m=2^{m+1}-1=2T-3
\]
iterations. Thus, a substantial amount of computation is wasted before identifying a suitable iteration budget. Our parallel framework significantly reduces this overhead by distributing the search across multiple processors. Theorem~\ref{lower-Geometric-thm} establishes a nearly tight lower bound on the unavoidable gap between the target iteration budget $T$ and the cumulative number of iterations $T_{j,i}^*$ executed before reaching it.



\section{Better Adaptivity via the Parallel Framework}

In this section, we apply the parallel framework to static gradient descent, thereby improving the adaptivity of the gradient descent algorithm. Both conditions, (\ref{condition-T-ineqn}) and (\ref{condition-T2-ineqn}), rely on choosing the parameter $T$ to be sufficiently large.

\begin{definition}
    Let $GD(x_0, T)$ be a gradient descent method for a function $F(x)$. The output of $GD(x_0, T)$ is the list  $Z=\langle x_1, x_2,\ldots,x_{\floor{T}}\rangle$  generated in its $\floor{T}$ iterations. It is denoted by $Z=GD(x_0, T)$.
\end{definition}


\begin{definition}
    Let GD$(x_0, T)$ be a gradient descent method for a function $F(x)$. A list of elements $\langle x_1,x_2,\ldots, x_{\floor{T_{j,i}}}\rangle$ is from Parallel-Framework(GD(.),.) at $\langle j,i\rangle$ if $\langle x_1,\ldots, x_{\floor{T_{j,i}}}\rangle$ is the output  of  GD$(x_0, T_{j, i})$.
\end{definition}

\begin{proposition}\label{embeding-prop}
 Let GD($x_0, T)$ be a gradient.
 Assume that Parallel-Framework$(GD(.),.)$ has $(p,\alpha_p)$-approximation.  Then for any $T\ge T_0$,
 Parallel-Framework$(GD(.))$ executes
 $GD(x_0, T_{j,i})$ satisfying $T\le T_{j,i}\le T_{j,i}^*\le \alpha_p T$.  
\end{proposition}

\begin{proof} 
By the condition of $(p, \alpha_p)$-approximation, we have $T\le T_{j,i}\le T_{j,i}^*\le \alpha_p T$. 
It follows from Theorem~\ref{main2b-thm}.
\end{proof}



We embed the first static gradient descent to Parallel framework and have  Theorem~\ref{embed-thm1} about its convergence. It shows that some processor $j$ generates a list of points at a stage $i$  converging to a stationary points with a high probability.

\begin{theorem}\label{embed-thm1} 
Suppose $F(.)$ is in $\mathbb{C}_L^1$ and $\inf_x F(x)>-\infty$. Function $G(\xi,x)$ satisfies the conditions in Definition~\ref{condition2}.  
Assume that $T$ satisfies (\ref{condition-T-ineqn}).
Then with probability at least $1-\delta$, Parallel-Framework(Static1-SGD(.,.),  $ x_0, h(.,.), T_0,p)$ generates $\langle x_1, x_2,\ldots, x_{\floor{T_{j,i}}}\rangle$ at $\langle j, i\rangle$

\begin{eqnarray*}
&&\min_{1\le j\le \floor{T_{j,i}}} (\Vert \bigtriangledown F(x_{j})\Vert^2 )
\le \frac{U(\eta,\sigma_0, \lambda, L)}{\delta \sqrt{T}}.\label{final2d-bound}
   \end{eqnarray*}   after running $T_{j,i}^*$ iterations with $T\le T_{j,i}\le T_{j,i}^*\le \alpha_p T$.
\end{theorem}

\begin{proof}By Proposition~\ref{embeding-prop}, Parallel-Framework(Static1-SGD(.,.),  $ x_0, h(.,.), T_0,p)$ executes Static1-SGD$(x_0, T_{j,i})$ that generates $\langle x_1, x_2,\ldots, x_{\floor{T_{j,i}}}\rangle$ at $\langle j, i\rangle$ and has $T\le T_{j,i}\le T_{j,i}^*\le \alpha_p T$. 
It follows from Corollary~\ref{SGD-cor1}.
    
\end{proof}

We embed the gradient descent Static2-SGD(.) to Parallel framework in Theorem~\ref{embed-thm2}. It has a faster convergence.

\begin{theorem}\label{embed-thm2} Let $\delta\in (0,1)$. Let $\beta_0\in (0,1)$.
Suppose $F(.)$ is in $\mathbb{C}_L^1$ and $\inf_x F(x)>-\infty$. Function $G(\xi,x)$ satisfies the conditions in Definition~\ref{condition2}.  
Assume that $T$ satisfies  (\ref{condition-T2-ineqn}).
Then with probability at least $1-\delta$, Parallel-Framework(Static2-SGD(.,.),  $ x_0, h(.,.), T_0,p)$ generates $\langle x_1, x_2,\ldots, x_{\floor{T_{j,i}}}\rangle$ at $\langle j, i\rangle$

\begin{eqnarray*}
&&\min_{1\le j\le \floor{T_{j,i}}} (\Vert \bigtriangledown F(x_{j})\Vert^2 )
\le \frac{V(\eta,\lambda,L)}{\delta T^{1-\beta_0}}.\label{final2c-bound}
   \end{eqnarray*}   after running $T_{j,i}^*$ iterations at process $j$  with $T\le T_{j,i}\le T_{j,i}^*\le \alpha_p T$.
\end{theorem}

\begin{proof} By Proposition~\ref{embeding-prop}, Parallel-Framework(Static2-SGD(.,.),  $ x_0, h(.,.), T_0,p)$ executes Static2-SGD$(x_0, T_{j,i})$ that generates $\langle x_1, x_2,\ldots, x_{\floor{T_{j,i}}}\rangle$ at $\langle j, i\rangle$ and has $T\le T_{j,i}\le T_{j,i}^*\le \alpha_p T$.
It follows from Corollary~\ref{SGD-cor}.
    
\end{proof}

The adaptivity is achieved by embedding static gradient descent into the parallel framework. When $T$ is sufficiently large, conditions (\ref{condition-T-ineqn}) and (\ref{condition-T2-ineqn}) are satisfied. The parameters $t$ and $s_t$ are chosen independently of $L$, $\eta$, $\lambda$, $\sigma_0$, and $\sigma_1$. The parallel framework employs a geometric sequence, determined by the number $p$ of processors, to search for an appropriate value of $T$. The step size $s_t$ is then adjusted according to the selected value of $T$.

\section{Avoiding Restarting from Scratch }\label{parallle2-sec}

In this section, we present a refined parallel framework that avoids restarting from scratch at each new stage of a processor. Instead of always using the same initial point $x_0$, the next stage starts from an improved point $x_0^*$ obtained from the $p$ processors, where
\[
F(x_0^*)\le F(x_0).
\]
In this way, the framework exploits the partial progress made during previous stages and reuses it to accelerate convergence in subsequent stages across all processors.

This refinement requires evaluating the objective function $F(\cdot)$ and introduces communication among processors to identify the best current iterate. In this section, we briefly describe this extension and discuss its potential advantages.


\subsection{A Refined Parallel Framework}

In this subsection, we describe a refined parallel framework for gradient descent. For an objective function of the form
\[
F(x)=\sum_{i=1}^{k}f_i(x)^2,
\]
evaluating the objective function $F(x)$ may take significantly longer than computing a stochastic gradient $\bigtriangledown f_i(x)$ with a random $i\in \{1,\ldots, k\}$. 


We also have a parameter $\delta$ to control how many steps for iterations and how many steps to find some $x_i$ with $F(x_i)<F(x_0^*)$.

\vskip 20pt
{\bf Algorithm} Parallel2-Framework$(GD(.,.),  x_0, h(.,.), T_0, p, \delta 
)$

Input: 

\begin{enumerate}[label=\arabic*.]
\item
GD$(x_0, T)$ is a gradient descent method with start point $x_0$, and $T$ iterations.

    \item $T_0$ is the least number of steps to execute.

    \item $h(j, i)\ge T_0: \mathbb{N}\times \mathbb{N}\rightarrow \mathbb{N}$ is a function to assign the number of iterations when calling a gradient descent method GD(.).

    \item $h(T):\mathbb{N}\rightarrow \mathbb{N}$ is a function to determine how many steps will be used to run the selection function. For example, $h(T)=\floor{T/10}$.


    \item $x_0\in \mathbb{R}^m$ is the start point.

    \item $p$ is the number of processors.

    \item $\delta\in(0,1)$.

\item $\{$

\item Let $x_0^*=x_0$ be shared by all processors.

\item Processor $j$ ($j=0,1,\ldots, p-1$):

\begin{enumerate}

   \item Let $i=1$

\item Repeat

   \item $\{$ 

     \item \qquad $T_{j,i}=h(j,i)$

          \item \qquad $t=\delta\cdot T_{j,i}$

\item \qquad $Z=$GD$(x_0^*, T_{j,i}-t)$

\item \qquad Select$(Z,  x_0^*, y_0^*, F(.), t 
                     )$


    \item \qquad Let $i=i+1$

\item $\}$


\end{enumerate}

\item 
$\}$

\end{enumerate}

{\bf End of Algorithm}

\vskip 20pt

\begin{proposition}\label{embeding2-prop}
 Assume that GD($x_0, T)$ is a gradient descent method.
 Assume that Parallel2-Framework$(GD(.))$ has $(p,\alpha_p)$-approximation.  Then 
 Parallel2-Framework$(GD(.))$ executes
 $GD(x_0, (1-\delta)T_{j,i})$ satisfying $T\le (1-\delta)T_{j,i}\le T_{j,i}\le T_{j,i}^*\le \frac{\alpha_p T}{1-\delta}$.  
\end{proposition}

\begin{proof} Let $T'=\frac{T}{1-\delta}$. 
By the condition of $(p, \alpha_p)$-approximation, 
there is $T_{j,i}$ with  $T'\le T_{j,i}\le T_{j,i}^*\le \alpha_p T'$. 
This implies $GD(x_0, (1-\delta)T_{j,i})$ satisfying $T\le (1-\delta)T_{j,i}\le T_{j,i}\le T_{j,i}^*\le \frac{\alpha_p T}{1-\delta}$.
\end{proof}


We require a selection function, denoted by Select(.), that chooses one $x_i$ from the sequence $x_1,x_2,\ldots,x_T$ generated by executing GD$(x_0,T)$. The parameter $t$ specifies the maximum number of iterates that Select(.) is allowed to access. 
The following principles may be used to design Select(.):
\begin{itemize}
\item It runs at most the time of   $t$ iterations in $GD(.)$.
\item If an $x_i$ satisfying $F(x_i)<F(x_0^*)$ is found after examining a subset of $Z$, then $x_0^*$ is updated to $x_i$.
\item There is a mutual exclusion to update $x_0^*$ among the $p$ processors.
\end{itemize}

There is no sufficient information to establish quantitative guarantees for the function
$\mathrm{Select}(\cdot)$. Therefore, we leave the implementation and analysis
of $\mathrm{Select}(\cdot)$ as a topic for future research.

\section{Conclusions and Future Developments}

In this paper, we develop a parallel framework that transforms static gradient descent methods into adaptive ones through parallel execution. Given a target number of iterations $T$ that may satisfy the desired convergence conditions, the $p$ processors in the framework search for a suitable parameter $T_{j,i}$ according to a carefully designed geometric sequence. The objective is to minimize the approximation factor $\alpha_p$ while ensuring that $
T\le T_{j,i}\le T_{j,i}^*\le \alpha_p T.
$

Several research directions remain open. First, it will be valuable to identify additional static gradient methods that can be incorporated into this framework. Second, after the number of iterations $T$ is determined, more effective strategies for selecting the corresponding learning rate should be investigated. Third, the arithmetically simple gradient descent methods proposed in this paper eliminate division and square-root operations by replacing them with binary shift operations. Drawing on the author's experience as an FPGA hardware engineer in the computer industry, we believe that this design is more suitable for hardware implementation and chip design. Developing even more efficient gradient descent algorithms for specialized hardware accelerators is therefore an interesting direction for future research.

Another interesting open problem is to close the gap between the current upper bound of $4$ (Corollary~\ref{main2b-cor}) and the lower bound of $2$ (Theorem~\ref{lower-1-2-theorem}) for the approximation factor $\alpha_1$ in a $(1,\alpha_1)$-approximation. Progress on narrowing this gap for the single-processor case may also provide new insights into closing the corresponding gap for $\alpha_p$ when $p>1$.


\section{Acknowledgments}
This author is grateful to Jose Nunez for proofreading an earlier version of this paper and for his helpful comments, which improved the presentation of the paper.

\bibliographystyle{abbrv}
\bibliography{bib}

\end{document}

%% file: mac90.tex
\hyphenation{super-terse mea-sure semi-recursive non-recursive
             non-superterse}
\newcounter{savenumi}
\newenvironment{savenumerate}{\begin{enumerate}
\setcounter{enumi}{\value{savenumi}}}{\end{enumerate}
\setcounter{savenumi}{\value{enumi}}}
\newtheorem{theoremfoo}{Theorem}
\newenvironment{theorem}{\pagebreak[1]\begin{theoremfoo}}{\end{theoremfoo}}
\newenvironment{repeatedtheorem}[1]{\vskip 6pt
\noindent
{\bf Theorem #1}\ \em
}{}
\newtheorem{propositionfoo}[theoremfoo]{Proposition}
\newenvironment{proposition}{\pagebreak[1]\begin{propositionfoo}}{\end{propositionfoo}}
\newtheorem{lemmafoo}[theoremfoo]{Lemma}
\newenvironment{lemma}{\pagebreak[1]\begin{lemmafoo}}{\end{lemmafoo}}
\newtheorem{conjecturefoo}[theoremfoo]{Conjecture}
\newenvironment{conjecture}{\pagebreak[1]\begin{conjecturefoo}}{\end{conjecturefoo}}
\newtheorem{corollaryfoo}[theoremfoo]{Corollary}
\newenvironment{corollary}{\pagebreak[1]\begin{corollaryfoo}}{\end{corollaryfoo}}
\newtheorem{exercisefoo}{Exercise}
\newenvironment{exercise}{\pagebreak[1]\begin{exercisefoo}\rm}{\end{exercisefoo}}
\newtheorem{openfoo}[theoremfoo]{Question}
\newenvironment{open}{\pagebreak[1]\begin{openfoo}}{\end{openfoo}}
\newtheorem{nttn}[theoremfoo]{Notation}
\newenvironment{notation}{\pagebreak[1]\begin{nttn}\rm}{\end{nttn}}

\newtheorem{dfntn}[theoremfoo]{Definition}
\newenvironment{definition}{\pagebreak[1]\begin{dfntn}\rm}{\end{dfntn}}

\newenvironment{proof}
    {\pagebreak[1]{\narrower\noindent {\bf Proof:\quad\nopagebreak}}}{\QED}
\newenvironment{sketch}
    {\pagebreak[1]{\narrower\noindent {\bf Proof sketch:\quad\nopagebreak}}}{\QED}
\newenvironment{comment}{\penalty -50 $(*$\nolinebreak\ }{\nolinebreak $*)$\linebreak[1]\ }
\renewcommand{\theenumi}{\roman{enumi}}
\newcommand{\yyskip}{\penalty-50\vskip 5pt plus 3pt minus 2pt}
\newcommand{\blackslug}{\hbox{\hskip 1pt
        \vrule width 4pt height 8pt depth 1.5pt\hskip 1pt}}

\newcommand{\dash}{{\rm\mbox{-}}}
\newcommand{\floor}[1]{\left\lfloor#1\right\rfloor}
\newcommand{\ceiling}[1]{\left\lceil#1\right\rceil}
\newcommand{\st}{\mathrel{:}}
\newcommand{\ang}[1]{\langle#1\rangle}
\newcommand{\DTIME}{{\rm DTIME}}
\newcommand{\DSPACE}{{\rm DSPACE}}
\newcommand{\NSPACE}{{\rm NSPACE}}
\newcommand{\polylog}{\mathop{\rm polylog}}
\newcommand{\UP}{{\rm UP}}
\newcommand{\PH}{{\rm PH}}
\newcommand{\R}{{\rm R}}
\newcommand{\NTIME}{{\rm NTIME}}
\newcommand{\PTIME}{{\rm PrTIME}}
\newcommand{\PSPACE}{{\rm PSPACE}}
\newcommand{\NE}{{\rm NE}}
\newcommand{\poly}{{\rm poly}}
\newcommand{\LOGSPACE}{{\rm LOGSPACE}}
\newcommand{\E}{{\rm E}}
\newcommand{\SAT}{{\rm SAT}}
\newcommand{\NP}{{\rm NP}}
\let\paragraphsym\P
\renewcommand{\P}{{\rm P}}      
\newcommand{\coNP}{{\co\NP}}
\newcommand{\propersubset}{\subset}
\newcommand{\xor}{\oplus}
\newcommand{\m}{{\rm m}}
\newcommand{\T}{{\rm T}}
\let\savett\tt
\renewcommand{\tt}{{\rm tt}}    
\newcommand{\btt}{{\rm btt}}
\newcommand{\ntt}{{n\rm\mbox{-}tt}}
\newcommand{\ktt}{{k\rm\mbox{-}tt}}
\newcommand{\kT}{{k\rm\mbox{-}T}}
\newcommand{\onett}{{1\rm\mbox{-}tt}}
\newcommand{\Abar}{{\bar{A}}}
\newcommand{\Bbar}{{\bar{B}}}
\newcommand{\Cbar}{{\bar{C}}}
\def\nre.{$n$\/-r.e.}
\newcommand{\ie}{{\it i.e.}}
\newcommand{\eg}{{\it e.g.}}
\newenvironment{multi}{\left\{\begin{array}{ll}}{\end{array}\right.}
\newcommand{\union}{\cup}
\newcommand{\Union}{\bigcup}
\newcommand{\intersection}{\cap}
\newcommand{\Intersection}{\bigcap}
\newcommand{\OR}{\vee}
\newcommand{\AND}{\wedge}
\newcommand{\reason}[1]{\mbox{\qquad\rm #1}}     
\newcommand{\nlreason}[1]{\\ & & \mbox{\rm #1}}  
\newcommand{\infinity}{\infty}
\newcommand{\Implies}{\mathrel{\Longrightarrow}}
\newcommand{\Iff}{\mathrel{\Longleftrightarrow}}
\newcommand{\scrod}{\quad\nopagebreak}
\newcommand{\ACCEPT}{{\rm ACCEPT}}
\newcommand{\PARITY}{{\rm PARITY}}
\newcommand{\PARITYP}{\PARITY\P}
\newcommand{\MOD}{{\rm MOD}}
\newcommand{\PP}{{\rm PP}}
\newcommand{\FewP}{{\rm FewP}}
\newcommand{\EP}{{\rm EP}}
\tolerance=2000
\def\pmod#1{\allowbreak\mkern6mu({\rm mod}\,\,#1)}
\newcommand{\co}{{\rm co}\dash}
\newcommand{\PF}{{\rm PF}}
\newcommand{\EXP}{{\rm EXP}}
\newcommand{\NEXP}{{\rm NEXP}}
\newcommand{\Hastad}{H{\aa}stad}
\newcommand{\Kobler}{K\"obler}
\newcommand{\Schoning}{Sch\"oning}
\newcommand{\Toran}{Tor\'an}
\newcommand{\Balcazar}{Balc{\'a}zar}
\newcommand{\Diaz}{D\'{\i}az}
\newcommand{\Gabarro}{Gabarr{\'o}}
\newcommand{\RIA}{CCR-8808949}
\newcommand{\PYI}{CCR-8958528}
\newcommand{\NSFFT}{CCR-9415410}
\newcommand{\NSFALEXIS}{CCR-9522084}
\newtheorem{factfoo}[theoremfoo]{Fact}
\newenvironment{fact}{\pagebreak[1]\begin{factfoo}}{\end{factfoo}}
\newenvironment{acknowledgments}{\par\vskip 20pt\noindent{\footnotesize\em Acknowledgments.}\footnotesize}{\par}
\newcommand{\POLYLOGSPACE}{{\rm POLYLOGSPACE}}
\newcommand{\SPACE}{{\rm SPACE}}
\newcommand{\BPP}{{\rm BPP}}
\newcommand{\lookatme}{\marginpar[$\Rightarrow$]{$\Leftarrow$}}
\newcommand{\muchless}{\ll}
\newcommand{\plusminus}{\pm}
\newcommand{\nopagenumbers}{\renewcommand{\thepage}{\null}}
\newenvironment{algorithm}{\renewcommand{\theenumii}{\arabic{enumii}}\renewcommand{\labelenumii}{Step \theenumii :}\begin{enumerate}}{\end{enumerate}}
\newcommand{\squeeze}{
\textwidth 6in
\textheight 8.8in
\oddsidemargin 0.2in
\topmargin -0.4in
}
\newcommand{\existsunique}{\exists \! ! \,}
\newcommand{\YaleAddress}{Yale University,
Dept.\ of Computer Science,
P.O.\ Box 208285, Yale Station,
New Haven, CT\ \ 06520-8285.
Email: {\savett beigel-richard@cs.yale.edu}}
\newcommand{\pageline}[2]{[{\bf p.~#1, l.~#2}]}
\newcommand{\Beigel}{Richard Beigel\thanks{\newsupport} \\
Yale University \\
Dept.\ of Computer Science \\
P.O. Box 208285, Yale Station \\
New Haven, CT\ \ 06520-8285}

\newcommand{\newsupport}{Supported in part by grants \RIA\ and \PYI\
from the National Science Foundation}
\newtheorem{propertyfoo}[theoremfoo]{Property}
\newenvironment{property}{\pagebreak[1]\begin{propertyfoo}}{\end{propertyfoo}}

\makeatletter

\def\@makechapterhead#1{ \vspace*{50pt} { \parindent 0pt \raggedright 
 \ifnum \c@secnumdepth >\m@ne \huge\bf \@chapapp{} \thechapter. \par 
 \vskip 20pt \fi \Huge \bf #1\par 
 \nobreak \vskip 40pt } }

\def\@sect#1#2#3#4#5#6[#7]#8{\ifnum #2>\c@secnumdepth
     \def\@svsec{}\else 
     \refstepcounter{#1}\edef\@svsec{\csname the#1\endcsname.\hskip 1em }\fi
     \@tempskipa #5\relax
      \ifdim \@tempskipa>\z@ 
        \begingroup #6\relax
          \@hangfrom{\hskip #3\relax\@svsec}{\interlinepenalty \@M #8\par}
        \endgroup
       \csname #1mark\endcsname{#7}\addcontentsline
         {toc}{#1}{\ifnum #2>\c@secnumdepth \else
                      \protect\numberline{\csname the#1\endcsname}\fi
                    #7}\else
        \def\@svsechd{#6\hskip #3\@svsec #8\csname #1mark\endcsname
                      {#7}\addcontentsline
                           {toc}{#1}{\ifnum #2>\c@secnumdepth \else
                             \protect\numberline{\csname the#1\endcsname}\fi
                       #7}}\fi
     \@xsect{#5}}

\def\@begintheorem#1#2{\it \trivlist \item[\hskip \labelsep{\bf #1\ #2.}]}

\def\@opargbegintheorem#1#2#3{\it \trivlist
      \item[\hskip \labelsep{\bf #1\ #2\ (#3).}]}

\makeatother


\newcommand{\kstar}{{\textstyle *}}
\newcommand{\F}[2]{F_{#1}^{#2}}
\newcommand{\FQ}[2]{\PF_{{#1\rm\mbox{-}T}}^{#2}}
\newcommand{\Q}[2]{\P_{{#1\rm\mbox{-}T}}^{#2}}
\newcommand{\FQp}[2]{\PF_{{#1\rm\mbox{-}tt}}^{#2}}
\newcommand{\Qp}[2]{\P_{{#1\rm\mbox{-}tt}}^{#2}}
\newcommand{\ppoly}{\P/\poly}
\newcommand{\pnplog}{\P^{\NP[\log]}}
\newcommand{\pleq}[1]{\leq_{#1}^{p}}

\def\abs#1{\left|#1\right|}
\def\exp#1{{\rm exp}\left(#1\right)}
\newcommand{\transfig}[1]{\begin{center}\vspace{-6pt}\input{#1.tex}\vspace{-12pt}\end{center}}


\newif\ifshortconferences
\shortconferencesfalse
\newif\ifmediumconferences
\mediumconferencesfalse

\def\ending#1{{\count1=#1\relax
\count2=\count1
\divide\count2 by 100
\multiply\count2 by 100
\advance\count1 by -\count2
\ifnum\count1=11
th%
\else \ifnum\count1=12
th%
\else \ifnum\count1=13
th%
\else 
\count2=\count1
\divide\count1 by 10
\multiply\count1 by 10
\advance\count2 by -\count1
\ifnum\count2=1
st%
\else \ifnum\count2=2
nd%
\else \ifnum\count2=3
rd%
\else th%
\fi\fi\fi\fi\fi\fi
}}

\def\STOC{\conf{STOC}}
\def\STOCname{\ifshortconferences ACM STOC\else\ifmediumconferences Ann. ACM Symp. Theor. Comput.\else Annual ACM Symposium on Theory of Computing\fi\fi}
\def\STOCzero{68}

\def\FOCS{\conf{FOCS}}
\def\FOCSname{\ifshortconferences IEEE FOCS\else\ifmediumconferences Ann. Symp. Found. Comput. Sci.\else IEEE Symposium on Foundations of Computer Science\fi\fi}
\def\FOCSzero{59}

\def\FSTTCS{\conf{FSTTCS}}
\def\FSTTCSname{\ifshortconferences FST\&TCS\else\ifmediumconferences Ann. Conf. Found. Softw. Theor. and Theor. Comp. Sci.\else Conference on Foundations of Software Theory and Theoretical Computer Science\fi\fi}
\def\FSTTCSzero{80}

\def\Complexity{\conf{Complexity}}
\def\Complexityname{\ifshortconferences Conf. Computational Complexity\else\ifmediumconferences Ann. Conf. Computational Complexity\else Annual Conference on Computational Complexity\fi\fi}
\def\Complexityzero{85}
\def\ComplexityOne{Structure in Complexity Theory}

\def\Structures{\conf{Structures}}
\def\Structuresname{\ifshortconferences Ann. Conf. Structure in Complexity Theory\else\ifmediumconferences Ann. Conf. Structure in Complexity Theory\else Annual Conference on Structure in Complexity Theory\fi\fi}
\def\Structureszero{85}
\def\StructuresOne{Structure in Complexity Theory}

\def\SODA{\conf{SODA}}
\def\SODAname{\ifshortconferences ACM-SIAM Symp. on Discrete Algorithms\else Annual ACM-SIAM Symposium on Discrete Algorithms\fi}
\def\SODAzero{89}

\def\STACS{\conf{STACS}}
\def\STACSname{\ifshortconferences STACS\else\ifmediumconferences\else Annual Symposium on Theoretical Aspects of Computer Science\fi\fi}
\def\STACSzero{83}

\def\SPAA{\conf{SPAA}}
\def\SPAAname{\ifshortconferences ACM SPAA\else\ifmediumconferences Ann. ACM Symp. Par. Alg. Arch.\else Annual ACM Symposium on Parallel Algorithms and Architectures\fi\fi}
\def\SPAAzero{88}

\def\Proceedings{\ifshortconferences Proc.\else\ifmediumconferences Proc.\else Proceedings\fi\fi}
\def\Proceedingsofthe{\ifshortconferences Proc.\else\ifmediumconferences Proc.\else Proceedings of the\fi\fi}

\newcounter{confnum}

\def\conf#1#2{%
\setcounter{confnum}{#2}%
\addtocounter{confnum}{-\csname #1zero\endcsname}%
\ifnum\value{confnum}=1%
\expandafter\ifx\csname #1One\endcsname\relax%
\Proceedingsofthe\ \arabic{confnum}\ending{\value{confnum}}\ \csname #1name\endcsname%
\else \csname #1One\endcsname\fi%
\else%
\Proceedingsofthe\
\arabic{confnum}\ending{\value{confnum}}\ \csname #1name\endcsname\fi}

\def\qsym{\vrule width0.7ex height0.9em depth0ex}

\newif\ifqed\qedtrue

\def\noqed{\global\qedfalse}

\def\qed{\ifqed{\penalty1000\unskip\nobreak\hfil\penalty50
\hskip2em\hbox{}\nobreak\hfil\qsym
\parfillskip=0pt \finalhyphendemerits=0\par\medskip}\fi\global\qedtrue}

\def\QEDcomment#1{\ifqed{\penalty1000\unskip\nobreak\hfil\penalty50
\hskip2em\hbox{}\nobreak\hfil\qsym\ #1
\parfillskip=0pt \finalhyphendemerits=0\par\medskip}\fi\global\qedtrue}

\makeatletter
\def\eqnqed{\noqed
	\def\@tempa{equation}
	\ifx\@tempa\@currenvir\def\@eqnnum{\qsym}%
	\addtocounter{equation}{-1}\else%
    \def\@@eqncr{\let\@tempa\relax
    \ifcase\@eqcnt \def\@tempa{& & &}\or \def\@tempa{& &}%
      \else \def\@tempa{&}\fi
     \@tempa {\def\@eqnnum{{\qsym}}\@eqnnum}
     \global\@eqnswtrue\global\@eqcnt\z@\cr}\fi}


\def\eqnlabel#1#2{\if@filesw {\let\thepage\relax%
   \def\protect{\noexpand\noexpand\noexpand}%
   \edef\@tempa{\write\@auxout{\string
      \newlabel{#2}{{{#1}}{\thepage}}}}%
   \expandafter}\@tempa%
   \if@nobreak \ifvmode\nobreak\fi\fi\fi%
	\def\@tempa{equation}
	\ifx\@tempa\@currenvir\def\theequation{{#1}}%
	\addtocounter{equation}{-1}\else%
    \def\@@eqncr{\let\@tempa\relax
    \ifcase\@eqcnt \def\@tempa{& & &}\or \def\@tempa{& &}%
      \else \def\@tempa{&}\fi
     \@tempa {\def\@eqnnum{{#1}}\@eqnnum}
     \global\@eqnswtrue\global\@eqcnt\z@\cr}\fi}

\makeatother


\newcommand{\listqed}{\qed\noqed} 
\def\littleqed{\ifqed{\penalty1000\unskip\nobreak\hfil\penalty50
\hskip2em\hbox{}\nobreak\hfil\littleqsym
\parfillskip=0pt \finalhyphendemerits=0\par\medskip}\fi\global\qedtrue}
\def\littleqsym{\vrule width0.6ex height0.6em depth0ex}
\def\QED{\qed}

\makeatother


\def\aftereqnskip{\vskip-3pt} 


\newcommand{\SPASAM}[1]{{\rm MOL\dash A}(#1)}
\newcommand{\SPASM}[1]{{\rm MOL'}(#1)}
\newcommand{\SPAM}[1]{{\rm MOL}(#1)}
\newcommand{\NPbits}[1]{{\rm NPbits}(#1)}
\newcommand{\NPinit}[1]{{\rm NPinit}(#1)}
\newcommand{\NPpaths}[1]{{\rm NPpaths}(#1)}
\newcommand{\LOR}{\bigvee}

\newcommand{\set}[1]{\left\{{#1}\right\}}

%% file: FifthContent.tex
\section{Introduction}

Stochastic Gradient Descent (SGD)~\cite{RobbinsAndMonro51} is one of the most widely used optimization methods in deep learning because of its efficiency and scalability in training large-scale neural networks. Unlike batch gradient descent, which computes the gradient using the entire training dataset at every iteration, SGD updates the model parameters using a single training example or a small mini-batch. Consequently, SGD requires significantly less memory and has a much lower computational cost per iteration. By processing only a small subset of the data at each step, SGD often converges more quickly in practice, particularly for large-scale datasets.



Gradient descent with diminishing step sizes has a long history. Classical stochastic approximation theory shows that the step sizes ${\eta_i}$ should satisfy
$
\sum_{i=1}^{\infty}\eta_i=+\infty
\quad\text{and}\quad
\sum_{i=1}^{\infty}\eta_i^2<+\infty
$
to guarantee convergence to a stationary point~\cite{RobbinsAndMonro51}. 
For stochastic optimization of smooth nonconvex functions, gradient descent with either a constant step size or a diminishing step size $\eta_i=O(1/\sqrt{i})$ achieves an $O(1/\sqrt{T})$ convergence rate to a stationary point~\cite{GhadimiAndLan13}. In particular, the analysis in~\cite{GhadimiAndLan13} selects the step size as
$\eta_i=\min \left(\frac{1}{L},
\sqrt{\frac{2(F(x_1)-F(x^*))}{L\sigma_0^2N}}
\right),
$
which depends on the Lipschitz smoothness constant $L$, the stochastic gradient variance parameter $\sigma_0$, and the optimality gap $F(x_1)-F(x^*)$, where $F(x^*)=\inf_x(F(x))$. Since these problem-dependent parameters are typically unknown in advance, the resulting static gradient method is non-adaptive. Moreover, the convergence rate of $O(1/\sqrt{T})$ is known to be optimal, matching the corresponding lower bound~\cite{BottouCurtisNocedal18,ArjevaniCDFSW23}.


Adaptive gradient descent methods have become widely used in deep learning in recent years. Unlike static gradient methods, adaptive methods dynamically adjust the learning rate during training according to the historical gradients of individual parameters. This adaptive mechanism reduces the need for manual tuning of learning rates and often improves optimization efficiency and robustness across a wide range of machine learning tasks. A large body of research has established convergence guarantees and convergence rates for adaptive gradient methods under various assumptions and optimization settings~\cite{DuchiHazanSinger2011,DBLP:journals/corr/abs-1002-4908,NemirovskiJuditskyLanShapiro2009,BottouCurtisNocedal2018,WardWuBottou19,XieWuWard2020,DBLP:journals/corr/abs-1212-5701}.


Since the introduction of AdaGrad~\cite{DuchiHazanSinger2011}, numerous adaptive gradient methods have been proposed, including AdaDelta~\cite{DBLP:journals/corr/abs-1212-5701}, Adam~\cite{KingmaBa2015}, AdamW~\cite{Loshchilov2019Decoupled}, AdaFTRL~\cite{OrabonaPal2015}, SGD-BB~\cite{TanMaDaiQian2016}, AdaBatch~\cite{DBLP:journals/corr/abs-1711-01761}, SC-AdaGrad~\cite{DBLP:conf/icml/MukkamalaH17}, AMSGrad~\cite{DBLP:conf/iclr/ReddiKK18}, and Padam~~\cite{DBLP:conf/ijcai/ChenZTYCG20}. These developments reflect the continuing effort to improve adaptive gradient methods by enhancing their efficiency, robustness, theoretical guarantees, and ease of use for large-scale machine learning applications.


Adaptive stochastic gradient descent methods dynamically adjust the step size according to predefined update rules. For example, AdaGrad-Norm~\cite{WardWuBottou19} updates the accumulated scaling factor and the model parameters as
$
s_{i+1}=s_i+|G(\xi,x_i)|^2,$
and
$x_{i+1}=x_i-\frac{\eta}{\sqrt{s_{i+1}}}\cdot G(\xi,x_i),
$
where $G(\xi,x_i)$ denotes the stochastic gradient evaluated at $x_i$. The convergence properties of adaptive stochastic gradient methods have been extensively studied in~\cite{WardWuBottou19,XieWuWard2020,FawRoutCaramanisShakkottai23,WangZhangMaChen023}. Under suitable assumptions, these methods are proven to converge to a stationary point with the optimal convergence rate of $O(1/\sqrt{N})$.


Parallel gradient descent has become an important optimization framework for large-scale machine learning and scientific computing because it enables gradient computations to be distributed across multiple processors, thereby significantly reducing training time and improving scalability. Early theoretical foundations for parallel and asynchronous iterative optimization were established by Dimitri P. Bertsekas and John N. Tsitsiklis~\cite{BertsekasTsitsiklis89}, who analyzed convergence properties under delayed and distributed updates.
Large-scale machine learning later motivated parallel SGD algorithms such as Hogwild \cite{recht2011hogwild}, parameter-server architectures \cite{li2014scaling}, and distributed deep learning systems \cite{dean2012large}. Recent adaptive parallel methods further combine distributed computation with adaptive learning-rate mechanisms for improved convergence behavior \cite{reddi2018convergence}.  More recent research has focused on adaptive and communication-efficient distributed optimization, including adaptive SGD methods \cite{cutkosky2018distributed} and multi-timescale distributed adaptive optimization frameworks \cite{iacob2025mtdao}.

\subsection{Our Contributions}

Our goal is to endow static gradient methods with adaptivity through a parallel framework. A static gradient method, denoted by $\mathrm{GD}(x_0,T)$, takes as input an initial point $x_0$ and an integer $T$ specifying the number of iterations. The step size is determined by
$s=S(T),$
where $S(\cdot)$ is a prescribed function of $T$. The method then performs the iterations
$
x_{i+1}=x_i-\frac{\eta}{s}\cdot g_i,
$
where $g_i$ is a stochastic gradient evaluated at $x_i$, and $\eta$ is a scaling factor. 
A fundamental challenge in applying a static gradient method is selecting an appropriate value of $T$. The parameter $T$ must be sufficiently large to guarantee the desired convergence, yet the required number of iterations typically depends on unknown problem characteristics, such as the Lipschitz smoothness constant and the stochastic gradient parameters. Our parallel framework addresses this challenge by searching for a suitable value of $T$ through parallel execution.


We develop a parallel framework for gradient descent that searches for an appropriate iteration budget $T$ in parallel according to a carefully designed geometric sequence. The framework assembles multiple static gradient methods into an adaptive gradient descent method. It consists of $p$ processors running in parallel. The number of iterations assigned to processor $j$ at stage $i$ is determined by $
T_{j,i}=h(j,i).
$
Processor $j$ $(j=0,1,\ldots,p-1)$ executes $\mathrm{GD}(x_0,T_{j,i})$ at stage $i$ ($i=0,1,2,\ldots$). We choose
$
h(j,i)=b_p^{jp+i}T_0,
$
where
$
b_p=(p+1)^{1/p},
$
and $T_0$ is the minimum number of iterations assigned to any stage.

We prove that, for every $T\ge T_0$, there exist a processor $j$ and a stage $i$ such that
\[T\le T_{j,i}\le T_{j,i}^*<\alpha_pT,\]
where
$T_{j,i}^*=\sum_{t=0}^{i}T_{j,t}$ 
is the total number of iterations executed by processor $j$ through stage $i$, and
\[
\alpha_p=\left(1+\frac{1}{p}\right)(p+1)^{1/p}\le\upperbound.
\]
The approximation factor $\alpha_p$ measures the computational overhead incurred before a processor reaches an iteration budget that is sufficient to satisfy the desired convergence guarantee. A smaller value of $\alpha_p$ indicates that fewer iterations are wasted in the preceding stages.

We further establish a nearly matching lower bound by proving that, for any scheduling function $h(j,i)$ and any constant $d>1$, every $(p,\alpha_p)$-approximation must satisfy
\[
\alpha_p\ge\lowerbound
\]
for all sufficiently large $p$.

The convergence behavior of the resulting parallel algorithm is therefore essentially the same as that of the underlying static gradient method $\mathrm{GD}(\cdot)$, whose convergence is determined by the iteration budget $T$. Consequently, our framework provides the adaptivity of parallel search while preserving the relatively simple convergence analysis of static gradient methods. Theoretical analysis establishes nearly matching upper and lower bounds on $\alpha_p$, revealing an intrinsic tradeoff between parallelism, adaptivity, and computational overhead.

We develop a static gradient descent method under the following $(\lambda,\sigma_0,\sigma_1)$-stochastic model. Let $\xi$ be a random variable, and let $G(\xi,x)$ denote a stochastic gradient of $F(x)$. We assume that

\[
\left\langle \mathbb{E}_{\xi}[G(\xi,x)],,\nabla F(x)\right\rangle
\ge
\lambda|\nabla F(x)|^2
\]
for some $\lambda\in(0,\infty)$; and

\[
\mathbb{E}_{\xi}\left[|\nabla F(x)-G(\xi,x)|^2\right]
\le
\sigma_0^2+\sigma_1^2|\nabla F(x)|^2
\]
for some $\sigma_0,\sigma_1\in[0,\infty)$.

This model generalizes the standard stochastic gradient model, which assumes
$
\mathbb{E}_{\xi}[G(\xi,x)]=\nabla F(x),
$
and
$
\mathbb{E}_{\xi}\left[|\nabla F(x)-G(\xi,x)|^2\right]
\le
\sigma_0^2
$
for some $\sigma_0\in[0,\infty)$.

Rigorous convergence analysis of stochastic and adaptive gradient methods is essential for understanding their theoretical behavior, improving their performance, and ensuring their reliability across a broad range of machine learning tasks. Such analyses also reveal how convergence depends on problem characteristics and algorithmic hyperparameters, thereby guiding the design of more robust optimization algorithms.

To the best of our knowledge, convergence guarantees under the above $(\lambda,\sigma_0,\sigma_1)$-stochastic model, in which all three parameters $\lambda$, $\sigma_0$, and $\sigma_1$ are allowed to be positive, have not been established in the existing literature. We establish the following convergence results under this model.


We develop a new gradient descent method under the proposed $(\lambda,\sigma_0,\sigma_1)$- stochastic model. Given an iteration budget of $T$ steps, the method performs the updates
\[
x_{j+1}=x_j-\frac{\eta}{s_T}G(\xi,x_j),
\]
where
$s_T=2^{\left\lceil \frac{\lceil\log T\rceil}{2}\right\rceil},$
and $\eta>0$ is an arbitrary input parameter.

Assuming that the objective function satisfies the standard $L$-Lipschitz smoothness condition,
\[
|\nabla F(x)-\nabla F(y)|\le L|x-y|,
\]
we prove that the proposed method converges to a stationary point for nonconvex optimization under the $(\lambda,\sigma_0,\sigma_1)$-stochastic model. Moreover, it achieves the optimal convergence rate of $O(1/\sqrt{T})$.

In the gradient descent methods proposed in this paper, every denominator is of the form $2^t$ for some integer $t$. As a result, division and square-root operations, which are commonly used in gradient descent algorithms, are eliminated and replaced by binary shift operations. This simplification makes the proposed methods more suitable for efficient hardware implementation and chip design.

The convergence analysis of the proposed stochastic model provides a theoretical explanation for why a parallel search over the iteration budget $T$ is necessary. Our algorithm is parameter-adaptive: it automatically adapts to the unknown Lipschitz smoothness constant $L$ and the stochastic gradient parameters $\lambda$, $\sigma_0$, and $\sigma_1$. Moreover, the parallel architecture is constructed independently of these unknown parameters, making the framework broadly applicable without prior knowledge of the optimization problem.


\subsection{Organization of This Paper}

The remainder of this paper is organized as follows. In Section~\ref{overview-sec}, we present an overview of the proposed parallel framework for static gradient descent. 
Section~\ref{parallle-sec} formally introduces the parallel model and the notion of a $(p,\alpha_p)$-approximation. In Section~\ref{Upper-sec}, we derive upper bounds on $\alpha_p$, while Section~\ref{lower-sec} establishes corresponding lower bounds. 
Section~\ref{alg-sec} 
presents the convergence analysis of the proposed static stochastic gradient method and demonstrates how it fits into the parallel framework. In Section~\ref{parallle2-sec}, we introduce a refined parallel model that avoids repeatedly restarting from the same initial point $x_0$. 
Instead, it progressively replaces $x_0$ 
with an improved starting point 
$x_0^*$ satisfying 
$F(x_0^*)\le F(x_0)$. 
Finally, we conclude that adaptivity can be achieved through parallelization while preserving the simplicity of convergence analysis for static gradient methods.


\section{Overview of Our Method}\label{overview-sec}

In the parallel framework developed in this paper, we assume that $p$ processors execute concurrently. The processors cooperatively search for a suitable iteration budget $T$ for the given static gradient method $\mathrm{GD}(x_0,T)$ by exploring a geometric sequence of candidate values. Each processor $j$ $(j=0,1,\ldots,p-1)$ proceeds through an infinite sequence of stages. At stage $i$, processor $j$ is assigned an iteration budget $T_{j,i}$ and executes $\mathrm{GD}(x_0,T_{j,i})$.

The candidate iteration budgets are selected from the geometric sequence
\[
T_0,\ (1+\epsilon)T_0,\ (1+\epsilon)^2T_0,\ \ldots,\ (1+\epsilon)^kT_0,\ \ldots.
\]
Specifically, we define
\[
T_{j,i}=h(j,i)=(1+\epsilon)^{ip+j}T_0,
\]
for $j=0,1,\ldots,p-1$. Consequently, for any desired iteration budget $T$ that is sufficiently large to satisfy the convergence guarantee, there always exists a value $T_{j,i}$ in the sequence such that $T_{j,i}$ is only slightly larger than $T$.

The scheduling function $h(j,i)$ and the parameter $\epsilon$ are determined by the number of processors $p$. As $p$ increases, the value of $\epsilon$ decreases, yielding a denser geometric sequence and thereby reducing the gap between the selected iteration budget and the desired value $T$.



\begin{figure}[htbp]
    \centering
\includegraphics[page=1,width=360pt]{Fig2.pdf}\caption{Four Parallel Processors with Infinitely Many Stages.}
    \label{Fig}
\end{figure}

We show that, for every target iteration budget $T$, there exist a processor $j$ and a stage $i$ such that
\[
T\le T_{j,i}\le \sum_{t=0}^{i}T_{j,t}<\alpha_pT.
\]
We derive both upper and lower bounds for $\alpha_p$, and show that these bounds are nearly tight in the proposed parallel model. The scheduling function $h(j,i)$ and the parameter $\epsilon$ are designed according to the number of processors $p$. As $p$ increases, $\alpha_p$ approaches $1$, implying that only a small amount of computation is wasted before reaching an iteration budget that satisfies the desired convergence guarantee.

Figure~\ref{Fig} illustrates the parallel framework with four processors. Each rectangle represents one stage of a processor, and the integer inside the rectangle denotes the number of iterations assigned to that stage. For example, when $T=26$, processor $P_2$ reaches Stage~1 with
\[
T\le T_{2,1}=28,
\]
and the cumulative number of iterations executed by that processor is
\[
T_{2,1}^*=T_{2,0}+T_{2,1}=36.
\]

We design the parameter $\epsilon$ to balance two competing objectives: efficiently locating a suitable iteration budget $T_{j,i}$ and keeping the approximation factor $\alpha_p$ close to $1$. Once a sufficiently large iteration budget $T_{j,i}$ is identified, the corresponding step size, determined by the function $S(T_{j,i})$, satisfies the conditions required for the convergence guarantee of the underlying static gradient method $\mathrm{GD}(\cdot)$.





\section{A Parallel Framework for Gradient Methods}\label{parallle-sec}

In this section, we introduce a parallel architecture that enables a gradient method to adapt automatically to unknown problem parameters. By running multiple instances of a static gradient method in parallel, each with a different fixed step size, our framework transforms a static gradient method into an adaptive one.



Let $\mathbb{R}=(-\infty,+\infty)$ denote the set of real numbers, and let $\mathbb{R}^{+}=(0,+\infty)$ denote the set of positive real numbers. Let $\mathbb{N}=\{0,1,2,\ldots\}$ denote the set of nonnegative integers. For a real number $x$, let $\lceil x\rceil$ denote the smallest integer greater than or equal to $x$, and $\floor{x}$ denote the largest integer less than or equal to $x$.


\subsection{Parallel Frameowrk}

\begin{definition}
A function $h:\mathbb{N}\times\mathbb{N}\rightarrow\mathbb{R}$ is called \emph{geometric} if there exist constants $h_0>0$, $a_0>1$, and an integer $p\ge1$ such that
\[
h(j,i)=h_0a_0^{jp+i}
\]
for all $j,i\in\mathbb{N}$.
Equivalently, the values of $h(j,i)$ are given by the geometric sequence
\[
h_0,\ h_0a_0,\ h_0a_0^2,\ \ldots.
\]
\end{definition}



We have the following parallel framework that calls a static gradient descent method $GD(x_0, T)$. The parallel executions of GD$(x_0,T_{j,i})$ finds a $T_{j,i}$ that will satisfy the condition of convergence. 

\vskip 20pt
{\bf Algorithm} Parallel-Framework$(A(.,.),  x_0, h(.,.), T_0, p)$

Related Parameters: 

\begin{itemize}
\item
$A(x_0, T)$ is an algorithm  with input point $x_0$, and an integer $T$ for the number of  iterations or steps to execute $A(.)$.

    \item $T_0$ is the least number of steps to execute

    \item $h(j, i): \mathbb{N}\times \mathbb{N}\rightarrow \mathbb{N}$ is a function to assign the number of iterations when running $A(.,.)$

    \item $x_0\in \mathbb{R}^m$ is the start point,

    \item $p$ is the number of processors.
\end{itemize}

Processor $j$ ($j=0,1,\ldots, p-1$):

\begin{enumerate}[label=\arabic*.]

   \item $i=0$

\item Repeat

   \item $\{$ 

     \item \qquad Let $T_{j,i}=h(j,i)$

\item \qquad $A(x_0, T_{j,i})$

    \item \qquad  $i=i+1$

\item $\}$

\end{enumerate}

{\bf End of Algorithm}

\begin{definition}\label{App-def}
Let $\mathrm{Parallel\mbox{-}Framework}(\cdot)$ denote the parallel framework defined by the algorithm.

\begin{enumerate}
\item
We say that $\mathrm{Parallel\mbox{-}Framework}(\cdot)$ has a {\it $(p,\alpha_p)$-approximation} if it consists of $p$ processors indexed by $0,1,\ldots,p-1$, and for every integer $T\ge T_0$, there exist a processor $j<p$ and a stage $i$ such that
\begin{eqnarray}
T\le T_{j,i}\le T_{j,i}^*<\alpha_pT,\label{basic-ineqn}    
\end{eqnarray}

where $
T_{j,i}^*=\sum_{t=0}^{i}T_{j,t}.$

\item
A {\it geometric parallel framework of $p$ processor} is a parallel framework in which the scheduling function $h(j,i)$ for $p$ processors generates the iteration budgets according to a geometric progression, and can be expressed as $h(j,i)=b_p^{ip+j}T_0$ for some $b_p>1$ and $T_0\ge 1$.
\end{enumerate}
\end{definition}

In Definition~\ref{App-def}, the condition (\ref{basic-ineqn})
measures the computational overhead incurred before reaching an iteration budget $T_{j,i}$ that is at least the target value $T$. We derive both upper and lower bounds for the approximation factor $\alpha_p$. Furthermore, the proposed parallel framework guarantees that $\alpha_p$ can be made arbitrarily close to $1$ as the number of processors $p$ increases.

Lemma~\ref{monotonic-lemma} establishes a monotonicity property of $T_{j,i}$ and $T_{j,i}^*$ in the geometric parallel framework with $p$ processors. This property will be used to derive a lower bound that matches the corresponding upper bound for $\alpha_p$ in a geometric parallel framework.

\begin{lemma}\label{monotonic-lemma}
For the geometric parallel framework with $p$ processors, if $0 \le j < k \le p-1$, then
\[
T_{j,i}<T_{k,i}
\quad\text{and}\quad
T_{j,i}^*<T_{k,i}^*
\]
for every stage $i$.
\end{lemma}

\begin{proof}
The result follows directly from Definition~\ref{App-def}, which defines $T_{j,i}$, $T_{j,i}^*$, and the geometric parallel framework, together with the assumptions $b_p>1$ and $T_0\ge 1$.
\end{proof}





\subsection{An Example for A(.)}

We first describe a static gradient method whose step size is determined by a function $S(T)$ of the prescribed number of iterations $T$. For example, let $S(T)=\sqrt{T}$. Instead of computing $\sqrt{T}$ exactly, we seek an integer $m$ such that
\[
2^m\in[\sqrt{T},\,4\sqrt{T}].
\]
This approximation eliminates square-root and division operations while preserving the desired asymptotic behavior.

We give a description of a static gradient descent. Its step size is determined by a function $S(T)$. For example, $S(T)=\sqrt{T}$. We tend to find an integer $m$ such that $S(T)=2^m\in [\sqrt{T}, 4\sqrt{T}]$.  This can remove division and square root operations.
\vskip 20pt
{\bf Algorithm} Static-SGD$(x_0, T)$

Related Parameters: 
\begin{itemize}
    \item 
$x_0$ is the start point
\item 
$T\in [1,+\infty)$ controls the number of iterations

\item 
$\eta$ is a scaling factor

\item 
$S(T):\mathbb{R^+}\rightarrow\mathbb{R^+}$ is a function to determine the stepsize based on $T$

\item 
$G(\xi, x)$ is a stochastic (approximate) gradient for $x$.
\end{itemize}

Steps:

\begin{enumerate}[label=\arabic*.]
    \item $i=1$

\item $x_1=x_0$

\item $s=S(T)$

    \item while $i\le T$ 

    \item $\{$

    \item \hskip 20pt $g_i=G(\xi, x_i)$

\item \hskip 20pt $x_{i+1}=x_i-\frac{\eta}{s} \cdot g_i$

    \item \hskip 20pt $i=i+1$

    \item $\}$

\end{enumerate}

{\bf End of Algorithm}

\vskip 20pt

\subsection{A Matrix Construction}

We transform the parallel framework into a $p\times \infty$ matrix $\Gamma_{p\times \infty}=(T_{j,i})_{p\times  \infty}$ for $j=0,1,\ldots$ and $i=0,1,\ldots$ (Both row and column indices start from $0$). The entry at row $j$ and column $i$ is denoted by $T_{j,i}$. $T_{j,i}^*$ and $(p,\alpha_p)$-approximation are defined as in Definition~\ref{App-def}.  A row $j$ in $\Gamma$ repents processor $j$ for its stages, and column $j$ in $\Gamma$ is for stage $j$. Matrix $\Gamma_{p\times \infty}=(T_{j,i})_{p\times  \infty}$ is generated by $h(j,i)$ if $T_{j,i}=h(j,i)$ for $j=0,1,\ldots,p-1$, and $j=0,1,\ldots$.
 
\section{Upper Bound for $\alpha_p$ in Parallel Model}\label{Upper-sec}
In this section, we show a $(p,\alpha_p)$-approximation for the parallel model. An upper bound for the parameter $\alpha_p$ will be derived.



\begin{lemma}\label{Taylor-lemma}
   For $x\in [0,1]$, $e^x\le 1+x+x^2$. 
\end{lemma}
\begin{proof}
  It follows from the Taylor expansion of $e^x$: $e^x=1+x+\frac{x^2}{2!}+\frac{x^3}{3!}+\ldots\le 1+x+\frac{x^2}{2!}+x^3(\frac{1}{3!}+\frac{1}{4!}+\ldots)\le 1+x+\frac{x^2}{2!}+x^3(\frac{1}{2^2}+\frac{1}{2^3}+\ldots)\le 1+x+x^2$.  
\end{proof}

\begin{lemma}\label{help-lemma}
If $p$ is an integer with $p\ge 1$,     then
\begin{eqnarray*}
\left(1+\frac{1}{p}\right)(1+p)^{1/p}\le 1+\frac{1+\ln(1+p)}{p}+\frac{1}{p^2}\left(2\left({\ln(1+p)}\right)^2+{\ln(1+p})\right).    
\end{eqnarray*}
\end{lemma}

\begin{proof}
It is easy to verify that $\frac{\ln(1+p)}{p}<1$ for all integers $p\ge 1$. A simple induction shows $\ln (1+p)< p$. It is true at   $p=1$  as $e\approx 2.71828$. Assume $\ln (1+p)\le p$. We have $\ln(1+(p+1))< \ln e(1+p)= 1+\ln(1+p)< 1+p$.
By Lemma~\ref{Taylor-lemma}, we have
  \begin{eqnarray*}
(1+p)^{1/p}=e^{\frac{\ln(1+p)}{p}}\le 1+\left(\frac{\ln(1+p)}{p}\right)+\left(\frac{\ln(1+p)}{p}\right)^2.    
  \end{eqnarray*}

Therefore, 

\begin{eqnarray*}
   &&\left(1+\frac{1}{p}\right)(1+p)^{1/p}\le \left(1+\frac{1}{p}\right) \left(1+\left(\frac{\ln(1+p)}{p}\right)+\left(\frac{\ln(1+p)}{p}\right)^2\right)\\
   &=&1+\frac{1}{p}+\left(\frac{\ln(1+p)}{p}\right)+\left(\frac{\ln(1+p)}{p}\right)^2+\frac{1}{p}\left(\left(\frac{\ln(1+p)}{p}\right)+\left(\frac{\ln(1+p)}{p}\right)^2\right)\\
   &\le&1+\frac{1+\ln(1+p)}{p}+\frac{1}{p^2}\left(2\left({\ln(1+p)}\right)^2+{\ln(1+p})\right).
\end{eqnarray*}
\end{proof}

Theorem~\ref{main2b-thm} shows an upper bound for $\alpha_p$ for $(p,\alpha_p)$-approximation. It covers all the cases for $p\ge 1$. Its proof shows how to select function $h(.,.)$.

\begin{theorem}\label{main2b-thm}
Let function $h(j, i)=b_p^{ip+j}T_0$ and  $b_p=(p+1)^{\frac{1}{p}}$.
For any integer $T\ge T_0$, the Parallel-Framework(.) has $(p,\alpha_p)$-approximation   with $\alpha_p= (1+\frac{1}{p})(1+p)^{\frac{1}{p}}
\le \upperbound.$
\end{theorem}

\begin{proof} The processor $j$ will use the steps $b_p^{j}\cdot T_0, b_p^{p+j}\cdot T_0, b_p^{2p+j}\cdot T_0,\ldots, b_p^{ip+j}\cdot T_0,\cdots$. At phase $i$, processor $j$ uses $T_{j, i}=h(j,i)=b_p^{ip+j}T_0$ to control the number of  of iterations in GD($x_0, T_{j,i}$). The proof also shows how $b_p$ is computed to get a minimal $\alpha_p$.

Define
\begin{eqnarray*}
T_{j, i}^*=\sum_{t=0}^{i} b_p^{tp+j}T_0=b_p^jT_0\sum_{t=0}^{i} b_p^{tp}=b_p^jT_0\cdot \frac{b_p^{(i+1)p}-1}{b_p^p-1}< \frac{b_p^{(i+1)p+j}T_0}{b_p^p-1}=\frac{b_p^pT_{j,i}}{b_p^p-1}.  
\end{eqnarray*}

We note that $T_{j,t}=b_p^{tp+j}T_0$ is the number of steps in the $t$-th iteration. 
Let $T_{j,t}=b_p^{tp+j}\cdot T_0$ be the least with $T\le T_{j,t}$. We have $T\le T_{j,t}<b_pT$. 






\begin{eqnarray*}
T\le T_{j,t}\le T_{j,t}^*&<& 
 \frac{b_p^p}{b_p^p-1}\cdot T_{j,t}\\
&\le& \frac{b_p^{p+1}}{b_p^p-1}\cdot T
\end{eqnarray*}

Define $f(x)$ by 
    \begin{eqnarray}
   f(x)=\frac{x^{p+1}}{x^p-1}. 
\end{eqnarray}

Take derivative for $f(x)$.

    \begin{eqnarray}  f(x)'=\frac{(p+1)x^p(x^p-1)-px^{2p}}{(x^p-1)^2}. 
\end{eqnarray}

Let 
    \begin{eqnarray}
   (p+1)x^p(x^p-1)-px^{2p}=0. 
\end{eqnarray}
It transformed into

    \begin{eqnarray}
   x^p-(p+1)=0. 
\end{eqnarray}
So, we can let $b_p=(p+1)^{1/p}$ to have least $f(b_p)$.

So,
    \begin{eqnarray*}  f(b_p)&=&\frac{b_p^{p+1}}{b_p^p-1}=\frac{(1+p)b_p}{p}\\
   &=&\frac{(1+p)(1+p)^{1/p}}{p}   =\left(1+\frac{1}{p}\right){(1+p)^{1/p}}\\  &\le&\upperbound (by~Lemma~\ref{help-lemma}). 
\end{eqnarray*}

Therefore, if $b_p=(p+1)^{1/p}$,  we have $T\le T_{j,i}\le T_{j,i}^*\le \alpha_pT$ with  $\alpha_p= (1+\frac{1}{p})(1+p)^{\frac{1}{p}}$.

\end{proof}




    
We have Corollary~\ref{main2b-cor} for the cases $p=1,2$. They correspond to the cases for one processor, and two processors, respectively.

\begin{corollary}\label{main2b-cor}Let $p$ be the number of processors in Parallel-Framework(.). 
We have 
\begin{enumerate}
    \item 
For $p=1$,  
  Parallel-Framework(.) has $(1,4)$-approximation with  $b_1=2$.
\item 
For $p=2$, Parallel-Framework(.) has $(2,2.5981)$-approximation
 with  $b_2=\sqrt{3}$.
\end{enumerate}
\begin{proof}
   It follows from Theorem~\ref{main2b-thm} with $\alpha_p= (1+\frac{1}{p})(1+p)^{\frac{1}{p}}$. 
\end{proof}

\end{corollary}

Using the numerical solutions for the expression of $\alpha_p$ in Theorem~\ref{main2b-thm}, we have upper bounds below:
\begin{eqnarray*}
    \alpha_3&\le& 2.11654,
    \alpha_4\le 1.86919,
    \alpha_5\le 1.71717,
    \alpha_6\le 1.61229,
    \alpha_7\le 1.53459\\
    \alpha_8&\le& 1.474397,
    \alpha_{9}\le 1.42615,
    \alpha_{10}\le 1.38644,
    \alpha_{11}\le 1.35309,
    \alpha_{12}\le 1.32459,\\
    \alpha_{13}&\le& 1.29994,
    \alpha_{14}\le 1.27844,
    \alpha_{15}\le 1.25966,
    \alpha_{16}\le 1.24329.
\end{eqnarray*}

\section{Lower Bounds for $\alpha_p$ with Arbitrary $h(.,.)$}\label{lower-sec}

In this section, we show a lower bound in the parallel model. The lower bound of this section has a small gap with the upper bound of Section~\ref{parallle-sec}. Our lower bound almost matches the upper bound.

\begin{lemma}\label{lowerbd-lemma} For any function  $h(j, i)$, 
 if  Parallel-Framework(.) has   $(p,\alpha_p)$- approximation,  then  we have 
 \begin{enumerate}
     \item\label{first-case-alphap} for any positive integer $z$,  $\alpha_p\ge 1+\frac{1}{\alpha_p^p}+\frac{1}{\alpha_p^{2p}}+\ldots+\frac{1}{\alpha_p^{zp}}$, and   
     \item $\alpha_p\ge r_0$, where $r_0>1$ is a root of $x^p-x^{p-1}-1=0$.
 \end{enumerate}

\end{lemma}

\begin{proof} We fix $p$ and $\alpha_p>1$  (by its definition).  Define $T_{j,i}^*=\sum_{1\le i\le t}T_{j,t}$
Let consider the sequence $V_0=T_0, V_1=\beta V_0,\ldots, V_k=\beta^k T_0,\ldots$. By the condition of $(a,\alpha_p)$-approximation (Definition~\ref{App-def}), for each $V_k$, we have a $T_{j,i}$ such that $V_k\le T_{j, i}\le T_{j, i}^*\le \alpha_p V_k$. Let 
\begin{eqnarray}
H=\sum_{k=0}^mV_k=T_0(1+\beta+\beta^2+\ldots+\beta^m)=\frac{\beta^{m+1}-1}{\beta-1}T_0.  \label{H-eqn}    
\end{eqnarray}

We have that for each $k\le m$, $V_k\le T_{j, i_j}^*< \alpha_p V_k$. We will select $\beta=\alpha_p$. This makes the case for each $T_{j,i}$, there is at most one $V_k$ to have $V_k\le T_{j,i}\le T_{j,i}^*< \alpha_p V_k$. This is because $V_{k+1}=V_k\beta=V_k\alpha_p> T_{j,i}^*$. Thus, $V_{k+1}$ does not satisfy the inequality $V_{k+1}\le T_{j,i}\le T_{j,i}^*< \alpha_p V_{k+1}$. 

 Let $Q\subseteq \{0,1,\ldots, p-1\}$ such that for each $q\in Q$, there is a $V_k$ with $0\le k\le m$ and $V_k\le T_{q, i}\le T_{q, i}^*\le \alpha_p V_k$ for some $i$. For a $q\in Q$, let $T_{q, i_q}$ be the largest with $V_k\le T_{q, i_q}\le T_{q, i_q}^*< \alpha_p V_k$ for some $V_k$ ($0\le k\le m$).

Among the series $V_1,V_2,\ldots, V_m$, the largest $p$ items are $V_{m-p+1}, V_{m-p+2},\ldots, V_m$. For each $q\in Q$, there is only one $T_{q,i_q}$ according to its definition.
As each $V_k$ has at most one $T_{j,i}$ with $V_k\le T_{j, i}\le T_{j, i}^*< \alpha_p V_k$, we have inequality
\begin{eqnarray*}
    \sum_{j\in Q} T_{j, i_j}^*< \alpha_p\cdot \sum_{(m-p+1)\le k\le m} V_k.
\end{eqnarray*}

By equation (\ref{H-eqn}), we have 

\begin{eqnarray*}
\frac{\beta^{m+1}-1}{\beta-1}T_0=H&\le& \sum_{j\in Q} T_{j, i_j}^*< \alpha_p\cdot \sum_{(m-p+1)\le k\le m} V_k\\   &=& \alpha_p\cdot \sum_{(m-p+1)\le k\le m} \beta^kT_0\\
&=& \alpha_p\cdot\beta^{m-p+1} T_0(1+\beta+\ldots+\beta^{p-1})\\
&=& \alpha_p\cdot\beta^{m-p+1} \cdot \frac{\beta^p-1}{\beta-1}\cdot T_0.
\end{eqnarray*}

Therefore, 
\begin{eqnarray*}
   \alpha_p&\ge&  \frac{(\beta^{m+1}-1)}{\beta^{m-p+1} \cdot (\beta^p-1)}=\frac{\beta^{m+1}-1}{\beta^{m+1}-\beta^{m-p+1}}\\
   &=&\frac{1-\frac{1}{\beta^{m+1}}}{1-\frac{1}{\beta^p}}
\end{eqnarray*}

Let $m=(z+1)p-1$. We have
\begin{eqnarray*}
   \alpha_p\ge  \frac{1-\frac{1}{\beta^{(z+1)p}}}{1-\frac{1}{\beta^p}}=1+\frac{1}{\beta^p}+\frac{1}{\beta^{2p}}+\ldots+\frac{1}{\beta^{zp}}.
\end{eqnarray*}

As $\beta=\alpha_p$, this proves (\ref{first-case-alphap}) of the lemma. 
We have
\begin{eqnarray*}
   1\ge \frac{1}{\alpha_p}++\frac{1}{\alpha_p^{p+1}}+\frac{1}{\alpha_p^{2p+1}}+\ldots+\frac{1}{\alpha_p^{zp+1}}.
\end{eqnarray*}

The number $\alpha_p$ is fixed in the beginning of this proof. Taking limit for $z\rightarrow+\infty$, we have 

\begin{eqnarray*}
   1\ge \frac{1}{\alpha_p}++\frac{1}{\alpha_p^{p+1}}+\frac{1}{\alpha_p^{2p+1}}+\ldots+\frac{1}{\alpha_p^{zp+1}}+\ldots.
\end{eqnarray*}

We consider the equation, 
\begin{eqnarray}
1&=&\frac{1}{x}+\frac{1}{x^{p+1}}+\frac{1}{x^{2p+1}}+\ldots+\frac{1}{x^{zp+1}}+\ldots\label{x-eqn}  \\
&=&\frac{1}{x}\cdot \frac{1}{1-\frac{1}{x^p}}= \frac{x^{p-1}}{x^p-1}
\end{eqnarray}

Thus, we have equation $x^p-x^{p-1}-1=0$. If $r_0>1$ is a root, then $r_0$ is also the root of equation~(\ref{x-eqn}).
The right side of equation~(\ref{x-eqn}) is strictly decreasing. 
We have $\alpha_p\ge r_0$.
\end{proof}

\subsection{The Case for Large Number of Processors $p$}

We derive a lower bound for the case $p$ is large. A special analysis for be given for the case $p=1$ in the next section.

\begin{theorem} For any function  $h(j, i)$, 
 if   Parallel-Framework(.)  has   $(p,\alpha_p)$-approximation,  then for any fixed $d>1$, $\alpha_p\ge 1+\frac{1+\ln (1+p)}{p}-\frac{d\ln\ln p}{p}$  for all large $p$.   
\end{theorem}

\begin{proof} 
By (\ref{first-case-alphap}) of Lemma~\ref{lowerbd-lemma}, 
we have
\begin{eqnarray*}
   \alpha_p\ge  1+\frac{1}{\alpha_p^p}.
\end{eqnarray*}



We will use the classical fact that $(1+\frac{1}{x})^{x}$ is increasing for all $x>1$, and $\lim_{x\rightarrow+\infty}(1+\frac{1}{x})^{x}=e\approx 2.71828$ (Euler's number). It can be found in most calculus textbooks.

Assume that $\alpha_p< 1+\frac{1+\ln (1+p)-d\ln\ln p}{p}$ with a fixed $d\in (1,+\infty)$. We have 
\begin{eqnarray*}
1+\frac{1}{\alpha_p^p}&\ge& 1+\frac{1}{(1+\frac{1+\ln (1+p)-d\ln\ln p}{p})^p}\\
&=& 1+\frac{1}{\left(1+\frac{1+\ln (1+p)-d\ln\ln p}{p}\right)^{\frac{p}{(1+\ln (1+p)-d\ln\ln p)}\cdot (1+\ln (1+p)-d\ln\ln p)}}\\
&>& 1+\frac{1}{e^{1+\ln (1+p)-d\ln\ln p}}\\
&=& 1+\frac{(\ln p)^d}{e(1+p)}>1+\frac{1+\ln (1+p)-d\ln\ln p}{p}>\alpha_p\ \ (for ~a~large~p).
\end{eqnarray*}
This brings a contradiction when $p$ is large.

\end{proof}

\subsection{The Case for Small Number of Processors $p$}

In this section, we give a lower for the case $p=1,2$. The case $p=1$ is important as it is related to single processor computation.  The case $p=2$ is the simplest parallel computation with two processors.

\begin{theorem}\label{lower-1-2-theorem} In the Parallel-Framework(.) model, for any function  $h(j, i)$, we have
\begin{enumerate}
    \item  
 if the parallel model has   $(1,\alpha_1)$-approximation,  then $\alpha_1\ge 2$.   
 \item 
 if the parallel model has   $(2,\alpha_3)$-approximation,  then  $\alpha_2\ge \frac{\sqrt{5}+1}{2}$.   
\end{enumerate}

\end{theorem}

\begin{proof} 
By Lemma~\ref{lowerbd-lemma}, we have the equation $x^p-x^{p-1}-1=0$ for the cases $p=1,2$. For $p=1$, $x=2$ is the only root. For $p=2$, $x=\frac{\sqrt{5}+1}{2}$ is the root greater than $1$. Therefore, we have $\alpha_1\ge 2$, and $\alpha_2\ge \frac{\sqrt{5}+1}{2}$.






\end{proof}









\begin{theorem} 
For any function  $h(j, i)$, 
 if   Parallel-Framework(.)  has   $(3,\alpha_3)$-approximation,  then we have $\alpha_3\ge r_3$, where $r_3=\frac{1}{3}+\sqrt[3]{\frac{29}{54}+\sqrt{\frac{31}{108}}} +\sqrt[3]{\frac{29}{54}-\sqrt{\frac{31}{108}}}\ge  1.46557$.   
\end{theorem}
\begin{proof}
By Lemma~\ref{lowerbd-lemma}, we have the equation $x^3-x^{2}-1=0$ for the cases $p=3$. With the transformation $x=y+\frac{1}{3}$,   it removes the quadratic term, and  becomes the Cardano's form:
\begin{eqnarray*}
    y^3-\frac{1}{3}y-\frac{29}{27}=0.
\end{eqnarray*}
We have root $y=\sqrt[3]{\frac{29}{54}+\sqrt{\frac{31}{108}}} +\sqrt[3]{\frac{29}{54}-\sqrt{\frac{31}{108}}}$ to satisfy that  $x$ is real number  greater than $1$.
It has root for $x$:
\begin{eqnarray*}
r_3=\frac{1}{3}+\sqrt[3]{\frac{29}{54}+\sqrt{\frac{31}{108}}} +\sqrt[3]{\frac{29}{54}-\sqrt{\frac{31}{108}}}\ge 1.46557.   
\end{eqnarray*}

\end{proof}

Using the numerical solutions, we have lower bounds when  $p$ goes from $4$ to $16$ below:
\begin{eqnarray*}
    \alpha_4&\ge& 1.38027,
    \alpha_5\ge 1.32471,
    \alpha_6\ge 1.28519,
    \alpha_7\ge 1.25542,
    \alpha_8\ge 1.23205,\\
    \alpha_9&\ge& 1.21314,
    \alpha_{10}\ge 1.19749,
    \alpha_{11}\ge 1.18427,
    \alpha_{12}\ge 1.17295,
    \alpha_{13}\ge 1.16311,\\
    \alpha_{14}&\ge& 1.15449,
    \alpha_{15}\ge 1.14685,
    \alpha_{16}\ge 1.14003.
\end{eqnarray*}

\section{Tight Lower Bounds  for $\alpha_p$ with Geometric $h(.,.)$}

In this section, we derive lower bound for $\alpha_p$ when $h(j,i)=T_0b_p^{ip+j}$ for some $b_p>1$. It matches the upper bound for each integer $p\ge 1$.

\begin{theorem}\label{lower-Geometric-thm}
Let function $h(j, i)=b_p^{ip+j}T_0$ for some $b_p>1$. If Parallel-Framework(.) has $(p,\alpha_p)$-approximation, then $\alpha_p\ge (1+\frac{1}{p})(p+1)^{1/p}$.
\end{theorem}

\begin{proof} We fix $p$ and $\alpha_p>1$ (by its definition). At phase $i$, processor $j$ uses $T_{j, i}=h(j,i)=b_p^{ip+j}T_0$ to control the number of  of iterations in GD($x_0, T_{j,i}$).

Define
\begin{eqnarray*}
T_{j, i}^*=\sum_{t=0}^{i} b_p^{tp+j}T_0=b_p^jT_0\sum_{t=0}^{i} b_p^{tp}=b_p^jT_0\cdot \frac{b_p^{(i+1)p}-1}{b_p^p-1}. 
\end{eqnarray*}

We note that $T_{j,t}=b_p^{tp+j}T_0$ is the number of steps in the $t$-th iteration. 
Let $T=b_p^{tp+j-1}+1$ with a large $t$. So, $T_{j,t}=b_p^{tp+j}$ is the least with $T\le T_{j,t}$. We have $T\le T_{j,t}<b_pT$. 

By Definition~\ref{App-def} and 
Lemma~\ref{monotonic-lemma}, 
we have 

\begin{eqnarray*}
\alpha_p T\ge &&T_{j,t}^*=b_p^jT_0\cdot \frac{b_p^{(t+1)p}-1}{b_p^p-1}\\
&=& \frac{b_p^{(t+1)p+j}T_0-b_p^jT_0}{b_p^p-1}\\
&=& \frac{b_p^{p+1}(b_p^{tp+j-1}T_0)-b_p^jT_0}{b_p^p-1}\\
&=& \frac{b_p^{p+1}(T-1)-b_p^jT_0}{b_p^p-1}\\
&=& \frac{b_p^{p+1}T-b_p^{p+1}-b_p^jT_0}{b_p^p-1}\\
&\ge&\frac{b_p^{p+1}T}{b_p^p-1}-\frac{b_p^{p+1}+b_p^jT_0}{b_p^p-1}
\end{eqnarray*}

We have

\begin{eqnarray*}
\alpha_p \ge \frac{b_p^{p+1}}{b_p^p-1}-\frac{b_p^jT_0}{T(b_p^p-1)}
\end{eqnarray*}

Let $f(x)=\frac{x^{p+1}}{x^p-1}$. 
Taking derivative, we have 

\begin{eqnarray*}
f(x)'=\frac{(p+1)x^p(x^p-1)-px^{p-1}\cdot x^{p+1}}{(x^p-1)^2} =\frac{x^{2p}-(p+1)x^p}{(x^p-1)^2}.   
\end{eqnarray*}

So, we let $x^p=p+1$ to have minimal $f(x)=\frac{(p+1)(p+1)^{1/p}}{p}=(1+\frac{1}{p})(p+1)^{1/p}$.

Therefore, 
\begin{eqnarray*}
\alpha_p &\ge& (1+\frac{1}{p})(p+1)^{1/p}-\frac{b_p^jT_0}{T(b_p^p-1)}\\
&\ge& (1+\frac{1}{p})(p+1)^{1/p}-\epsilon \ \ \ (for\ large\ T).
\end{eqnarray*}
Since both $p$ and $\alpha_p$ are fixed in the beginning of this proof and $\epsilon$ is arbitrarily close to zero, we have $\alpha\ge (1+\frac{1}{p})(p+1)^{1/p}$.
\end{proof}

\section{Arithmetically Simple Gradient Descent for Nonconvex Optimization}\label{alg-sec}

In this section, we present an arithmetically simple static gradient descent method for nonconvex optimization. The algorithm takes the iteration budget $T$ as input, which determines the total number of gradient descent iterations. The step size is computed using a denominator of the form $2^t$, where the integer $t$ is determined from $T$.

This design makes the algorithm particularly suitable for hardware implementation. Since every denominator is a power of two, division operations can be replaced by binary shift operations, eliminating expensive floating-point division. Moreover, the algorithm avoids square-root computations altogether. These arithmetic simplifications make the proposed method attractive for hardware accelerators and chip implementations.

The convergence analysis in this section also explains the motivation for the proposed parallel framework. Because the convergence guarantee of the static gradient method depends on selecting an appropriate iteration budget $T$, the parallel framework searches for a suitable value of $T$ adaptively while preserving the simplicity of the underlying static algorithm.


\subsection{Notations for Gradient Descent}

 A vector in $\mathbb{R}^m$ is $(a_1, a_2,\ldots, a_m)$ with $a_i\in \mathbb{R}$ for $i=1,2,\cdots, m$. The inner product between two vectors $V=(v_1,v_2,\ldots, v_m)$ and $U=(u_1,u_2,\ldots, u_m)$ is denoted by $\langle U, V\rangle=\sum_{i=1}^m u_iv_i$. The length of a vector $V=(v_1,v_2,\ldots, v_m)$  is denoted by $\Vert V\Vert=\sqrt{v_1^2+v_2^2+\ldots+v_m^2}$. For a differentiable function $F(x_1,x_2,\cdots, x_m):\mathbb{R}^m\rightarrow \mathbb{R}$, its gradient at a point $(x_1,x_2,\cdots, x_m)$ is 
 \begin{eqnarray*}
   &&\bigtriangledown F(x_1,x_2,\cdots, x_m)\\
   &=&\left(\frac{\partial F(x_1,x_2,\cdots, x_m)}{\partial x_1}, \frac{\partial F(x_1,x_2,\cdots, x_m)}{\partial x_2},\ldots, \frac{\partial F(x_1,x_2,\cdots, x_m)}{\partial x_m}\right).    
 \end{eqnarray*}

  In the rest of this paper, let $x^*\in \mathbb{R}^m$ be a point with $F(x^*)=\inf_x\{F(x)\}$ if $\inf_x\{F(x)\}>-\infty$.  The expectation on a random variable $\xi$ is expressed $\Expect_{\xi}(.)$. For example, a stochastic gradient $G(\xi, x)$ for function $F(x)$ may satisfy the condition $\Expect_{\xi}(\xi, x)=\bigtriangledown F(x)$, which is often assumed in many SGD algorithms. 
In this section, we give some theoretical results about the rate of convergence. The following two conditions are often assumed for non-convex optimization.

\begin{itemize}
    \item 
    {\bf $L$-Lipschitz} smoothness: $\Vert \bigtriangledown(F(x))-\bigtriangledown(F(y))\Vert\le L\Vert x-y\Vert$.

\item $F^*=\inf_{x} F(x)>-\infty$

\end{itemize}
Let $C_L^1$ be the class of functions that are $L$-Lipschitz smooth.
The following Lemma~\ref{basic-lemma} , which is often mentioned in existing publications, can be easily proven by $L$-Lipschitz condition and Taylor expansion (See~\cite{Lan2020}). 

\begin{lemma}\label{basic-lemma}
    Let $F(x_1,\cdots, x_d)$ be a function $\mathbb{R}^d\rightarrow \mathbb{R}$ in $C_L^1$, we have $F(x)\le F(y)+(\bigtriangledown F(y), x-y)+\frac{L}{2}\Vert x-y\Vert^2$.
\end{lemma}

The stochastic gradient is controlled by three parameters $\lambda, \sigma_0$, and $\sigma_1$. It is given in Definition~\ref{condition}.

\begin{definition}\label{condition} A $(\lambda,\sigma_0,\sigma_1)$-stochastic gradient $G(\xi,x)$ for $F(x):\mathbb{R}^m\rightarrow\mathbb{R}$ is that $\xi$ is a random variable and $G(\xi, x)$ is an approximation for $\bigtriangledown F(x)$ satisfying the conditions:
 
\begin{enumerate}
    \item\label{condition1-def} $\langle\Expect_{\xi}(G(\xi, x)),\bigtriangledown F(x)\rangle \ge \lambda\Vert \bigtriangledown F(x)\Vert^2$ for some $\lambda\in (0,+\infty)$, and
    \item\label{condition2-def} 
    $\Expect_{\xi}(\Vert \bigtriangledown F(x)-G(\xi, x)\Vert^2)\le\sigma_0^2+\sigma_1^2\Vert \bigtriangledown F(x)\Vert^2$ for some $\sigma_0,\sigma_1\in [0,+\infty)$.
\end{enumerate}

\end{definition}

A standard stochastic model, which is broadly used in stochastic gradient descent,  is the special case with $\lambda=1$ and $\sigma_1=0$. Our stochastic model is more general, and fits the convergence analysis for our  algorithm.

\begin{lemma}\label{stochastic-bound-lemma}
    Assume $F(x)$ and $G(\xi, x)$ satisfy 
        $\Expect_{\xi}(\Vert \bigtriangledown F(x)-G(\xi, x)\Vert^2)\le\sigma_0^2+\sigma_1^2\Vert \bigtriangledown F(x)\Vert^2$ for some $\sigma_0,\sigma_1\in [0,+\infty)$.
    Then $\Expect_{\xi}(\Vert G(\xi, x)\Vert^2)\le 2\sigma_0^2+ (2+2\sigma_1^2)\Vert \bigtriangledown F(x)\Vert^2$.
\end{lemma}

\begin{proof} We have inequalities
\begin{eqnarray*}
    &&\Vert G(\xi, x)\Vert^2=\langle G(\xi, x),G(\xi, x)\rangle\\
    &=&\langle G(\xi, x)-\bigtriangledown F(x)+\bigtriangledown F(x),G(\xi, x)-\bigtriangledown F(x)+\bigtriangledown F(x)\rangle\\
    &\le& \Vert G(\xi, x)-\bigtriangledown F(x)\Vert^2+\Vert \bigtriangledown F(x)\Vert^2+2\langle G(\xi, x)-\bigtriangledown F(x), \bigtriangledown F(x)\rangle\\
    &\le& \Vert G(\xi, x)-\bigtriangledown F(x)\Vert^2+\Vert \bigtriangledown F(x)\Vert^2+\Vert G(\xi, x)-\bigtriangledown F(x)\Vert^2+ \Vert\bigtriangledown F(x)\Vert^2\\
    &&(by~Cauchy-Schwarz~inequality)\\    
    &=& 2\Vert G(\xi, x)-\bigtriangledown F(x)\Vert^2+2\Vert \bigtriangledown F(x)\Vert^2.
\end{eqnarray*}
Therefore,

\begin{eqnarray*}
    \Expect_{\xi}(\Vert G(\xi, x)\Vert^2)&\le&\Expect( 2\Vert G(\xi, x)-\bigtriangledown F(x)\Vert^2+2\Vert \bigtriangledown F(x)\Vert^2)\\
&=&2\Expect( \Vert G(\xi, x)-\bigtriangledown F(x)\Vert^2)+2\Vert \bigtriangledown F(x)\Vert^2\\
&\le&2(\sigma_0^2+\sigma_1^2\Vert\bigtriangledown F(x)\Vert^2)+2\Vert \bigtriangledown F(x)\Vert^2\\
&=&2\sigma_0^2+(2+2\sigma_1^2)\Vert\bigtriangledown F(x)\Vert^2.
\end{eqnarray*}

\end{proof}

\subsection{A Static Gradient Method}

We give a static gradient descent algorithm in this section. The learning rate is computed based on  one of the parameters.

\begin{definition}
  A gradient descent method is {\it arithmetically simple} if the operations are limited to  $+,-, \times$, and  division $x/y$ with $y=2^t$ for some integer $t$.
\end{definition}

We present a version of SGD that is arithmetically simple. When the stochastic gradient oracle $G(\cdot)$ is treated as a black box, the algorithm requires neither floating-point division nor square-root computations.

\vskip 20pt
{\bf Algorithm} SGD$(G(.,.), \eta,  x_0, t, T)$

Input: 

\begin{itemize}
\item
$G(\xi, x):\mathbb{R}^m\rightarrow \mathbb{R}$ is an stochastic approximation for $\bigtriangledown F(x)$,

    \item $\eta\in (0,+\infty)$, 



    \item $x_0\in \mathbb{R}^m$ is the start point, 

    \item $t$ is an integer to control rate, 

    \item $T$ is for the number of steps 
\end{itemize}

Steps:

\begin{enumerate}[label=\arabic*.]


 \item Let $x_{1}=x_0$



    \item Let $s_t=2^t$
   
    \item Let $j=1$

\item while $i\le T$

   \item $\{$

\item\qquad $g_i=G(\xi_{j},x_{j})$
       
    \item \qquad
     $x_{j+1}=x_{j}-\frac{\eta}{ s_t}\cdot g_i$ 

    \item \qquad  $j=j+1$

\item $\}$


\end{enumerate}

{\bf End of Algorithm}

\subsection{Convergence at $(\lambda,\sigma_0,\sigma_1)$-Stochastic Model}

The convergence of the algorithm at $(\lambda,\sigma_0,\sigma_1)$-Stochastic Model is proven in this section. With 
$t=\ceiling{\frac{\ceiling{\log_2T}}{2}}$, it converges to a stationary point with rate $\Omega\left(\frac{1}{ \sqrt{T}}\right)$.



Lemma~\ref{up-adjustment1-lemma} derives an upper bound by summing the inequalities in Lemma~\ref{basic-lemma} for the gradient method. It then follows that at least one iterate generated by SGD(.) has a gradient whose expected norm is close to zero. As the step size $s_t$ does not change with a given $T$, the convergence analysis turns out to be simple.

\vskip 20pt

\begin{lemma}\label{up-adjustment1-lemma}
Assume $F(x)$ is $L$-Lipschitz  smooth and $G(\xi, x)$ satisfies the condition in Definition~\ref{condition}.
Then
  \begin{eqnarray*}
&&\sum_{j=1}^T\Expect\left(\frac{\eta\lambda}{ s_t}\Vert \bigtriangledown F(x_{j})\Vert^2\right)-\frac{\eta^2 L}{ 2s_t^{2}}\sum_{j=1}^T \Expect (\Vert G(\xi_{j}, x_{j})\Vert^2)\le F(x_1)-F(x^*)
   \end{eqnarray*}   
\end{lemma}

\begin{proof}
   As $F(x)$ is $L$-Lipschitz smooth, by Lemma~\ref{basic-lemma}, we have 
   \begin{eqnarray*}
   F(x_{j+1})&\le& F(x_{j})+\langle\bigtriangledown F(x_{j}), x_{j+1}-x_{j}\rangle+\frac{L}{ 2}\Vert x_{j+1}-x_{j}\Vert^2\\
   &=& F(x_{j})-\frac{\eta}{s_t }\langle \bigtriangledown F(x_{j}), G(\xi_j,x_{j})\rangle+\frac{L}{2}\Vert x_{j+1}-x_{j}\Vert^2\\
   &\le & F(x_{j})-\frac{\eta}{ s_t}\langle \bigtriangledown F(x_{j}), G(\xi_{j},x_{j})\rangle+\frac{\eta^2 L}{ 2s_t^2}\Vert G(\xi_{j},x_{j})\rangle\Vert^2\\
   &\le & F(x_{j})-\frac{\eta\lambda}{ s_t}\cdot \Vert\bigtriangledown F(x_{j})\Vert^2+\frac{\eta^2 L}{2s_t^2}\Vert G(\xi_j,x_{j})\Vert^2\\
   &&\ \ \ (by\ Condition~\ref{condition1-def}\ in\  Definition~\ref{condition})\\
   \end{eqnarray*}

Thus,
   \begin{eqnarray*}
&&\left(\frac{\eta\lambda}{s_t}(\Vert\bigtriangledown F(x_{j})\Vert^2\right)-\frac{\eta^2 L}{ 2s_t^{2}}\Vert G(\xi_{j},x_{j})\Vert^2\le F(x_{j})-F(x_{j+1}). \end{eqnarray*}

We have    \begin{eqnarray*}
&&\sum_{j=1}^T\Expect\left(\frac{\eta\lambda}{ s_t}\Vert \bigtriangledown F(x_{j})\Vert^2\right)-\sum_{j=1}^T \Expect (\frac{\eta^2 L}{ 2s_t^{2}}(\Vert G(\xi_{j}, x_{j})\Vert^2)\le F(x_1)-F(x^*).
   \end{eqnarray*}   
Thus,
   \begin{eqnarray*}
&&\sum_{j=1}^T\Expect\left(\frac{\eta\lambda}{ s_t}\Vert \bigtriangledown F(x_{j})\Vert^2\right)-\frac{\eta^2 L}{2s_t^{2}}\sum_{j=1}^T  \Expect (\Vert G(\xi_{j}, x_{j})\Vert^2)\le F(x_1)-F(x^*).
   \end{eqnarray*}   
\end{proof}

Lemma~\ref{up-adjustment-lemma} shows that one of the iterates generated by SGD(.) has a gradient whose expected norm is close to zero. Consequently, SGD(.) converges to a stationary point.

\begin{lemma}\label{up-adjustment-lemma}
Assume $F(x)$ is $L$-Lipschitz-smooth and $G(\xi, x)$ satisfies the condition in Definition~\ref{condition}. Assume that $s_t\in[\sqrt{T}, 4\sqrt{T}]$ and $T$ satisfy the conditions:

\begin{eqnarray}
\frac{\eta L(1+\sigma_1^2)}{s_t \lambda}&\le& \frac{1}{2}\label{first0-ineqn}
\end{eqnarray}
Then $\min_i\Expect(\Vert \bigtriangledown F(x_i)\Vert^2)\le \frac{U(\eta,\sigma_0, \lambda, L)}{\sqrt{T}}$, 
where 
\begin{eqnarray}
U(\eta,\sigma_0, \lambda, L)=\left(\frac{4}{\eta \lambda}(F(x_0)-F(x^*))+\frac{2\eta\sigma_0^2 L}{ \lambda}\right).\label{def-U-eqn}  
\end{eqnarray}
\end{lemma}

\begin{proof} By Lemma~\ref{up-adjustment1-lemma}, we have
   \begin{eqnarray*}
&&\sum_{j=1}^T\Expect\left(\frac{\eta(\lambda)}{s_t}\Vert \bigtriangledown F(x_{j})\Vert^2\right)-\frac{\eta^2 L}{ 2s_t^{2}}\sum_{j=1}^T \Expect (\Vert G(\xi_{j}, x_{j})\Vert^2)\le F(x_1)-F(x^*)
   \end{eqnarray*}   

By Lemma~\ref{stochastic-bound-lemma}, we have
   \begin{eqnarray*}
&&\sum_{j=1}^T\Expect\left(\frac{\eta \lambda}{ s_t}\Vert \bigtriangledown F(x_{j})\Vert^2\right)-\sum_{j=1}^T \frac{\eta^2 L}{ 2s_t^{2}}\Expect(2\sigma_0^2+ (2+2\sigma_1^2)\Vert \bigtriangledown F(x)\Vert^2)\le F(x_1)-F(x^*)
   \end{eqnarray*}

   \begin{eqnarray*}
&&\sum_{j=1}^T\frac{\eta\lambda}{s_t}\left(1-\frac{\eta L(1+\sigma_1^2)}{s_t \lambda}\right)\Expect(\Vert \bigtriangledown F(x_{j})\Vert^2)-\sum_{j=1}^T \frac{\eta^2 \sigma_0^2 L}{s_t^{2}}
\le F(x_1)-F(x^*)
   \end{eqnarray*}

   \begin{eqnarray*}
&&\sum_{j=1}^T\left(\frac{\eta\lambda}{ 2s_t}\Expect(\Vert \bigtriangledown F(x_{j})\Vert^2-\frac{\eta^2 \sigma_0^2 L}{s_t^{2}}\right)\le F(x_1)-F(x^*)\ \ \ (by~inequality~(\ref{first0-ineqn}))
   \end{eqnarray*}

 We have

   \begin{eqnarray*}
&&\sum_{j=1}^T \frac{\eta\lambda}{ 2s_t}\Expect(\Vert \bigtriangledown F(x_{j})\Vert^2)- \frac{\sigma_0^2\eta^2L T}{s_t^{2}}\le F(x_1)-F(x^*)
   \end{eqnarray*}

   \begin{eqnarray*}
\frac{\eta\lambda}{ 2s_t}\cdot T\min_{1\le j\le T} \Expect(\Vert \bigtriangledown F(x_{j})\Vert^2 )
&\le& \sum_{j=1}^T\frac{\eta\lambda}{2s_t}\Expect(\Vert \bigtriangledown F(x_{j})\Vert^2)\\
&\le& (F(x_1)-F(x^*))+\frac{\sigma_0^2\eta^2L T }{s_t^{2}}.
   \end{eqnarray*} 

We have

   \begin{eqnarray*}
&&\min_{1\le j\le T} \Expect(\Vert \bigtriangledown F(x_{j})\Vert^2 )
\le \frac{s_t}{T\eta \lambda}\left(F(x_1)-F(x^*)\right)+\frac{2\eta\sigma_0^2 L }{ s_t\lambda}\\
&\le& \frac{4}{\sqrt{T}\eta \lambda}(F(x_1)-F(x^*))+\frac{2\eta\sigma_0^2 L }{ \sqrt{T}\lambda}\\
&\le& \frac{1}{\sqrt{T}}\left(\frac{4}{ \eta \lambda}(F(x_1)-F(x^*))+\frac{2\eta\sigma_0^2 L }{\lambda}\right)\\
&=&\frac{1}{\sqrt{T}}\left(\frac{4}{\eta \lambda}(F(x_0)-F(x^*))+\frac{2\eta\sigma_0^2 L }{\lambda}\right).
\label{final-bound}
   \end{eqnarray*}    
\end{proof}

Assume $F(.)$ and it stochastic gradient  $G(\xi,x)$ satisfy the conditions in Definition~\ref{condition}. 
Theorem~\ref{main-thm} shows that the gradient descent algorithm SGD(.) converges to a stationary point at rate $\Omega\left(\frac{1}{\sqrt{T}}\right)$.

\begin{theorem}\label{main-thm}
Suppose $F(.)$ is in $\mathbb{C}_L^1$ and $\inf_x F(x)>-\infty$. Function $G(\xi,x)$ satisfies the conditions in Definition~\ref{condition}. Assume that integers $T$ and $t$ satisfy the conditions
\begin{eqnarray}
T&\ge& \left(\frac{2\eta L(1+\sigma_1^2)}{  \lambda}\right)^2\label{condition-T-ineqn}\\
t&=&\ceiling{\ceiling{\log_2T}/2}\label{condition-t-ineqn2}.    
\end{eqnarray}
 Then the algorithm SGD$(G(.,.), \eta,  x_0, t, T)$ is arithmetically simple and has   
 \begin{eqnarray*}
 \min_{1\le i\le T}\Expect(\Vert \bigtriangledown F(x_i)\Vert^2)\le \frac{U(\eta,\sigma_0, \lambda, L)}{ \sqrt{T}},     
 \end{eqnarray*}
 where $U(.)$ is given at equation (\ref{def-U-eqn}).
\end{theorem}

\begin{proof}
Let $T$ be the number of steps to run. We select $a=\ceiling{\log_2 T}$, which is the number of bits if $T$ is in binary format (for example, number $7$ has binary format $111$, and $\ceiling{\log_2 7}=3$). We have $T\le 2^a\le 2T$. Let $t=\ceiling{\frac{a}{ 2}}\le \frac{a}{ 2}+1$. We have $\sqrt{T}\le 2^\frac{a}{ 2}\le 2^t\le 2\sqrt{2^a}\le 2\sqrt{2T}< 4\sqrt{T}$. Thus, $s_t=2^t\in [\sqrt{T}, 4\sqrt{T}]$.
With the condition $T\ge \left(\frac{2\eta L(1+\sigma_1^2)}{  \lambda}\right)^2$, we have
$s_t\ge \sqrt{T}\ge \left(\frac{2\eta L(1+\sigma_1^2)}{  \lambda}\right)$. So, inequality (\ref{first0-ineqn}) is satisfied.

Run 
GD$(G(.), \eta, x_0, t, T)$. We have $\min_{1\le i\le T}\Expect(\Vert \bigtriangledown F(x_i)\Vert^2)\le \frac{U(\eta,\sigma_0, \lambda, L)}{\sqrt{T}}$ by Lemma~\ref{up-adjustment-lemma}.


    
\end{proof}

\vskip 20pt
{\bf Algorithm} Static1-SGD$(   x_0,T)$

Input: 
\begin{itemize}
    \item $x_0$ is the start point

    \item $T$ is the number of iterations
\end{itemize}

Steps:

\begin{enumerate}[label=\arabic*.]

\item 
Assign to $t$ as equation (\ref{condition-t-ineqn2}) 
\item
SGD$(G(.,.), \eta,  x_0, t, T)$

\end{enumerate}

{\bf End of Algorithm}

\vskip 20pt

\begin{corollary}\label{SGD-cor1} Let $\delta\in (0,1)$.
Suppose $F(.)$ is in $\mathbb{C}_L^1$ and $\inf_x F(x)>-\infty$. Function $G(\xi,x)$ satisfies the conditions in Definition~\ref{condition}. Assume $T$ satisfies (\ref{condition-T-ineqn}). Then with probability at least $1-\delta$, the algorithm Static1-SGD$(G(.,.), \eta,  x_0, t, T)$ is arithmetically simple and  has  $\min_i (\Vert \bigtriangledown F(x_i)\Vert^2)\le \frac{U(\eta,\sigma_0, \lambda, L)}{ \delta\sqrt{T}}$.
\end{corollary}

\begin{proof}
It follows Theorem~\ref{main-thm} and Markov inequality $\Prob(X\ge \frac{\Expect(X)}{\delta})\le \delta$. The proof of Theorem~\ref{main-thm} also shows that $t$ is computed via a arithmetically simple way.
    
\end{proof}

\subsection{Faster Convergence in $(\lambda,0, \sigma_1)$ Stochastic Model}

In this section we show a faster adaptive gradient descent analysis in $(\lambda,0, \sigma_1)$ stochastic Model. It is convergence rate is almost linear, and  faster than the general $(\lambda,\sigma_0, \sigma_1)$ stochastic Model.

\begin{definition}\label{condition2} A $(\lambda, 0,\sigma_1)$-stochastic gradient $G(\xi,x)$ for $F(x):\mathbb{R}^m\rightarrow\mathbb{R}$ is that $\xi$ is a random variable and $G(\xi, x)$ is an approximation for $\bigtriangledown F(x)$ satisfying the conditions:
 
\begin{enumerate}
    \item $\langle\Expect_{\xi}(G(\xi, x)),\bigtriangledown F(x)\rangle \ge \lambda\Vert \bigtriangledown F(x)\Vert^2$ for some $\lambda\in (0,+\infty)$, and
    \item\label{condition2b-def} 
    $\Expect_{\xi}(\Vert \bigtriangledown F(x)-G(\xi, x)\Vert^2)\le\sigma_1^2\Vert \bigtriangledown F(x)\Vert^2$ for some $\sigma_1\in [0,+\infty)$.
\end{enumerate}
   
\end{definition}

\begin{lemma}\label{up-adjustment-lemma2}
Assume $F(x)$ is $L$-smooth and $G(\xi, x)$ satisfies the condition in Definition~\ref{condition2}. 
Assume $T$ and $s_t$ satisfy the following conditions:

\begin{eqnarray}
\frac{2\eta (2+2\sigma_1^2) L}{ s_t\lambda}&\le& \frac{1}{ 2}\label{second-lemma-ineqn}
\end{eqnarray}
Then
\begin{eqnarray*}
&&\min_{1\le j\le T} \Expect(\Vert \bigtriangledown F(x_{j})\Vert^2 )
\le \frac{2s_t}{T\eta\lambda}(F(x_0)-F(x^*)).
   \end{eqnarray*}  
\end{lemma}

\begin{proof}
By Lemma~\ref{up-adjustment1-lemma}, we have    \begin{eqnarray*}
&&\sum_{j=1}^T\Expect\left(\frac{\eta\lambda}{ s_t}\Vert \bigtriangledown F(x_{j})\Vert^2\right)-\frac{2\eta^2 L}{ s_t^{2}}\sum_{j=1}^T \Expect (\Vert G(\xi_{j}, x_{j})\Vert^2)\le F(x_1)-F(x^*).
   \end{eqnarray*} 
By Lemma~\ref{stochastic-bound-lemma}, we have 
   \begin{eqnarray*}
&&\sum_{j=1}^T\Expect\left(\frac{\eta\lambda}{ s_t}\Vert \bigtriangledown F(x_{j})\Vert^2\right)-\sum_{j=1}^T \frac{2\eta^2 L}{s_t^{2}}((2+2\sigma_1^2)\Vert \bigtriangledown F(x_{j})\Vert^2)\le F(x_1)-F(x^*).
   \end{eqnarray*}

   \begin{eqnarray*}
&&\sum_{j=1}^T\left(\frac{\eta\lambda}{s_t}-\frac{2\eta^2 (2+2\sigma_1^2) L}{ s_t^{2}}\right)\Expect(\Vert \bigtriangledown F(x_{j})\Vert^2)
\le F(x_1)-F(x^*).
   \end{eqnarray*} 

   \begin{eqnarray*}
&&\sum_{j=1}^T\frac{\eta\lambda}{s_t}\left(1-\frac{2\eta (2+2\sigma_1^2) L}{ s_t\lambda}\right)\Expect(\Vert \bigtriangledown F(x_{j})\Vert^2)
\le F(x_1)-F(x^*).
   \end{eqnarray*}

By inequality~(\ref{second-lemma-ineqn}), we have 
   \begin{eqnarray*}
&&\sum_{j=1}^T\left(\frac{\eta\lambda}{ 2s_t}\Expect(\Vert \bigtriangledown F(x_{j})\Vert^2\right)\le F(x_1)-F(x^*).
   \end{eqnarray*}

 We have

   \begin{eqnarray*}
T\cdot \left(\frac{\eta\lambda}{2s_t}\right)\min_{j} \Expect(\Vert \bigtriangledown F(x_{j})\Vert^2 )\le \sum_{j=1}^T(\frac{\eta}{ 2s_t})\Expect(\Vert \bigtriangledown F(x_{j})\Vert^2)\le (F(x_1)-F(x^*)).
   \end{eqnarray*} 

This brings inequality

   \begin{eqnarray*}
&&\min_{1\le j\le T} \Expect(\Vert \bigtriangledown F(x_{j})\Vert^2 )
\le \frac{2s_t}{ T\eta\lambda}(F(x_1)-F(x^*))=\frac{2s_t}{ T\eta\lambda}(F(x_0)-F(x^*)).\label{final2-bound}
   \end{eqnarray*}    
\end{proof}

Assume $F(.)$ and it stochastic gradient  $G(\xi,x)$ satisfy the conditions in Definition~\ref{condition2}. 
Theorem~\ref{main2-thm} shows that the gradient descent algorithm SGD(.) converges to a stationary point at rate $\Omega\left(\frac{1}{T^{1-\beta_0}}\right)$ for any $\beta_0\in (0,1)$. The parameter $t$ depends on parameter $\beta_0$.

\begin{theorem}\label{main2-thm}Let $\beta_0\in (0,1)$.
Suppose $F(.)$ is in $\mathbb{C}_L^1$ and $\inf_x F(x)>-\infty$. Function $G(\xi,x)$ satisfies the conditions in Definition~\ref{condition2}.  Assume 
\begin{eqnarray}
T&\ge& \left(\frac{2\eta L(1+\sigma_1^2)}{  \lambda}\right)^\frac{1}{ \beta_0}\label{condition-T2-ineqn}\\ t&=&\ceiling{\ceiling{\log_2T}\cdot \beta_0}\label{condition-t2-ineqn}.    
\end{eqnarray}
 Then the algorithm SGD$(G(.,.), \eta,  x_0, t, T)$ is arithmetically simple and  has 
 
 \begin{eqnarray}
\min_{1\le j\le T} \Expect(\Vert \bigtriangledown F(x_{j})\Vert^2 )
\le \frac{V(\eta,\lambda,L)}{T^{1-\beta_0}},\label{final3-bound}
   \end{eqnarray}   
   
 where $V(\eta,\lambda,L)=\frac{2}{\eta\lambda}(F(x_0)-F(x^*))$.
\end{theorem}

\begin{proof}Let $T$ be the number of steps to run. We select $a=\ceiling{\log_2 T}$. We have $T\le 2^a\le 2T$. Let $t=\ceiling{a\beta_0}\le {a\beta_0}+1$. We have $T^{\beta_0}\le 2^{a\beta_0}\le 2^t\le 2^{a\beta_0+1}=2\cdot (2^a)^{\beta_0}\le 2\cdot (2T)^{\beta_0}\le 4 T^{\beta_0}$. Thus, $s_t=2^t\in [T^{\beta_0}, 4T^{\beta_0}]$.
With the condition $T\ge \left(\frac{2\eta L(1+\sigma_1^2)}{  \lambda}\right)^\frac{1}{ \beta_0}$, we have
$s_t=2^t\ge T^{\beta_0}\ge \left(\frac{2\eta L(1+\sigma_1^2)}{  \lambda}\right)$. So, inequality (\ref{second-lemma-ineqn}) is satisfied.

Run the algorithm with SGD$(G(.), \eta,  x_0, t, T)$. By Lemma~\ref{up-adjustment-lemma2}, we get inequality (\ref{final3-bound}).

\end{proof}

\vskip 20pt
{\bf Algorithm} Static2-SGD$(   x_0,T)$

Input: 
\begin{itemize}
    \item $x_0$ is the start point

    \item $T$ is the number of iterations
\end{itemize}

Steps:

\begin{enumerate}[label=\arabic*.]

\item 
Assign to $t$ as equation (\ref{condition-t2-ineqn}) 
\item
SGD$(G(.,.), \eta,  x_0, t, T)$

\end{enumerate}

{\bf End of Algorithm}

\vskip 20pt

\begin{corollary}\label{SGD-cor}
Let $\delta\in (0,1)$. Let $\beta_0\in (0,1)$.
Suppose $F(.)$ is in $\mathbb{C}_L^1$ and $\inf_x F(x)>-\infty$. Function $G(\xi,x)$ satisfies the conditions in Definition~\ref{condition2}.  Assume $T$ satisfies condition (\ref{condition-T2-ineqn}).
Then with probability at least $1-\delta$, the algorithm Static2-SGD$(G(.,.), \eta,  x_0, t, T)$ is arithmetically simple and  has \begin{eqnarray*}
&&\min_{1\le j\le T} (\Vert \bigtriangledown F(x_{j})\Vert^2 )
\le \frac{V(\eta,\lambda,L)}{ \delta T^{1-\beta_0}}.\label{final2b-bound}
   \end{eqnarray*}  
\end{corollary}

\begin{proof}
 It follows Theorem~\ref{main2-thm} and Markov inequality $\Prob(X\ge \frac{\Expect(X)}{\delta})\le \delta$.  
 The proof of Theorem~\ref{main2-thm} also shows that $t$ is computed via a arithmetically simple way.
\end{proof}






\subsection{Why Do Static Gradient Methods Need a Parallel Framework?}

Theorems~\ref{main-thm} and~\ref{main2-thm} require Conditions~(\ref{condition-T-ineqn}) and~(\ref{condition-T2-ineqn}), respectively, to guarantee convergence. In practice, however, the parameters $L$, $\eta$, and $\sigma_1$ are typically unknown or difficult to estimate accurately. Consequently, selecting an appropriate iteration budget $T$ in advance is a challenging task. Our parallel framework addresses this difficulty by allowing multiple processors to search over a geometric sequence of candidate values $T_{j,i}$ simultaneously until one of them satisfies the required convergence conditions.

To illustrate the motivation, suppose a single processor tests the iteration budgets
\[
1,2,,2^2,,\ldots,,2^m,,\ldots,
\]
and assume that the smallest satisfactory choice is $T=2^m+1$. Before reaching the first candidate that is at least $T$, namely $2^{m+1}$, the processor must execute
\[
1+2+\cdots+2^m=2^{m+1}-1=2T-3
\]
iterations. Thus, a substantial amount of computation is wasted before identifying a suitable iteration budget. Our parallel framework significantly reduces this overhead by distributing the search across multiple processors. Theorem~\ref{lower-Geometric-thm} establishes a nearly tight lower bound on the unavoidable gap between the target iteration budget $T$ and the cumulative number of iterations $T_{j,i}^*$ executed before reaching it.



\section{Better Adaptivity via the Parallel Framework}

In this section, we apply the parallel framework to static gradient descent, thereby improving the adaptivity of the gradient descent algorithm. Both conditions, (\ref{condition-T-ineqn}) and (\ref{condition-T2-ineqn}), rely on choosing the parameter $T$ to be sufficiently large.

\begin{definition}
    Let $GD(x_0, T)$ be a gradient descent method for a function $F(x)$. The output of $GD(x_0, T)$ is the list  $Z=\langle x_1, x_2,\ldots,x_{\floor{T}}\rangle$  generated in its $\floor{T}$ iterations. It is denoted by $Z=GD(x_0, T)$.
\end{definition}


\begin{definition}
    Let GD$(x_0, T)$ be a gradient descent method for a function $F(x)$. A list of elements $\langle x_1,x_2,\ldots, x_{\floor{T_{j,i}}}\rangle$ is from Parallel-Framework(GD(.),.) at $\langle j,i\rangle$ if $\langle x_1,\ldots, x_{\floor{T_{j,i}}}\rangle$ is the output  of  GD$(x_0, T_{j, i})$.
\end{definition}

\begin{proposition}\label{embeding-prop}
 Let GD($x_0, T)$ be a gradient.
 Assume that
 
 Parallel-Framework$(GD(.),.)$ has $(p,\alpha_p)$-approximation.  Then for any $T\ge T_0$,
 Parallel-Framework$(GD(.))$ executes
 $GD(x_0, T_{j,i})$ satisfying $T\le T_{j,i}\le T_{j,i}^*\le \alpha_p T$.  
\end{proposition}

\begin{proof} 
By the condition of $(p, \alpha_p)$-approximation, we have $T\le T_{j,i}\le T_{j,i}^*\le \alpha_p T$. 
It follows from Theorem~\ref{main2b-thm}.
\end{proof}



We embed the first static gradient descent to Parallel framework and have  Theorem~\ref{embed-thm1} about its convergence. It shows that some processor $j$ generates a list of points at a stage $i$  converging to a stationary points with a high probability.

\begin{theorem}\label{embed-thm1} 
Suppose $F(.)$ is in $\mathbb{C}_L^1$ and $\inf_x F(x)>-\infty$. Function $G(\xi,x)$ satisfies the conditions in Definition~\ref{condition2}.  
Assume that $T$ satisfies (\ref{condition-T-ineqn}).
Then with probability at least $1-\delta$, Parallel-Framework(Static1-SGD(.,.),  $ x_0, h(.,.), T_0,p)$ generates $\langle x_1, x_2,\ldots, x_{\floor{T_{j,i}}}\rangle$ at $\langle j, i\rangle$

\begin{eqnarray*}
&&\min_{1\le j\le \floor{T_{j,i}}} (\Vert \bigtriangledown F(x_{j})\Vert^2 )
\le \frac{U(\eta,\sigma_0, \lambda, L)}{\delta \sqrt{T}}.\label{final2d-bound}
   \end{eqnarray*}   after running $T_{j,i}^*$ iterations with $T\le T_{j,i}\le T_{j,i}^*\le \alpha_p T$.
\end{theorem}

\begin{proof}By Proposition~\ref{embeding-prop}, Parallel-Framework(Static1-SGD(.,.),  $ x_0, h(.,.), T_0,p)$ executes Static1-SGD$(x_0, T_{j,i})$ that generates $\langle x_1, x_2,\ldots, x_{\floor{T_{j,i}}}\rangle$ at $\langle j, i\rangle$ and has $T\le T_{j,i}\le T_{j,i}^*\le \alpha_p T$. 
It follows from Corollary~\ref{SGD-cor1}.
    
\end{proof}

We embed the gradient descent Static2-SGD(.) to Parallel framework in Theorem~\ref{embed-thm2}. It has a faster convergence.

\begin{theorem}\label{embed-thm2} Let $\delta\in (0,1)$. Let $\beta_0\in (0,1)$.
Suppose $F(.)$ is in $\mathbb{C}_L^1$ and $\inf_x F(x)>-\infty$. Function $G(\xi,x)$ satisfies the conditions in Definition~\ref{condition2}.  
Assume that $T$ satisfies  (\ref{condition-T2-ineqn}).
Then with probability at least $1-\delta$, Parallel-Framework(Static2-SGD(.,.),  $ x_0, h(.,.), T_0,p)$ generates $\langle x_1, x_2,\ldots, x_{\floor{T_{j,i}}}\rangle$ at $\langle j, i\rangle$

\begin{eqnarray*}
&&\min_{1\le j\le \floor{T_{j,i}}} (\Vert \bigtriangledown F(x_{j})\Vert^2 )
\le \frac{V(\eta,\lambda,L)}{\delta T^{1-\beta_0}}.\label{final2c-bound}
   \end{eqnarray*}   after running $T_{j,i}^*$ iterations at process $j$  with $T\le T_{j,i}\le T_{j,i}^*\le \alpha_p T$.
\end{theorem}

\begin{proof} By Proposition~\ref{embeding-prop}, Parallel-Framework(Static2-SGD(.,.),  $ x_0, h(.,.), T_0,p)$ executes Static2-SGD$(x_0, T_{j,i})$ that generates $\langle x_1, x_2,\ldots, x_{\floor{T_{j,i}}}\rangle$ at $\langle j, i\rangle$ and has $T\le T_{j,i}\le T_{j,i}^*\le \alpha_p T$.
It follows from Corollary~\ref{SGD-cor}.
    
\end{proof}

The adaptivity is achieved by embedding static gradient descent into the parallel framework. When $T$ is sufficiently large, conditions (\ref{condition-T-ineqn}) and (\ref{condition-T2-ineqn}) are satisfied. The parameters $t$ and $s_t$ are chosen independently of $L$, $\eta$, $\lambda$, $\sigma_0$, and $\sigma_1$. The parallel framework employs a geometric sequence, determined by the number $p$ of processors, to search for an appropriate value of $T$. The step size $s_t$ is then adjusted according to the selected value of $T$.

\section{Avoiding Restarting from Scratch }\label{parallle2-sec}

In this section, we present a refined parallel framework that avoids restarting from scratch at each new stage of a processor. Instead of always using the same initial point $x_0$, the next stage starts from an improved point $x_0^*$ obtained from the $p$ processors, where
\[
F(x_0^*)\le F(x_0).
\]
In this way, the framework exploits the partial progress made during previous stages and reuses it to accelerate convergence in subsequent stages across all processors.

This refinement requires evaluating the objective function $F(\cdot)$ and introduces communication among processors to identify the best current iterate. In this section, we briefly describe this extension and discuss its potential advantages.


\subsection{A Refined Parallel Framework}

In this subsection, we describe a refined parallel framework for gradient descent. For an objective function of the form
\[
F(x)=\sum_{i=1}^{k}f_i(x)^2,
\]
evaluating the objective function $F(x)$ may take significantly longer than computing a stochastic gradient $\bigtriangledown f_i(x)$ with a random $i\in \{1,\ldots, k\}$. 


We also have a parameter $\delta$ to control how many steps for iterations and how many steps to find some $x_i$ with $F(x_i)<F(x_0^*)$.

\vskip 20pt
{\bf Algorithm} Parallel2-Framework$(GD(.,.),  x_0, h(.,.), T_0, p, \delta 
)$

Input: 

\begin{enumerate}[label=\arabic*.]
\item
GD$(x_0, T)$ is a gradient descent method with start point $x_0$, and $T$ iterations.

    \item $T_0$ is the least number of steps to execute.

    \item $h(j, i)\ge T_0: \mathbb{N}\times \mathbb{N}\rightarrow \mathbb{N}$ is a function to assign the number of iterations when calling a gradient descent method GD(.).

    \item $h(T):\mathbb{N}\rightarrow \mathbb{N}$ is a function to determine how many steps will be used to run the selection function. For example, $h(T)=\floor{T/10}$.


    \item $x_0\in \mathbb{R}^m$ is the start point.

    \item $p$ is the number of processors.

    \item $\delta\in(0,1)$.

\item $\{$

\item Let $x_0^*=x_0$ be shared by all processors.

\item Processor $j$ ($j=0,1,\ldots, p-1$):

\begin{enumerate}

   \item Let $i=1$

\item Repeat

   \item $\{$ 

     \item \qquad $T_{j,i}=h(j,i)$

          \item \qquad $t=\delta\cdot T_{j,i}$

\item \qquad $Z=$GD$(x_0^*, T_{j,i}-t)$

\item \qquad Select$(Z,  x_0^*, y_0^*, F(.), t 
                     )$


    \item \qquad Let $i=i+1$

\item $\}$


\end{enumerate}

\item 
$\}$

\end{enumerate}

{\bf End of Algorithm}

\vskip 20pt

\begin{proposition}\label{embeding2-prop}
 Assume that GD($x_0, T)$ is a gradient descent method.
 Assume that Parallel2-Framework$(GD(.))$ has $(p,\alpha_p)$-approximation.  Then 
 Parallel2-Framework$(GD(.))$ executes
 $GD(x_0, (1-\delta)T_{j,i})$ satisfying $T\le (1-\delta)T_{j,i}\le T_{j,i}\le T_{j,i}^*\le \frac{\alpha_p T}{1-\delta}$.  
\end{proposition}

\begin{proof} Let $T'=\frac{T}{1-\delta}$. 
By the condition of $(p, \alpha_p)$-approximation, 
there is $T_{j,i}$ with  $T'\le T_{j,i}\le T_{j,i}^*\le \alpha_p T'$. 
This implies $GD(x_0, (1-\delta)T_{j,i})$ satisfying $T\le (1-\delta)T_{j,i}\le T_{j,i}\le T_{j,i}^*\le \frac{\alpha_p T}{1-\delta}$.
\end{proof}


We require a selection function, denoted by Select(.), that chooses one $x_i$ from the sequence $x_1,x_2,\ldots,x_T$ generated by executing GD$(x_0,T)$. The parameter $t$ specifies the maximum number of iterates that Select(.) is allowed to access. 
The following principles may be used to design Select(.):
\begin{itemize}
\item It runs at most the time of   $t$ iterations in $GD(.)$.
\item If an $x_i$ satisfying $F(x_i)<F(x_0^*)$ is found after examining a subset of $Z$, then $x_0^*$ is updated to $x_i$.
\item There is a mutual exclusion to update $x_0^*$ among the $p$ processors.
\end{itemize}

There is no sufficient information to establish quantitative guarantees for the function
$\mathrm{Select}(\cdot)$. Therefore, we leave the implementation and analysis
of $\mathrm{Select}(\cdot)$ as a topic for future research.

\section{Conclusions and Future Developments}

In this paper, we develop a parallel framework that transforms static gradient descent methods into adaptive ones through parallel execution. Given a target number of iterations $T$ that may satisfy the desired convergence conditions, the $p$ processors in the framework search for a suitable parameter $T_{j,i}$ according to a carefully designed geometric sequence. The objective is to minimize the approximation factor $\alpha_p$ while ensuring that $
T\le T_{j,i}\le T_{j,i}^*\le \alpha_p T.
$

Several research directions remain open. First, it will be valuable to identify additional static gradient methods that can be incorporated into this framework. Second, after the number of iterations $T$ is determined, more effective strategies for selecting the corresponding learning rate should be investigated. Third, the arithmetically simple gradient descent methods proposed in this paper eliminate division and square-root operations by replacing them with binary shift operations. Drawing on the author's experience as an FPGA hardware engineer in the computer industry, we believe that this design is more suitable for hardware implementation and chip design. Developing even more efficient gradient descent algorithms for specialized hardware accelerators is therefore an interesting direction for future research.

Another interesting open problem is to close the gap between the current upper bound of $4$ (Corollary~\ref{main2b-cor}) and the lower bound of $2$ (Theorem~\ref{lower-1-2-theorem}) for the approximation factor $\alpha_1$ in a $(1,\alpha_1)$-approximation. Progress on narrowing this gap for the single-processor case may also provide new insights into closing the corresponding gap for $\alpha_p$ when $p>1$.


\section{Acknowledgments}
This author is grateful to Jose Nunez for proofreading an earlier version of this paper and for his helpful comments, which improved the presentation of the paper.

\bibliographystyle{abbrv}
\bibliography{bib}

%% file: bib.bib
@string{Springer = "Springer-Verlag"}

@article{r,
author={Sweeney, Latanya},
title={$k$-anonymity: a model for protecting privacy},
journal={International Journal on Uncertainity Fuzziness Knowledge-Based Systems},
volume={10},
number={5},
pages={557-570},
year={2002}}

@article{WardWuBottou19,
author="R Ward and X Wu and L Bottou",
title="Adagrad stepsizes: Sharp convergence over nonconvex landscapes",
journal="Journal of Machine Learning Research",
volume="21 (219)",
year="2020",
pages="1-30",
}

@inproceedings{XieWuWard2020,
author="Yuege Xie and Xiaoxia Wu and Rachel Ward",
title="Linear convergence of adaptive stochastic gradient descent",
booktitle="Proceedings of the Twenty Third International Conference on Artificial Intelligence and Statistics",
year="2020",
pages="PMLR 108:1475-1485",
}

@article{DuchiHazanSinger2011,
author="John Duchi and Elad Hazan and Yoram Singer",
title="Adaptive Subgradient Methods for Online Learning and Stochastic Optimization",
journal="The Journal of Machine Learning Research",
volume="12",
year="2011",
pages="2121–2159",
}

@article{NemirovskiJuditskyLanShapiro2009,
author="A. Nemirovski and A. Juditsky and G. Lan and A. Shapiro",
title="Robust Stochastic Approximation Approach to Stochastic Programming",
journal="SIAM Journal on Optimization",
volume="19",
year="2009",
pages="1574-1609",
}

@article{DBLP:journals/corr/abs-1212-5701,
  author       = {Matthew D. Zeiler},
  title        = {{ADADELTA:} An Adaptive Learning Rate Method},
  journal      = {CoRR},
  volume       = {abs/1212.5701},
  year         = {2012},
  url          = {http://arxiv.org/abs/1212.5701},
  eprinttype    = {arXiv},
  eprint       = {1212.5701},
  bibsource    = {dblp computer science bibliography, https://dblp.org}
}

@article{DBLP:journals/corr/abs-1002-4908,
  author       = {H. Brendan McMahan and
                  Matthew J. Streeter},
  title        = {Adaptive Bound Optimization for Online Convex Optimization},
  journal      = {CoRR},
  volume       = {abs/1002.4908},
  year         = {2010},
  url          = {http://arxiv.org/abs/1002.4908},
  eprinttype    = {arXiv},
  eprint       = {1002.4908},
  bibsource    = {dblp computer science bibliography, https://dblp.org}
}

@article{BottouCurtisNocedal2018,
author="Léon Bottou, Frank E. Curtis, Jorge Nocedal",
title="Optimization Methods for Large-Scale Machine Learning",
journal="SIAM Reviews",
volume="60(2)",
year="2018",
pages="223–311",
}

@unpublished{KingmaBa2015,
author ="Diederik P. Kingma and Jimmy Ba",
title = "Adam: A Method for Stochastic Optimization",
year = "2015",
note = "arXiv:1412.6980",
}

@inproceedings{OrabonaPal2015,
author="Francesco Orabona and D´avid P´al",
title="Scale-Free Algorithms for Online Linear Optimization",
booktitle="Algorithmic Learning Theory. ALT 2015. Lecture Notes in Computer Science, vol 9355.",
year="2015",
pages="287–301",
}

@article{DBLP:journals/corr/abs-1711-01761,
  author       = {Alexandre D{\'{e}}fossez and
                  Francis R. Bach},
  title        = {AdaBatch: Efficient Gradient Aggregation Rules for Sequential and
                  Parallel Stochastic Gradient Methods},
  journal      = {CoRR},
  volume       = {abs/1711.01761},
  year         = {2017},
  url          = {http://arxiv.org/abs/1711.01761},
  eprinttype    = {arXiv},
  eprint       = {1711.01761},
  bibsource    = {dblp computer science bibliography, https://dblp.org}
}

@inproceedings{TanMaDaiQian2016,
author="Conghui Tan and Shiqian Ma and Yu-Hong Dai and Yuqiu Qian",
title="Barzilai-Borwein Step Size for Stochastic Gradient Descent",
booktitle="30th Conference on Neural Information Processing Systems (NIPS 2016), Barcelona, Spain",
year="2016",
pages="685-693",
}

@inproceedings{DBLP:conf/ijcai/ChenZTYCG20,
  author       = {Jinghui Chen and
                  Dongruo Zhou and
                  Yiqi Tang and
                  Ziyan Yang and
                  Yuan Cao and
                  Quanquan Gu},
  editor       = {Christian Bessiere},
  title        = {Closing the Generalization Gap of Adaptive Gradient Methods in Training
                  Deep Neural Networks},
  booktitle    = {Proceedings of the Twenty-Ninth International Joint Conference on
                  Artificial Intelligence, {IJCAI} 2020},
  pages        = {3267--3275},
  publisher    = {ijcai.org},
  year         = {2020},
  url          = {https://doi.org/10.24963/ijcai.2020/452},
  doi          = {10.24963/IJCAI.2020/452},
  bibsource    = {dblp computer science bibliography, https://dblp.org}
}

@inproceedings{DBLP:conf/iclr/ReddiKK18,
  author       = {Sashank J. Reddi and
                  Satyen Kale and
                  Sanjiv Kumar},
  title        = {On the Convergence of Adam and Beyond},
  booktitle    = {6th International Conference on Learning Representations, {ICLR} 2018,
                  Vancouver, BC, Canada, April 30 - May 3, 2018, Conference Track Proceedings},
  publisher    = {OpenReview.net},
  year         = {2018},
  url          = {https://openreview.net/forum?id=ryQu7f-RZ},
  bibsource    = {dblp computer science bibliography, https://dblp.org}
}

@inproceedings{DBLP:conf/icml/MukkamalaH17,
  author       = {Mahesh Chandra Mukkamala and
                  Matthias Hein},
  editor       = {Doina Precup and
                  Yee Whye Teh},
  title        = {Variants of RMSProp and Adagrad with Logarithmic Regret Bounds},
  booktitle    = {Proceedings of the 34th International Conference on Machine Learning,
                  {ICML} 2017, Sydney, NSW, Australia, 6-11 August 2017},
  series       = {Proceedings of Machine Learning Research},
  volume       = {70},
  pages        = {2545--2553},
  publisher    = {{PMLR}},
  year         = {2017},
  url          = {http://proceedings.mlr.press/v70/mukkamala17a.html},
  bibsource    = {dblp computer science bibliography, https://dblp.org}
}

@article{GhadimiAndLan13,
  author       = {Saeed Ghadimi and
                  Guanghui Lan},
  title        = {Stochastic First- and Zeroth-Order Methods for Nonconvex Stochastic
                  Programming},
  journal      = {{SIAM} J. Optim.},
  volume       = {23},
  number       = {4},
  pages        = {2341--2368},
  year         = {2013},
  url          = {https://doi.org/10.1137/120880811},
  doi          = {10.1137/120880811},
  bibsource    = {dblp computer science bibliography, https://dblp.org}
}

@article{BottouCurtisNocedal18,
  author       = {L{\'{e}}on Bottou and
                  Frank E. Curtis and
                  Jorge Nocedal},
  title        = {Optimization Methods for Large-Scale Machine Learning},
  journal      = {{SIAM} Rev.},
  volume       = {60},
  number       = {2},
  pages        = {223--311},
  year         = {2018},
  url          = {https://doi.org/10.1137/16M1080173},
  doi          = {10.1137/16M1080173},
  bibsource    = {dblp computer science bibliography, https://dblp.org}
}

@article{RobbinsAndMonro51,
author="Herbert Robbins and Sutton Monro",
title="A Stochastic Approximation Method",
journal="Annuals of Mathematical Statistics",
volume="22(3)",
year="1951",
pages="400-407",
}

@inproceedings{FawRoutCaramanisShakkottai23,
  author       = {Matthew Faw and
                  Litu Rout and
                  Constantine Caramanis and
                  Sanjay Shakkottai},
  editor       = {Gergely Neu and
                  Lorenzo Rosasco},
  title        = {Beyond Uniform Smoothness: {A} Stopped Analysis of Adaptive {SGD}},
  booktitle    = {The Thirty Sixth Annual Conference on Learning Theory, {COLT} 2023,
                  12-15 July 2023, Bangalore, India},
  series       = {Proceedings of Machine Learning Research},
  volume       = {195},
  pages        = {89--160},
  publisher    = {{PMLR}},
  year         = {2023},
  url          = {https://proceedings.mlr.press/v195/faw23a.html},
  bibsource    = {dblp computer science bibliography, https://dblp.org}
}

@inproceedings{WangZhangMaChen023,
  author       = {Bohan Wang and
                  Huishuai Zhang and
                  Zhiming Ma and
                  Wei Chen},
  editor       = {Gergely Neu and
                  Lorenzo Rosasco},
  title        = {Convergence of AdaGrad for Non-convex Objectives: Simple Proofs and
                  Relaxed Assumptions},
  booktitle    = {The Thirty Sixth Annual Conference on Learning Theory, {COLT} 2023,
                  12-15 July 2023, Bangalore, India},
  series       = {Proceedings of Machine Learning Research},
  volume       = {195},
  pages        = {161--190},
  publisher    = {{PMLR}},
  year         = {2023},
  url          = {https://proceedings.mlr.press/v195/wang23a.html},
  bibsource    = {dblp computer science bibliography, https://dblp.org}
}

@article{ArjevaniCDFSW23,
  author       = {Yossi Arjevani and
                  Yair Carmon and
                  John C. Duchi and
                  Dylan J. Foster and
                  Nathan Srebro and
                  Blake E. Woodworth},
  title        = {Lower bounds for non-convex stochastic optimization},
  journal      = {Math. Program.},
  volume       = {199},
  number       = {1},
  pages        = {165--214},
  year         = {2023},
  url          = {https://doi.org/10.1007/s10107-022-01822-7},
  doi          = {10.1007/S10107-022-01822-7},
  bibsource    = {dblp computer science bibliography, https://dblp.org}
}

@book{BertsekasTsitsiklis89,
author = "D. P. Bertsekas and  J. N. Tsitsiklis",
title = "Parallel and Distributed Computation: Numerical Methods",
publisher = "Prentice Hall",
address = "",
year = "1989",
ISBN = "ISBN-13: 9780136487005"
}

@inproceedings{reddi2018convergence,
  title={On the Convergence of Adam and Beyond},
  author={Reddi, Sashank J. and Kale, Satyen and Kumar, Sanjiv},
  booktitle={International Conference on Learning Representations (ICLR)},
  year={2018}
}

@article{iacob2025mtdao,
  title={MT-DAO: Multi-Timescale Distributed Adaptive Optimizers with Local Updates},
  author={Iacob, Alex and Jovanovic, Andrej and Safaryan, Mher and Kurmanji, Meghdad and Sani, Lorenzo and Horvath, Samuel and Shen, William F. and Qiu, Xinchi and Lane, Nicholas D.},
  journal={arXiv preprint arXiv:2510.05361},
  year={2025}
}

@inproceedings{recht2011hogwild,
  title={Hogwild!: A Lock-Free Approach to Parallelizing Stochastic Gradient Descent},
  author={Recht, Benjamin and Re, Christopher and Wright, Stephen and Niu, Feng},
  booktitle={Advances in Neural Information Processing Systems},
  volume={24},
  year={2011}
}

@article{li2014scaling,
  title={Scaling Distributed Machine Learning with the Parameter Server},
  author={Li, Mu and Andersen, David G. and Park, Jun Woo and Smola, Alexander and Ahmed, Amr and others},
  journal={USENIX Symposium on Operating Systems Design and Implementation},
  year={2014}
}

@article{dean2012large,
  title={Large Scale Distributed Deep Networks},
  author={Dean, Jeffrey and Corrado, Greg and Monga, Rajat and Chen, Kai and Devin, Matthieu and others},
  journal={Advances in Neural Information Processing Systems},
  volume={25},
  year={2012}
}

@article{cutkosky2018distributed,
  title={Distributed Stochastic Optimization via Adaptive SGD},
  author={Cutkosky, Ashok and Busa-Fekete, Róbert},
  journal={arXiv preprint arXiv:1802.05811},
  year={2018}
}

@inproceedings{Loshchilov2019Decoupled,
  title={Decoupled Weight Decay Regularization},
  author={Loshchilov, Ilya and Hutter, Frank},
  booktitle={International Conference on Learning Representations (ICLR)},
  year={2019}
}

@book{Lan2020,
author = "Guanghui Lan",
title = "First-order and Stochastic Optimization Methods for Machine Learning",
publisher = "Springer",
year = "2020",
ISBN = "978-3-030-39568-1"
}
